\documentclass{article}

\usepackage{arxiv}

\usepackage[utf8]{inputenc} 
\usepackage[T1]{fontenc}    
\usepackage{hyperref}       
\usepackage{url}            
\usepackage{booktabs}       
\usepackage{amsfonts}       
\usepackage{nicefrac}       
\usepackage{microtype}      
\usepackage{graphicx}
\usepackage{natbib}
\usepackage{doi}

\usepackage{multirow}
\usepackage{layout}

\usepackage{amsmath}
\usepackage{amsthm}
\usepackage{amssymb}
\usepackage{amsfonts}
\usepackage{mathtools}
\usepackage{siunitx}

\usepackage{enumitem}

\usepackage{bbm}

\usepackage{url}
\usepackage{color}
\usepackage{graphicx}
\usepackage{verbatim}

\usepackage{cleveref}

\usepackage{apptools}
\usepackage{thmtools}
\usepackage{mdframed}
\usepackage{algorithm, algpseudocode}

\usepackage[flushleft]{threeparttable}
\usepackage{array,booktabs,makecell}

\usepackage{listings}
\usepackage{xcolor}

\newcommand{\norm}[1]{\left\| #1 \right\|}

\newcommand{\inp}[2]{\left\langle#1,#2\right\rangle} 

\newcommand{\parens}[1]{\left( #1 \right)}
\newcommand{\brac}[1]{\left\{ #1 \right\}}

\newcommand{\cB}{\mathcal{B}}

\newcommand{\cD}{\mathcal{D}}

\newcommand{\cL}{\mathcal{L}}

\newcommand{\cO}{\mathcal{O}}
\newcommand{\cP}{\mathcal{P}}

\newcommand{\cX}{\mathcal{X}}

\newcommand{\supp}{\textnormal{supp}}

\newcommand{\del}[1]{}

\newcommand{\R}{\mathbb{R}} 

\newcommand{\I}[1]{\mathbb{I}\left(#1\right)}

\newcommand{\eqdef}{:=} 

\newcommand{\Prob}[1]{{\mathbb{P}}\left(#1\right)} 

\newcommand{\Exp}[1]{{\mathbb{E}}\left[#1\right]}

\newcommand{\ExpCond}[2]{{\mathbb{E}}\left[\left.#1\right\vert#2\right]}
\newcommand{\ExpSub}[2]{{\mathbb{E}}_{#1}\left[#2\right]}

\DeclareMathOperator{\tr}{tr}           

\DeclareMathOperator*{\argmax}{arg\,max}
\DeclareMathOperator*{\argmin}{arg\,min}

\def\ceil#1{\left\lceil #1 \right\rceil}
\def\flr#1{\left\lfloor #1 \right\rfloor}

\usepackage{tcolorbox}
\usepackage{pifont}
\definecolor{mydarkgreen}{RGB}{39,130,67}
\definecolor{mydarkred}{RGB}{192,47,25}
\definecolor{mydarkorange}{RGB}{39,130,67}
\definecolor{yaleblue}{rgb}{0.06, 0.3, 0.57}
\definecolor{myred}{RGB}{215,60,50}
\definecolor{coral}{HTML}{FF7F50}
\definecolor{peach}{HTML}{CC5500}
\definecolor{NavyBlue}{RGB}{0, 0, 128}
\definecolor{CrimsonRed}{RGB}{220, 20, 60}
\definecolor{Gold}{RGB}{204, 172, 0}
\definecolor{PaleRed}{rgb}{0.95, 0.3, 0.3}
\definecolor{MPLOrange}{rgb}{1.0, 0.6470588235294118, 0.0}
\definecolor{MPLGreen}{rgb}{0.0, 0.5019607843137255, 0.0}
\definecolor{ForestGreen}{RGB}{34,139,34}
\definecolor{OliveGreen}{RGB}{107,142,35}
\definecolor{PurplePrint}{RGB}{117, 112, 179}
\definecolor{GreenPrint}{RGB}{27, 158, 119}
\definecolor{RedPrint}{RGB}{217, 95, 2}
\definecolor{VeryLightGray}{rgb}{0.9,0.9,0.9}

\newcommand{\algname}[1]{{\sf #1}}
\newcommand{\algnamesmall}[1]{{\small \sf #1}}

\newcommand{\newalgsmall}{\algnamesmall{Drop-Muon}}

\newcommand{\lmo}[2]{{\rm LMO}_{#1}(#2)}

\mdfdefinestyle{theoremstyle}{
    backgroundcolor=lightgray!12, 
    linecolor=lightgray!12, 
    innertopmargin=6pt, 
    innerbottommargin=4pt,
    innerrightmargin=5pt,
    innerleftmargin=5pt
}

\declaretheoremstyle[
    headfont=\bfseries,
    bodyfont=\normalfont,
    mdframed={
        style=theoremstyle
    }
]{graytheoremstyle}

\declaretheorem[
    name=Theorem,
    style=graytheoremstyle,
    numberwithin=section
]{theorem}

\declaretheorem[
    name=Lemma,
    style=graytheoremstyle,
    numberwithin=section
]{lemma}

\declaretheorem[
    name=Assumption,
    style=graytheoremstyle,
    numberwithin=section
]{assumption}

\theoremstyle{plain}

\theoremstyle{definition}
\newtheorem{example}{Example}
\newtheorem{remark}[theorem]{Remark}

\definecolor{codegreen}{rgb}{0,0.6,0}
\definecolor{codegray}{rgb}{0.5,0.5,0.5}
\definecolor{codepurple}{rgb}{0.58,0,0.82}
\definecolor{backcolour}{rgb}{0.95,0.95,0.92}

\lstdefinestyle{mystyle}{
    backgroundcolor=\color{backcolour},   
    commentstyle=\color{codegreen},
    keywordstyle=\color{magenta},
    numberstyle=\tiny\color{codegray},
    stringstyle=\color{codepurple},
    basicstyle=\ttfamily\footnotesize,
    breakatwhitespace=false,         
    breaklines=true,                 
    captionpos=b,                    
    keepspaces=true,                 
    numbers=left,                    
    numbersep=5pt,                  
    showspaces=false,                
    showstringspaces=false,
    showtabs=false,                  
    tabsize=2
}

\newcommand{\squeeze}{}

\title{Drop-Muon: Update Less, Converge Faster}

\author{Kaja Gruntkowska \\
	    KAUST\thanks{King Abdullah University of Science and Technology}\\
    	Center of Excellence for Generative AI\\ 	
	    Thuwal, Saudi Arabia \\
        \And
    	Yassine Maziane \\
    	KAUST$^*$ \\	
    	Center of Excellence for Generative AI\\ 
        Thuwal, Saudi Arabia \\
        \And
    	Zheng Qu \\
    	Shenzhen University \\	
    	School of Mathematical Sciences\\ 
        Shenzhen, China \\
    	\And
    	Peter Richt\'{a}rik \\
    	KAUST$^*$ \\	
    	Center of Excellence for Generative AI\\ 
        Thuwal, Saudi Arabia \\
}

\date{}

\renewcommand{\shorttitle}{Drop-Muon: Update Less, Converge Faster}

\usepackage[hyperpageref]{backref}
\hypersetup{ 
    colorlinks = true,
    urlcolor = yaleblue,
    linkcolor = yaleblue,
    citecolor = yaleblue
}
\renewcommand*{\backref}[1]{}
\renewcommand*{\backrefalt}[4]{%
   \ifcase #1 %
     \footnotesize{(Not cited.)}%
   \or
     \footnotesize{(Cited on page~#2)}%
   \else
     \footnotesize{(Cited on page~#2)}%
\fi }

\allowdisplaybreaks

\begin{document}
\maketitle

\begin{abstract}
	Conventional wisdom in deep learning optimization dictates updating all layers at every step--a principle followed by all recent state-of-the-art optimizers such as \algnamesmall{Muon}. In this work, we challenge this assumption, showing that full-network updates can be fundamentally suboptimal, both in theory and in practice. We introduce a non-Euclidean Randomized Progressive Training method---{\newalgsmall}---a simple yet powerful framework that updates only a subset of layers per step according to a randomized schedule, combining the efficiency of progressive training with layer-specific non-Euclidean updates for top-tier performance. We provide rigorous convergence guarantees under both layer-wise smoothness and layer-wise $(L^0, L^1)$-smoothness, covering deterministic and stochastic gradient settings, marking the first such results for progressive training in the stochastic and non-smooth regime. Our cost analysis further reveals that full-network updates are not optimal unless a very specific relationship between layer smoothness constants holds. Through controlled CNN experiments, we empirically demonstrate that {\newalgsmall} consistently outperforms full-network \algnamesmall{Muon}, achieving the same accuracy up to $1.4\times$ faster in wall-clock time. Together, our results suggest a shift in how large-scale models can be efficiently trained, challenging the status quo and offering a highly efficient, theoretically grounded alternative to full-network updates.
\end{abstract}

\section{Introduction}\label{sec:intro}

Since their debut, \algnamesmall{Adam} and related methods \citep{kingma2015adam, loshchilov2019decoupled} have dominated deep learning optimization. Yet, the field is now at an inflection point. Recent advances highlight a new generation of algorithms designed to better capture the geometry of modern models, with
\algnamesmall{Muon} \citep{jordan2024muon} and its successors---\algnamesmall{Scion} \citep{pethick2025training} and \algnamesmall{Gluon} \citep{riabinin2025gluon}---emerging as promising alternatives. Fueled by state-of-the-art performance in large language models (LLMs) training \citep{liu2025muon, shah2025practical, therien2025muloco, moonshotai2025, wen2025fantastic} and emerging theoretical developments \citep{pethick2025training, kovalev2025understanding, riabinin2025gluon}, these methods are on track to disrupt entrenched practices.

Central to their design are the layer-specific \emph{linear minimization oracles} (LMOs) over non-Euclidean norm balls, enabling better alignment with the highly anisotropic loss landscapes of neural networks.
Concretely, let $X = [X_1, \ldots, X_b]$ denote the parameters of a $b$-layer model, with $X_i$ indexing the parameters of layer $i \in[b] \eqdef \{1, \ldots, b\}$. Each of the aforementioned optimizers can be viewed as an instance of the general update rule
\begin{align}\label{eq:muon}
    X_i^{k+1} = X_i^k + \lmo{\cB_i(0,t_i^k)}{M_i^k}, \qquad {\color{myred} i\in\{1, \ldots, b\}}
\end{align}
(see \Cref{sec:muon}). Here, $\cB_i(X_i,t_i) \eqdef \{Z_i \in \cX_i: \norm{X_i - Z_i}_{(i)} \leq t_i\}$ is a ball of radius $t_i$ centered at $X_i$ in the vector space $\cX_i$, where $\norm{\cdot}_{(i)}$ is a norm chosen for the $i$th layer, $M_i^k$ is a momentum term, and the linear minimization oracle is defined as $\lmo{\cB_i(X_i,t_i)}{M_i} \eqdef \argmin_{Z_i \in \cB_i(X_i,t_i)} \inp{M_i}{Z_i}$.
Different choices of $\norm{\cdot}_{(i)}$ produce different algorithms (for example, \algnamesmall{Muon} uses the spectral norm for hidden layers). Crucially, however, all these updates share a common characteristic: {\color{myred} all layers are updated at every iteration}.  
In this work, we question this design choice. Our central hypothesis is simple yet fundamental:
\begin{center}
    \emph{Updating the entire network at every step may \textbf{not} be optimal.}
\end{center}
In the remainder of this paper, we demonstrate that the default practice of full-network updates is indeed not universally the best choice, from both \emph{theoretical} and \emph{practical} points of view, calling into question a core principle of standard training protocols.

\subsection{Background}

Let us begin by formalizing the setup. We consider the optimization problem
\begin{align}\label{eq:problem}
    \min_{X\in\cX} \brac{f(X) \eqdef \ExpSub{\xi\sim\cP}{f(X; \xi)}},
\end{align}
where $X \in \cX$ represents the collection of trainable parameters of a neural network. Specifically, $X$ is composed of block variables $X_i \in \cX_i \eqdef \R^{m_i \times n_i}$ corresponding to layer $i \in[b]$; we write $X = [X_1, \ldots, X_b]$. In this context, $\cX$ is the $d$-dimensional product space
\begin{align*}
    \squeeze \cX \eqdef \bigotimes_{i = 1}^b \cX_i \equiv \cX_1 \otimes \ldots \otimes \cX_b,
\end{align*}
where $d \eqdef \sum_{i=1}^b m_i n_i $. Each function $f(\cdot; \xi): \cX \to \R$ is continuously differentiable, potentially nonconvex and non-smooth, and represents the loss of the model evaluated at a data point $\xi$ sampled from the probability distribution $\cP$. We denote by $\nabla_i f(X) \in \cX_i$ the gradient component corresponding to the $i$th layer, so that $\nabla f(X) = [\nabla_1 f(X), \dots, \nabla_b f(X)] \in \cX$.
Each space $\cX_i$ is equipped with the trace inner product, defined as  $\langle X_i, Y_i \rangle_{(i)} \eqdef \tr(X_i^{\top} Y_i)$ for $X_i,Y_i \in \cX_i$, which induces the standard Euclidean norm, denoted by $\norm{\cdot}_2$. In addition, each space is endowed with an arbitrary norm $\norm{\cdot}_{(i)}$ (which need not be induced by this inner product). We let $\norm{\cdot}_{(i) \star}$ be the dual norm associated with~$\norm{\cdot}_{(i)}$ (i.e., $\|X_i\|_{(i) \star} \eqdef \sup_{\norm{Z_i}_{(i)} \leq 1} \inp{X_i}{Z_i}_{(i)}$ for any $X_i\in \cX_i$).

\begin{algorithm}[t]
\caption{{\newalgsmall}}\label{alg:rt_arbitrary_stoch}
\begin{algorithmic}[1]
\State \textbf{Input:} initial iterate $X^0 = [X_1^0, \dots, X_b^0] \in \cX$; momentum $M^0=[M_1^0, \dots, M_b^0]\in \cX$; stepsizes $\gamma_i^k > 0$; momentum parameters $\beta^k \in [0,1)$
\For{$k = 0, \dots, K-1$}
    \State Sample $\xi^k \sim \cP$ and the set of \emph{active} layers $S^k\sim\cD$
    \For{$i \not\in S^k$} \hfill {\scriptsize\color{gray} $\triangleright$ Freeze layers not selected as \emph{active}}
        \State $M_i^k = M^{k-1}_i$
        \State $X_i^{k+1} = X_i^k$
    \EndFor
    \For{$i \in S^k$} \hfill {\scriptsize\color{gray} $\triangleright$ Update \emph{active} layers}
        \State Update momentum $M_i^k = (1-\beta_i) M^{k-1}_i + \beta_i \nabla_i f(X^k; \xi^k)$
        \State Update parameters via
        \begin{align}\label{eq:sharp_upd_stoch}
            X_i^{k+1} = \lmo{\cB(X_i^k,t_i^k)}{M_i^k}
			= X_i^k - \gamma_i^k \parens{M_i^k}^{\sharp}
        \end{align}
    \EndFor
\EndFor
\end{algorithmic}
\end{algorithm}

\section{The Algorithm}

Before summarizing our main contributions (see the end of \Cref{sec:cost_model}), we dive directly into presenting our method. Motivated by the considerations in \Cref{sec:intro}, we propose {\newalgsmall} (\Cref{alg:rt_arbitrary_stoch})--a non-Euclidean layer-wise optimizer for deep learning based on the idea of sub-network training.
At each iteration~$k$, instead of updating the entire network as in standard \algnamesmall{Muon}, {\newalgsmall} samples a random subset $S^k \subseteq [b]$ of layers according to a user-defined distribution $\cD$ and updates only the parameters of layers in~$S^k$, keeping all other layers frozen.
As the reader may have noticed, the main LMO update step \eqref{eq:sharp_upd_stoch} admits two equivalent formulations. The alternative representation of \eqref{eq:muon} uses the \emph{sharp operator} \citep{nesterov2012efficiency, kelner2014almost}, defined for any $M \in \cX$ by $M^{\sharp} \eqdef \argmax_{X \in \cX} \{\inp{M}{X} - \frac{1}{2} \norm{X}^2\}$. It is well known that $M^{\sharp}$ relates to the LMO via $M^{\sharp} = - \norm{M}_{\star} \lmo{\cB(0,1)}{M}$, and hence
\begin{eqnarray}\label{eq:lmo_sharp}
    \squeeze X_i^{k+1} = X_i^k + t_i^k \lmo{\cB_i(0,1)}{M_i^k} = X_i^k - \frac{t_i^k}{\norm{M_i^k}_{(i) \star}} \parens{M_i^k}^{\sharp},
\end{eqnarray}
which corresponds to a layer-wise normalized \emph{steepest descent} step with stepsize $\gamma_i^k \eqdef \nicefrac{t_i^k}{\norm{M_i^k}_{\star}}$.
When $\norm{\cdot}_{(i)}=\norm{\cdot}_2$ is the standard Euclidean norm, the sharp operator reduces to the identity mapping, so that $M^{\sharp} = M$, and the update coincides with Stochastic Gradient Descent with momentum (\algnamesmall{SGDM}) \citep{cutkosky2020momentum}, though here performed layer-wise. These equivalent formulations will be repeatedly invoked in the proofs of the results from Sections~\ref{sec:iter_det} and~\ref{sec:iter_stoch}.

\section{Cost Model}\label{sec:cost_model}

Our key theoretical contribution is that strategically skipping some layer updates may lead to performance gains.
To isolate the core phenomenon, we first study a simplified variant of the method without stochasticity in the gradients and momentum, described in \Cref{alg:rt_arbitrary}. This deterministic version follows the same fundamental principles as \Cref{alg:rt_arbitrary_stoch}, with the difference that momentum terms $M_i^k$ are replaced by the components of the full gradient $\nabla_i f(X^k)$.
When $\norm{\cdot}_{(i)}=\norm{\cdot}_2$ is the Euclidean norm, the update \eqref{eq:sharp_upd} coincides with that of layer-wise Gradient Descent~(\algnamesmall{GD}).
 
\begin{algorithm}[t]
    \begin{algorithmic}[1]
    \State {\bf Input:} initial iterate $X^0 = [X_1^0, \ldots, X_b^0] \in \cX$; stepsizes $\gamma_i^k > 0$, $i\in[b]$, $k\geq0$
    \For{$k = 0, \dots, K-1$}
        \State Sample $S^k\sim\cD$
        \State Freeze the layers $\{X_i\}_{i\not\in S^k}$ ($X_i^{k+1} = X_i^k$ for $i\not\in S^k$)
        \State Update the layers $\{X_i\}_{i\in S^k}$ via
		\begin{align}\label{eq:sharp_upd}
			X_i^{k+1} = X_i^k - \gamma_i^k \parens{\nabla_i f(X^k)}^{\sharp}
		\end{align}
    \EndFor
    \caption{{\newalgsmall} (deterministic gradient variant)}
	\label{alg:rt_arbitrary}
	\end{algorithmic}
\end{algorithm}

How expensive is one step of \Cref{alg:rt_arbitrary}? The answer is governed by the sampling distribution $\cD$. Consider iteration $k$ and denote $s^k \eqdef \min S^k$ (the smallest index of an active layer). The operations performed by the algorithm can be summarized as follows:
\begin{enumerate}[label=(\roman*)]
    \item \textbf{Backward pass:}
    Backpropagate gradients through layers $[X^k_{s^k}, \ldots, X^k_b]$. Since layers $1, \ldots, s^k - 1$ are frozen, no gradients are computed for them, effectively truncating the gradient flow at the first active layer.
    
    \item \textbf{Forward pass:}
    To evaluate the loss, activations must in principle be propagated through all layers $1, \ldots, b$. However, since only layers $[X^k_{s^k}, \ldots, X^k_b]$ are updated at iteration $k$, the activations up to layer $s^k - 1$ may be cached and reused in the next step.
    
    \item \textbf{Gradient transformation:}
    Given the gradients $\{\nabla_i f(X^k)\}_{i\in S^k}$, compute the corresponding sharp operators $\{(\nabla_i f(X^k))^{\sharp}\}_{i\in S^k}$ (or, equivalently, the LMOs; see \eqref{eq:lmo_sharp}).
        
    \item \textbf{Parameter updates:} 
    Update the parameters of layers $\{X_i^k\}_{i\in S^k}$ using their computed (transformed) gradients, while keeping the frozen layers unchanged.
\end{enumerate}

To model the total computational effort of the optimization procedure, we associate a \emph{cost} with each step (measured, for example, in FLOPs or wall-clock time). Let $c_{\mathrm{ov}} \geq 0$ denote the fixed per-iteration overhead (e.g., data loading).
As noted above, backpropagation must be performed from the last layer $b$ down to layer $s^k$, while forward-pass activations up to layer $s^k - 1$ can be cached and reused in subsequent iterations. Hence, the costs of steps (i) and (ii) can be aggregated into a single per-layer constant $c_i > 0$ for each $i \in [b]$.
In the non-Euclidean setting (where $M^{\sharp} \neq M$), an additional cost arises from computing sharp operators. We denote by $c_i^{\sharp} \geq 0$ the combined cost of evaluating this operator and performing the corresponding parameter update for layer $i$ (steps (iii) and (iv)).
Under this model, the total compute cost of iteration $k$ is  
\begin{eqnarray}\label{eq:cost_k}
    \squeeze \mathrm{cost}(S^k) \eqdef c_{\mathrm{ov}} + \sum\limits_{i=s^k}^b c_i + \sum\limits_{i \in S^k} c_i^{\sharp},
\end{eqnarray}
and, consequently, for a fixed target accuracy $\varepsilon > 0$, the expected cost of the entire optimization procedure can be expressed as
\begin{eqnarray}\label{eq:exp_cost_gen}
    \mathrm{cost}_{\varepsilon}(\cD) \eqdef K \times \Exp{\mathrm{cost}(\hat{S})},
\end{eqnarray}
where we write $\hat{S} \sim \cD$ to denote a random variable with the same distribution as that of the samplings (since $\{S^k\}_{k\geq0}$ are i.i.d.), and $K$ is the number of iterations to reach convergence (interpreted in the nonconvex case as reaching an $\varepsilon$-stationary point in expectation).
In the remainder of this paper, we compute $K$ under two smoothness regimes: layer-wise smoothness (\Cref{as:arbitrary_layer_smoothness}) and layer-wise $(L^0, L^1)$--smoothness (\Cref{as:arbitrary_layer_gen_smoothness}). We then evaluate various layer-update strategies and compute their expected costs as in \eqref{eq:exp_cost_gen}. These results converge to a single conclusion:
\begin{mdframed}[linecolor=lightgray!12, backgroundcolor=lightgray!12]
    \begin{center}
        Full-network updates are \emph{not} optimal unless a very specific relationship \\ between layer smoothness constants holds.
    \end{center}
\end{mdframed}
We provide a rigorous statement and justification for this claim in the sections that follow. The main takeaway, however, is simple: as this condition is highly unlikely to be realized in practice, updating only a subset of layers at each iteration should be seen as more efficient than the default strategy of updating all parameters at each iteration.

Our \textbf{contributions} can be summarized as follows:

1. \textbf{Challenging full network updates.}
We provide, to our knowledge, the first systematic investigation of the practice of updating \emph{all layers} of a network at every iteration. We argue---and rigorously demonstrate both in theory and in practice---that this design choice is \emph{generally suboptimal}.

2. \textbf{General framework for sub-network optimization.}
We introduce {\newalgsmall} (\Cref{alg:rt_arbitrary_stoch} and its deterministic gradient counterpart \Cref{alg:rt_arbitrary}), a principled layer-wise optimization framework with randomized layer subsampling. {\newalgsmall} strictly generalizes LMO-type methods (including \algnamesmall{Muon} \citep{jordan2024muon}, \algnamesmall{Scion} \citep{pethick2025training}, and \algnamesmall{Gluon} \citep{riabinin2025gluon}) by allowing random subsets of layers to be updated per step, with full-network training as a special case.
{\newalgsmall} supports virtually any layer sampling scheme (\Cref{sec:arbitrary_sampling}). In the main part of this paper, we focus on Randomized Progressive Training (\algnamesmall{RPT}) \citep{szlendak2024understanding}, a natural strategy aligned with backpropagation mechanics that avoids redundant gradient computations and reduces compute cost while maintaining strong convergence guarantees.

3. \textbf{Tight iteration complexity guarantees under novel smoothness regimes.}
We establish convergence guarantees for {\newalgsmall} under two regimes: \emph{layer-wise smoothness} (\Cref{thm:rpt_smooth_iter}) and \emph{layer-wise $(L^0,L^1)$--smoothness} (Theorems \ref{thm:rpt_l0l1_iter} and \ref{thm:stoch_rpt_l0l1_iter}). Our rates recover the state of the art for \algnamesmall{SGD}- and \algnamesmall{Muon}-type methods, and, to our knowledge, provide the first convergence guarantees for progressive training-style methods in the non-smooth setting.

4. \textbf{Theoretical compute-optimality results.}
To isolate the key phenomena, we first consider a deterministic gradient variant of {\newalgsmall} (\Cref{alg:rt_arbitrary}). Using a simple yet expressive cost model (\Cref{eq:cost_k}) accounting for per-layer forward/backward passes, gradient transformations, and parameter updates, we prove that full-network updates are \emph{not} optimal unless a very specific condition on layer smoothness constants holds (Theorems \ref{thm:rpt_p1} and \ref{thm:rpt_p1_l0l1}), which is unlikely in practice. This formally justifies selective layer updates as the compute-optimal default.

5. \textbf{Empirical validation.} Controlled CNN experiments on \texttt{MNIST}, \texttt{Fashion-MNIST}, and \texttt{CIFAR-10} show that {\newalgsmall} consistently outperforms standard full-network \algnamesmall{Muon}, achieving the same accuracy up to $1.4\times$ faster in wall-clock time.

\section{Randomized Progressive Training}\label{sec:rpt}

The general framework in \Cref{alg:rt_arbitrary} allows virtually any sampling strategy. However, due to the mechanics of backpropagation, it is most natural to update all layers from the last one down to some sampled minimal index. Specifically, if the smallest sampled index at iteration $k$ is $s^k$, then computing the gradient $\nabla_{s^k} f(X^k)$ requires backpropagating from the last layer $b$ up to layer $s^k$, which automatically produces all gradient components $[\nabla_{s^k} f(X^k), \ldots, \nabla_b f(X^k)]$.

Formally, we can define the sampling distribution $\cD$ as follows: at each iteration $k$, sample $s^k \in [b]$ with probabilities $p_i \eqdef \Prob{s^k=i}$, where $\sum_{i=1}^b p_i = 1$ and $p_1>0$, and set $S^k = \{s^k, \ldots, b\}$. Algorithms \ref{alg:rt_arbitrary_stoch} and \ref{alg:rt_arbitrary} then update the layers $[X^k_{s^k}, \ldots, X^k_b]$, while $[X_1^k, \ldots, X_{s^k-1}^k]$ remain frozen at their previous values (with the convention that $[X_1^k, X_0^k] = \emptyset$ when $s^k = 1$). We refer to this sampling scheme as \emph{Randomized Progressive Training} (\algnamesmall{RPT}), or \algnamesmall{Drop}-training for short (see \Cref{sec:pt_lit_rev}).

\subsection{Iteration Complexity -- Deterministic Gradient Setting}\label{sec:iter_det}

We now analyze the iteration complexity of \Cref{alg:rt_arbitrary} under two smoothness regimes; the proofs are deferred to the Appendix. Throughout, we make the standard assumption that the objective function is lower-bounded.
\begin{assumption}\label{as:lower_bound}
   There exist $f^{\star} \in \R$ such that $f(X) \geq f^{\star}$ for all $X \in \cX$.
\end{assumption}
This ensures the existence of an approximately stationary point for any desired level of accuracy.

\paragraph{Smooth case.} We first establish convergence under the \emph{layer-wise smoothness} assumption.

\begin{assumption}[$\supp(\cD)$--layer-wise smoothness]\label{as:arbitrary_layer_smoothness}
    The function $f: \cX \mapsto \R$ is $\supp(\cD)$--layer-wise $L^0$--smooth with constants $L^0 \eqdef \{(L_{1,S}^0, \ldots, L_{b,S}^0)\}_{S\in\supp(\cD)}$, $(L_{1,S}^0, \ldots, L_{b,S}^0) \in \R^b_+$, i.e., for any $S\in\supp(\cD)$,
	\begin{eqnarray*}
		\squeeze f(X + \Gamma) - f(X) - \inp{\nabla f(X)}{\Gamma}
		\leq \sum\limits_{i\in S} \frac{L_{i,S}^0}{2}\norm{\Gamma_i}_{(i)}^2.
	\end{eqnarray*}
    for all $X = [X_1, \ldots, X_b]\in \cX$ and $\Gamma = [\Gamma_1, \ldots, \Gamma_b] \in \cX$ such that $\Gamma_i = 0$ for all $i\not\in S$.
\end{assumption}
We take $L^0$ to be the smallest collection of constants satisfying the above. Throughout, we use $\supp(\cD)$ to denote the subsets of $[b]$ assigned positive probability mass by $\cD$. In the progressive training setting, these supported sets take the form  $\supp(\cD) = \{\{j,\ldots,b\}, j\in[b]\}$. By definition, $L_{i,S}^0 = 0$ whenever $i \notin S$. Moreover, if $S_1, S_2 \in \supp(\cD)$ are such that $S_1 \subseteq S_2$, then $L_{i,S_1}^0 \le L_{i,S_2}^0$ (see \Cref{lemma:smooth_subs}).

The assumption is inspired by the coordinate descent (\algnamesmall{CD}) literature \citep{wright2015coordinate}, reducing to the standard block-wise Lipschitz continuity of the gradient~\citep{rich14, fer15, qu16} in the special case when $\supp(\cD) = \{\{1\}, \{2\}, \ldots, \{b\}\}$. \Cref{as:arbitrary_layer_smoothness} captures the intuition that each subset of layers can have its own effective smoothness constant, allowing tighter bounds on the local curvature of $f$ and better reflecting the structure of the model. Importantly, it is \emph{not} more restrictive than standard smoothness--rather, it offers a richer parametric description by assigning separate constants to different layer subsets, allowing a more precise analysis without shrinking the function class.

With the assumptions in place, we are now ready to state the first formal convergence result.

\begin{theorem}\label{thm:rpt_smooth_iter}
	Let Assumptions \ref{as:lower_bound} and \ref{as:arbitrary_layer_smoothness} hold, and let $\{X^k\}_{k=0}^{K-1}$ be the iterates of \Cref{alg:rt_arbitrary} run with stepsizes $\gamma_i^k = \nicefrac{1}{L_{i,S^k}^0}$. Then
	\begin{eqnarray*}
        \squeeze \frac{1}{K} \sum\limits_{k=0}^{K-1} \sum\limits_{i=1}^b \frac{w_i}{\frac{1}{b} \sum_{j=1}^b w_j} \Exp{\norm{\nabla_i f(X^k)}^2_{(i) \star}}
		\leq \frac{f(X^0) - f^{\star}}{K \parens{\frac{1}{b} \sum_{j=1}^b w_j}},
    \end{eqnarray*}
	where $w_i \eqdef \sum_{s=1}^i \frac{p_s}{2 L_{i, \{s,\dots,b\}}^0}$.
\end{theorem}

\Cref{thm:rpt_smooth_iter} establishes an $\cO(K^{-1})$ convergence rate for a weighted sum of squared gradient component norms, matching the theoretical rates previously established for \algnamesmall{Muon}-type methods under classical smoothness \citep{li2025note, pethick2025training, kovalev2025understanding}.
For the result to be meaningful, every gradient component must contribute to the weighted average. In other words, the sampling distribution must satisfy $w_i>0$ for all $i\in[b]$.
This is a natural requirement, equivalent to ensuring that $p_1>0$, i.e., that all layers are updated with nonzero probability. Obviously, if this was not the case, the first layer would be completely ignored, making convergence impossible.

\paragraph{Generalized smooth case.}

Layer-wise optimizers considered here are designed for deep learning, where the classical smoothness assumption is often violated \citep{zhang2020why, riabinin2025gluon}. Consequently, the layer-wise smoothness model in \Cref{as:arbitrary_layer_smoothness} may not accurately capture the local geometry of the loss. To address this, we adopt a more expressive framework building upon $(L^0, L^1)$--smoothness \citep{zhang2020why, chen2023generalized}. \Cref{as:arbitrary_layer_gen_smoothness} below generalizes \Cref{as:arbitrary_layer_smoothness} by letting the local curvature of each layer depend not only on fixed constants $L_{i,S}^0$, but also on the magnitude of the layer's gradient via additional terms $L_{i,S}^1 \norm{\nabla_i f(X)}_{(i) \star}$.

\begin{assumption}[$\supp(\cD)$--layer-wise $(L^0, L^1)$--smoothness]\label{as:arbitrary_layer_gen_smoothness}
    The function $f: \cX \mapsto \R$ is $\supp(\cD)$--layer-wise $(L^0, L^1)$--smooth with constants $L^\alpha \eqdef \{(L_{1,S}^\alpha, \ldots, L_{b,S}^\alpha)\}_{S\in\supp(\cD)}$, $(L_{1,S}^\alpha, \ldots, L_{b,S}^\alpha) \in \R^b_+$, $\alpha\in\{0,1\}$, i.e., for any $S\in\supp(\cD)$,
    \begin{eqnarray*}
		\squeeze f(X + \Gamma) - f(X) - \inp{\nabla f(X)}{\Gamma}
		\leq \sum\limits_{i\in S} \frac{L_{i,S}^0 + L_{i,S}^1 \norm{\nabla_i f(X)}_{(i) \star}}{2}\norm{\Gamma_i}_{(i)}^2,
	\end{eqnarray*}
    for all $X = [X_1, \ldots, X_b]\in \cX$ and $\Gamma = [\Gamma_1, \ldots, \Gamma_b] \in \cX$ such that $\Gamma_i = 0$ for all $i\not\in S$.
\end{assumption}

As with \Cref{as:arbitrary_layer_smoothness}, assigning separate constants to each subset of layers $S$ allows for tighter, subset-specific bounds, reflecting the interactions among layers.

\begin{theorem}\label{thm:rpt_l0l1_iter}
	Let Assumptions \ref{as:lower_bound} and \ref{as:arbitrary_layer_gen_smoothness} hold, fix $\varepsilon>0$, and let $\{X^k\}_{k=0}^{K-1}$ be the iterates of \Cref{alg:rt_arbitrary} run with stepsizes $\gamma_i^k = \big(L_{i,S^k}^0 + L_{i,S^k}^1 \norm{\nabla_i f(X^k)}_{(i) \star}\big)^{-1}$.
    Then, to guarantee that
    \begin{eqnarray*}
        \squeeze \min_{k=0,\ldots,K-1} \sum\limits_{i=1}^b \left[\frac{w_i}{\frac{1}{b} \sum_{l=1}^b w_l} \Exp{\norm{\nabla _i f(X^k)}_{(i) \star}}\right] \leq \varepsilon,
    \end{eqnarray*}
    it suffices to run the algorithm for
	\begin{eqnarray*}
        \squeeze K = \left\lceil \frac{2 \delta^0 \sum\limits_{i=1}^b \frac{\parens{\sum_{s=1}^i p_s}^2 \parens{\sum_{s=1}^i p_s L^0_{i, \{s,\dots,b\}}}}{\parens{\sum_{s=1}^i p_s L^1_{i, \{s,\dots,b\}}}^2}}{\varepsilon^2 \parens{\frac{1}{b} \sum_{l=1}^b w_l}^2}
        + \frac{2 \delta^0}{\varepsilon \parens{\frac{1}{b} \sum_{l=1}^b w_l}} \right\rceil
    \end{eqnarray*}
	iterations, where $\delta^0 \eqdef f(X^0) - f^{\star}$ and $w_i \eqdef \frac{\parens{\sum_{s=1}^i p_s}^2}{\sum_{s=1}^i p_s L^1_{i, \{s,\dots,b\}}}$.
\end{theorem}
Similar to \Cref{thm:rpt_smooth_iter}, \Cref{thm:rpt_l0l1_iter} guarantees an $\cO(K^{-1/2})$ convergence rate for a weighted sum of gradient component norms.\footnote{\Cref{thm:rpt_l0l1_iter} bounds the gradient norms directly, while \Cref{thm:rpt_smooth_iter} bounds their squares; this difference naturally arises from the distinct smoothness models used in each analysis.} Again, the sampling distribution must ensure that $w_i > 0$ for all $i\in[b]$, which amounts to requiring that $p_1>0$. In the extreme case when $(p_1,p_2,\ldots,p_b) = (1,0,\ldots,0)$, corresponding to full-network training, the weights simplify to $w_i = \nicefrac{1}{L^1_{i,[b]}}$, exactly recovering the convergence rate of deterministic \algnamesmall{Gluon} \citep[Theorem 1]{riabinin2025gluon}, demonstrating the tightness of our guarantees. Importantly, the stepsizes naturally scale inversely with the layer-specific smoothness constants and gradient magnitudes. This automatic adaptation to local geometry prevents overshooting, ensures stable convergence, and allows more aggressive updates when the $(L^0, L^1)$ constants and gradient norms are small.

\subsubsection{Cost Optimization}\label{sec:cost_opt_det}

Let us now make clear why performing full-network updates in all iterations is \emph{not} in general an optimal strategy. We first consider the layer-wise smooth case. According to \Cref{thm:rpt_smooth_iter}, under \Cref{as:arbitrary_layer_smoothness}, \Cref{alg:rt_arbitrary} guarantees that $\frac{1}{K} \sum_{k=0}^{K-1} \sum_{i=1}^b \Exp{\norm{\nabla_i f(X^k)}^2_{(i) \star}} \leq \varepsilon$ after
\begin{align*}
    \squeeze K = \ceil{\frac{f(X^0) - f^{\star}}{\varepsilon \times \min_{i\in[b]} \left[\sum\limits_{s=1}^i \frac{p_s}{2 L_{i, \{s,\dots,b\}}^0}\right]}}
\end{align*}
iterations. Since in this case $\Exp{\mathrm{cost}(\hat{S})} = c_{\mathrm{ov}} + \sum_{j=1}^b (c_j + c_j^{\sharp}) \sum_{i=1}^j p_i$, the expected total cost of the optimization procedure satisfies
\begin{eqnarray}\label{eq:smooth_cost_full}
    \squeeze \mathrm{cost}_{\varepsilon}(\cD)
    \,\overset{\eqref{eq:exp_cost_gen}}{=}\, K \times \Exp{\mathrm{cost}(\hat{S})}
    \,\propto\, \frac{c_{\mathrm{ov}} + \sum_{i=1}^b (c_i + c_i^{\sharp}) \sum_{s=1}^i p_s}{\min_{i\in[b]} \left[\sum\limits_{s=1}^i \frac{p_s}{2 L_{i, \{s,\dots,b\}}^0}\right]}
\end{eqnarray}
(see \Cref{sec:rpt_sol} for the details of the derivation). Finding the sampling distribution that minimizes $\mathrm{cost}_{\varepsilon}(\cD)$ is equivalent to optimizing over $\{p_i\}_{i\in[b]}$. Letting $(p_1,p_2,\ldots,p_b) = (1,0,\ldots,0)$ recovers full-network training, which serves as a natural baseline.
Yet, can we do better? The following theorem shows that this configuration is optimal under very specific conditions only.

\begin{theorem}\label{thm:rpt_p1}
    The cost \eqref{eq:smooth_cost_full} is minimized by $(p_1, p_2, \ldots, p_b)=(1, 0, \ldots, 0)$ \textbf{if and only if} 
    \begin{equation*}
        L_{1,\{1,\ldots,b\}}=\max_{i\in[b]} L_{i,\{1,\ldots,b\}}.
    \end{equation*}
\end{theorem}
Note that the condition in \Cref{thm:rpt_p1} is entirely independent of the cost parameters! In fact, one can derive a \emph{recursive construction of the optimal probabilities} that depends solely on the smoothness constants, from which \Cref{thm:rpt_p1} follows as a simple corollary (see \Cref{rem:rpt_p1}).

The layer-wise $(L^0, L^1)$–smooth case (\Cref{thm:rpt_p1_l0l1}) leads to the same conclusion: \textbf{full-network updates are optimal only if the first layer is associated with the largest smoothness constant}--a restrictive and rarely observed condition, as confirmed by experiments on \texttt{NanoGPT} \citep{riabinin2025gluon}. While broader validation is needed, there is no reason to expect this phenomenon to hold broadly.
Overall, from the theoretical standpoint, \emph{the prevalent practice of always updating all layers is fundamentally flawed}.

\subsection{Iteration Complexity -- Stochastic Gradient Setting}\label{sec:iter_stoch}

We now turn to the convergence analysis of the practical variant of {\newalgsmall} with stochastic gradients and momentum (\Cref{alg:rt_arbitrary_stoch}), within the general layer-wise $(L^0, L^1)$--smoothness framework.

\begin{assumption}[$\supp(\cD)$--layer-wise $(L^0, L^1)$--smoothness II]\label{as:arbitrary_layer_gen_smoothness2}
    The function $f: \cX \mapsto \R$ is $\supp(\cD)$--layer-wise $(L^0, L^1)$--smooth with constants $L^\alpha \eqdef \{(L_{1,S}^\alpha, \ldots, L_{b,S}^\alpha)\}_{S\in\supp(\cD)}$, $(L_{1,S}^\alpha, \ldots, L_{b,S}^\alpha) \in \R^b_+$, $\alpha\in\{0,1\}$, i.e., for any $S\in\supp(\cD)$,
    \begin{eqnarray*}
        \norm{\nabla_i f(X + \Gamma) - \nabla_i f(X)}_{(i) \star} \leq \parens{L_{i,S}^0 + L_{i,S}^1 \norm{\nabla_i f(X)}_{(i) \star}} \norm{\Gamma_i}_{(i)}
    \end{eqnarray*}
    for all $X = [X_1, \ldots, X_b]\in \cX$ and $\Gamma = [\Gamma_1, \ldots, \Gamma_b] \in \cX$ such that $\Gamma_i = 0$ for all $i\not\in S$.
\end{assumption}

\Cref{as:arbitrary_layer_gen_smoothness2} is a slightly stronger variant of  \Cref{as:arbitrary_layer_gen_smoothness} used in the analysis of \Cref{alg:rt_arbitrary} (see \Cref{lemma:arb_layer_gen_smooth}). Consequently, all results proven under \Cref{as:arbitrary_layer_gen_smoothness} carry over to this setting. This stronger form is necessary to rigorously establish convergence in the stochastic case.

Furthermore, we assume access to a standard stochastic gradient oracle with bounded variance.

\begin{assumption}\label{as:bounded_var}
    The stochastic gradient estimator $\nabla f(\cdot; \xi): \cX \mapsto \cX$ is unbiased and has bounded variance, i.e., $\ExpSub{\xi \sim \cD}{\nabla f(X; \xi)} = \nabla f(X)$ for all $X \in \cX$ and there exist $\sigma_i \geq 0$ such that $\ExpSub{\xi \sim \cD}{\norm{\nabla_i f(X; \xi) - \nabla_i f(X)}^2_2} \leq \sigma_i^2$ for all $X \in \cX$ and $i = 1, \dots, p$.
\end{assumption}
Note that, to facilitate the proofs, the variance bound in \Cref{as:bounded_var} is defined with respect to the Euclidean norm, consistent with prior analyses \citep{pethick2025training, kovalev2025understanding, riabinin2025gluon}. This is not restrictive: since $\cX$ is finite-dimensional, there exist $\underline{\rho}_i, \bar{\rho}_i > 0$ such that $\underline{\rho}_i \norm{X_i}_{(i)} \leq \norm{X_i}_2 \leq \bar{\rho}_i \norm{X_i}_{(i)}$ for all $X_i\in\cX_i$
(equivalently, $\underline{\rho}_i \norm{X_i}_2 \leq \norm{X_i}_{(i) \star} \leq \bar{\rho}_i \norm{X_i}_2$).

\begin{theorem}\label{thm:stoch_rpt_l0l1_iter}
    Let Assumptions \ref{as:lower_bound}, \ref{as:arbitrary_layer_gen_smoothness2}, and \ref{as:bounded_var} hold. Let $\{X^k\}_{k=0}^K$, $K \geq 0$, be the iterates of \Cref{alg:rt_arbitrary_stoch} initialized with $M_i^0 = \nabla_i f(X^0; \xi^0)$ and run with $t_i^k \equiv t_i = \nicefrac{\eta_i}{(K+1)^{3/4}}$,
    where $0 < \eta_i^2 \leq \min\bigg\{\frac{1}{4} (K+1)^{1/2} \Big(\sum_{s=1}^i p_s L^1_{i,\{s,\dots,b\}}\Big)^{-1} \Big(\sum_{s=1}^b p_s \max_{i\in[b]} L^1_{i,\{s,\dots,b\}}\Big)^{-1}$, $\frac{\underline{\rho}_i p_1}{16 \bar{\rho}_i (1-\beta_i)} \Big(\sum_{s=1}^i p_s\Big)^{-1} \Big(\sum_{s=1}^i p_s L^1_{i,\{s,\dots,b\}}\Big)^{-1} \Big(\sum_{s=1}^b p_s \max_{i\in[b]} L^1_{i,\{s,\dots,b\}}\Big)^{-1}, 1\bigg\}$, and $\beta_i \equiv \beta = (K+1)^{-1/2}$.
    Then
    \small{\begin{align*}
        &\squeeze \min\limits_{k=0,\ldots,K} \sum\limits_{i=1}^b \frac{\parens{\sum_{s=1}^i p_s} \eta_i}{\frac{1}{b} \sum_{l=1}^b \sum_{s=1}^l p_s \eta_l} \Exp{\norm{\nabla_i f(X^k)}_{(i) \star}} \\
        &\leq \squeeze \frac{3 \delta^0}{(K+1)^{1/4} \frac{1}{b} \sum_{l=1}^b \sum_{s=1}^l p_s \eta_l}
        + \frac{6}{(K+1)^{1/2}} \sum\limits_{i=1}^b \frac{\eta_i \bar{\rho}_i \sigma_i}{\frac{1}{b} \sum_{l=1}^b \sum_{s=1}^l p_s \eta_l} + \sum\limits_{i=1}^b \frac{2 \bar{\rho}_i \sigma_i \parens{\sum_{s=1}^i p_s} \eta_i}{(K+1)^{1/4} \frac{1}{b} \sum_{l=1}^b \sum_{s=1}^l p_s \eta_l} \\
        &\quad\squeeze + \sum\limits_{i=1}^b \frac{2 \bar{\rho}_i \eta_i^2}{\underline{\rho}_i (K+1)^{1/4} \frac{1}{b} \sum_{l=1}^b \sum_{s=1}^l p_s \eta_l} \parens{\sum\limits_{s=1}^i p_s L^0_{i,\{s,\dots,b\}} + \parens{\sum\limits_{s=1}^i p_s L^1_{i,\{s,\dots,b\}}} \parens{\sum\limits_{s=1}^i p_s \frac{L^0_{i,\{s,\dots,b\}}}{L^1_{i,\{s,\dots,b\}}}}} \\
        &\quad\squeeze + \sum\limits_{i=1}^b \frac{\eta_i^2}{2 (K+1)^{3/4} \frac{1}{b} \sum_{l=1}^b \sum_{s=1}^l p_s \eta_l} \parens{\sum\limits_{s=1}^i p_s L^0_{i, \{s,\dots,b\}} + \parens{\sum\limits_{s=1}^i p_s L^1_{i, \{s,\dots,b\}}} \parens{\sum\limits_{s=1}^i p_s \frac{L^0_{i,\{s,\dots,b\}}}{L^1_{i,\{s,\dots,b\}}}}}.
    \end{align*}}
\end{theorem}

\Cref{thm:stoch_rpt_l0l1_iter} establishes an $\cO(K^{-1/4})$ convergence rate for a weighted sum of gradient component norms, in line with the state-of-the-art results for \algnamesmall{SGD}- and \algnamesmall{Muon}-type methods \citep{cutkosky2020momentum, sun2023momentum, kovalev2025understanding, riabinin2025gluon}.
As in the deterministic setting, one could attempt a cost analysis; however, the stochastic bound’s complexity necessitates a case-by-case treatment depending on which term dominates. While this prevents deriving as clean analytic results as in \Cref{sec:iter_stoch}, our experiments (\Cref{sec:experiments}) clearly confirm that the approach remains effective in the stochastic setting, too.

\section{Prior Work on Progressive Training}\label{sec:pt_lit_rev}

The concept of partial network updates has been explored in several prior works. Progressive Training (\algnamesmall{PT}) \citep{karras2017progressive} was first introduced in the context of Generative Adversarial Networks (GANs) \citep{goodfellow2014generative}. The central idea is to begin with a shallow network trained on low-resolution inputs and gradually increase both the input resolution and network depth by adding layers over time. This progressive growing strategy has been shown to improve computational efficiency and stabilize the training process.
Despite its empirical success, \algnamesmall{PT} has remained a largely heuristic method without formal convergence guarantees. An early attempt to provide such guarantees was made by \citet{wang2022progfed} with \algnamesmall{ProgFed}, which also extended progressive training to distributed settings. However, the accompanying theoretical analysis has been criticized as vacuous \citep{szlendak2024understanding}.
A more rigorous treatment was proposed by \citet{szlendak2024understanding}, who introduced Randomized Progressive Training (\algnamesmall{RPT}). This method can be viewed as a randomized proxy for \algnamesmall{PT} and was the first to establish theoretical convergence rates for progressive training on general smooth objectives, building on the framework of Randomized Coordinate Descent (\algnamesmall{RCD}) \citep{nesterov2012efficiency, richtarik2014iteration}.
Conceptually, \algnamesmall{RPT} is closely related to our approach: like {\newalgsmall}, it updates a random subnetwork at each iteration. However, the method of \citet{szlendak2024understanding} is simply a form of sketched \algnamesmall{GD} with a particular choice of sketch operator, and has several limitations that restrict its practical utility. First, it treats network parameters as a flat vector, ignoring the layer-wise structure. Consequently, it fails to exploit layer-specific, non-Euclidean update rules--failing, to recover methods such as \algnamesmall{Muon} as special cases. Second, \algnamesmall{RPT} requires non-stochastic gradient computation, making it computationally infeasible for large-scale architectures.

\section{Experiments}\label{sec:experiments}

We evaluate {\newalgsmall} on three benchmarks--\texttt{MNIST}, \texttt{Fashion-MNIST}, and \texttt{CIFAR-10}--using 3-layer convolutional neural networks (CNNs) of varying capacity. For all $i \in [b]$, we set $\norm{\cdot}_{(i)} = \norm{\cdot}_{2\to2}$, i.e., the spectral norm, consistent with \algnamesmall{Muon}.
We study \algnamesmall{RPT} sampling strategy with several smallest-index sampling rules, including uniform and an epoch-shift distribution that gradually shifts probability mass from shallow to deep layers as training progresses (see \Cref{sec:experiments_app}).
Our baseline is standard \algnamesmall{Muon} (i.e., {\newalgsmall} with spectral norms and full-network training) run under identical hyperparameter configurations. This design allows us to isolate the effect of partial network updates while controlling for all other factors. For each configuration, we record accuracy as a function of both epochs and wall-clock time.
Further experiments and details of hyperparameter selection procedure are provided in \Cref{sec:experiments_app}.

\paragraph{Results on \texttt{MNIST}.}

\Cref{fig:MNIST_unif_tr} illustrates a representative run with uniform layer index sampling: although {\newalgsmall} converges slightly slower per epoch (left), its lower per-epoch cost ($27$s vs. $39$s for \algnamesmall{Muon}) yields faster overall training in wall-clock time (right).
This advantage is even clearer with epoch-shift sampling: \Cref{fig:MNIST_FashionMNIST_epochshift_bar} (left) aggregates time-to-target results over multiple seeds, showing that {\newalgsmall} consistently outperforms \algnamesmall{Muon} and achieves up to $1.4\times$ speed-up. 

\begin{figure}[t]
    \centering  
    \includegraphics[width=0.8\linewidth]{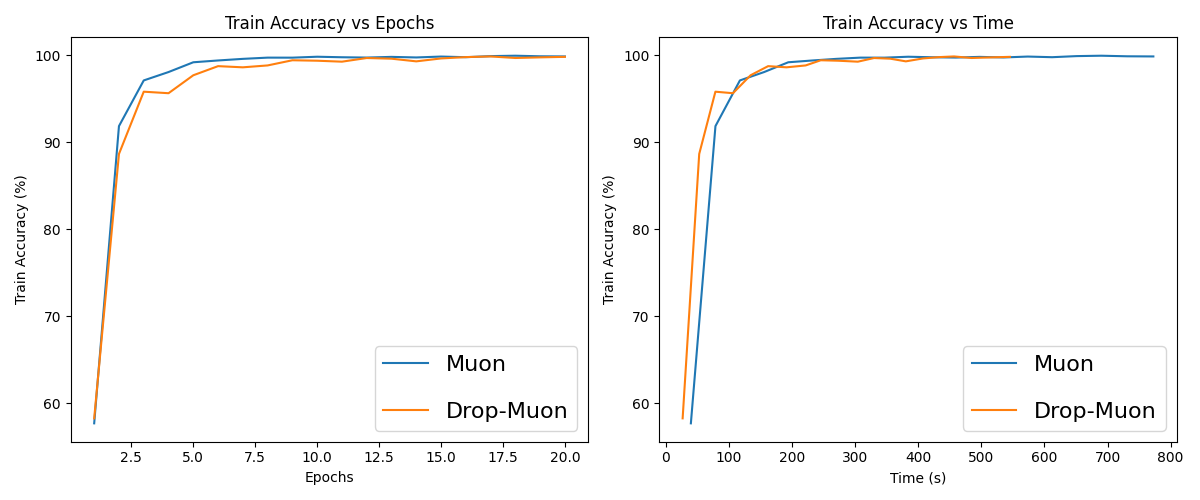}
    \caption{Evolution of the training accuracy for \algnamesmall{Muon} and {\newalgsmall} with uniform index sampling on \texttt{MNIST}. Batch size $=8192$, learning rate $=0.1$, channels $=[64,128,256]$.}
    \label{fig:MNIST_unif_tr}
\end{figure}

\paragraph{Results on \texttt{Fashion-MNIST}.}
We repeat the experiment on \texttt{Fashion-MNIST}.
{\newalgsmall} again delivers meaningful acceleration: as before, its per-epoch convergence is slightly slower, but it overtakes \algnamesmall{Muon} in wall-clock time.
\Cref{fig:MNIST_FashionMNIST_epochshift_bar} (right) shows aggregated time-to-target results over multiple seeds, with {\newalgsmall} achieving roughly $1.2\times$ faster convergence on average across accuracy thresholds (we omit the $99\%$ threshold since neither method reaches it).

\begin{figure}[t]
    \centering  
    \begin{minipage}{0.49\linewidth}
        \includegraphics[width=\linewidth]{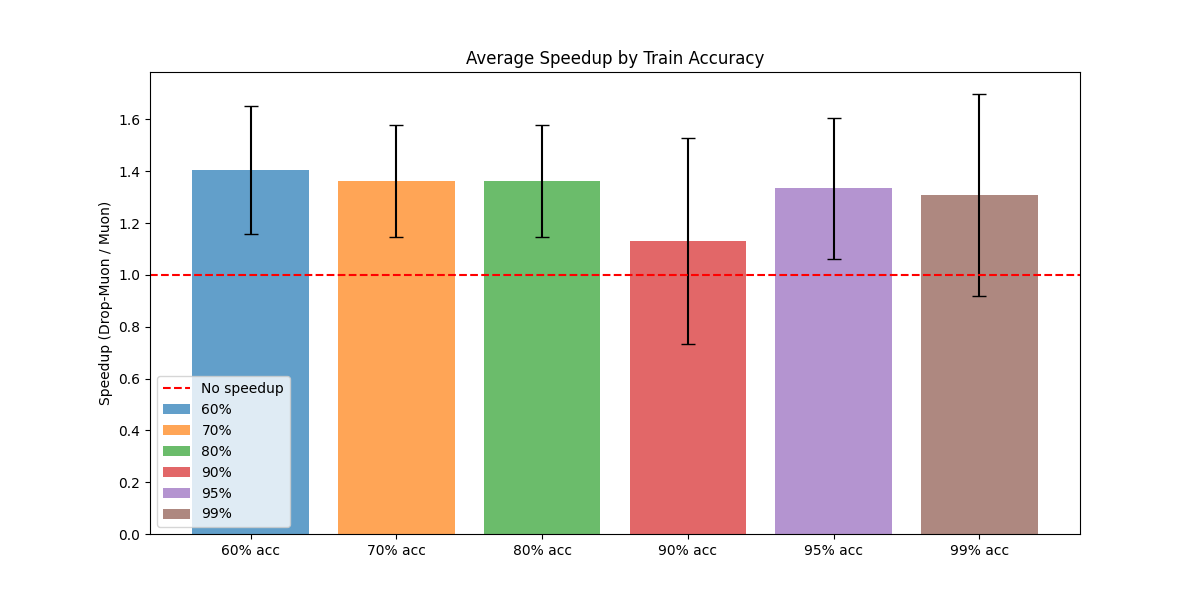}
    \end{minipage}
    \hfill
    \begin{minipage}{0.49\linewidth}
        \includegraphics[width=\linewidth]{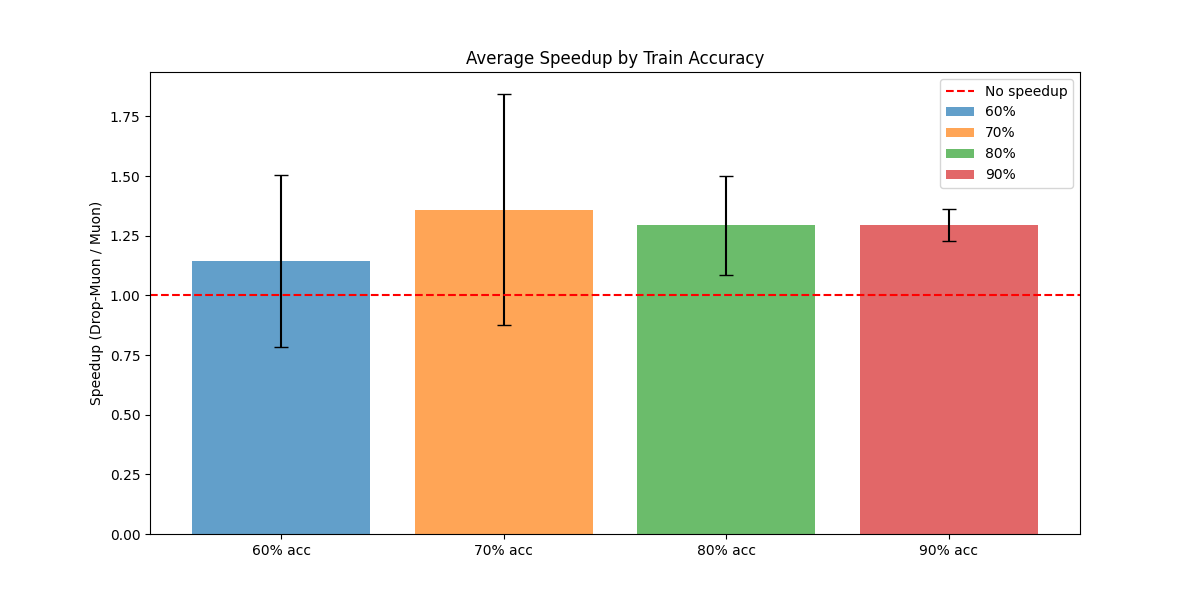}
    \end{minipage}
    \caption{Averaged time-to-target speed-up over multiple runs comparing \algnamesmall{Muon} and {\newalgsmall} with epoch-shift index sampling.
    Left: \texttt{MNIST} with batch size $=8192$, learning rate $=0.1$, and channels $=[64,128,256]$.
    Right: \texttt{Fashion-MNIST} with batch size $=32768$, learning rate $=0.1$, and channels $=[64,128,256]$.}
    \label{fig:MNIST_FashionMNIST_epochshift_bar}
\end{figure}

\paragraph{Results on \texttt{CIFAR-10}.}

\Cref{fig:cifar10_epochshift_tr} presents a representative run.
Although the absolute gain is smaller than on \texttt{MNIST} or \texttt{Fashion-MNIST}, {\newalgsmall} still reaches $90\%$ training accuracy earlier in wall-clock time (see also \Cref{fig:cifar10_epochshift_bar}).

\begin{figure}[h]
    \centering  
    \includegraphics[width=0.8\linewidth]{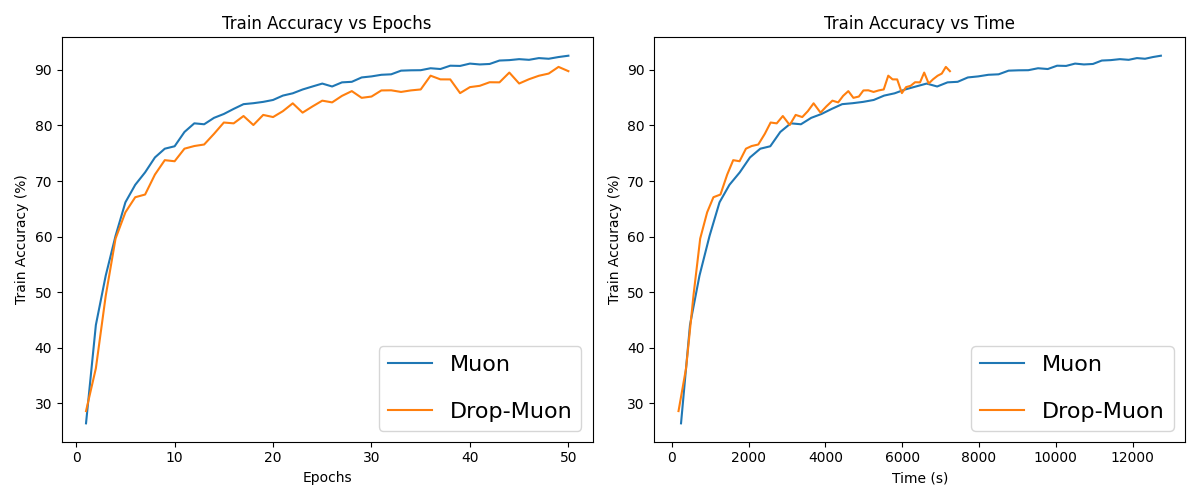}
    \caption{Evolution of the training accuracy for \algnamesmall{Muon} and {\newalgsmall} with epoch-shift index sampling on \texttt{CIFAR-10}. Batch size $=8192$, learning rate $=0.1$, channels $=[128,256,512]$.}
    \label{fig:cifar10_epochshift_tr}
\end{figure}

In summary, even though {\newalgsmall} trains on partial gradients, it reaches high accuracy levels earlier than \algnamesmall{Muon} in wall-clock time. Importantly, {\newalgsmall} is simple to implement, requiring only a few lines of code. The observed speed-up over \algnamesmall{Muon} could be even greater with dedicated tuning: in our experiments, both methods use the same constant learning rates, whereas theory (Theorems \ref{thm:rpt_smooth_iter}, \ref{thm:rpt_l0l1_iter}, and \ref{thm:stoch_rpt_l0l1_iter}) and the practice of coordinate descent (\Cref{sec:cd}) suggest that {\newalgsmall} can safely employ larger stepsizes than the full-network \algnamesmall{Muon} baseline. Here, we kept the learning rates identical for both optimizers to isolate the wall-clock benefit of partial layer updates, but additional gains are likely achievable with method-specific tuning. 
Overall, our results demonstrate that {\newalgsmall} is an effective, practical, and easily implementable strategy that, with a well-chosen sampling strategy, consistently accelerates training while retaining high accuracy, making it a compelling choice for modern neural network optimization.

\section*{Acknowledgments}
The research reported in this publication was supported by funding from King Abdullah University of Science and Technology (KAUST): i) KAUST Baseline Research Scheme, ii) CRG Grant ORFS-CRG12-2024-6460, and iii) Center of Excellence for Generative AI, under award number 5940.

\bibliographystyle{plainnat}
\bibliography{references}

\begin{thebibliography}{42}
\providecommand{\natexlab}[1]{#1}
\providecommand{\url}[1]{\texttt{#1}}
\expandafter\ifx\csname urlstyle\endcsname\relax
  \providecommand{\doi}[1]{doi: #1}\else
  \providecommand{\doi}{doi: \begingroup \urlstyle{rm}\Url}\fi

\bibitem[Bernstein et~al.(2018)Bernstein, Wang, Azizzadenesheli, and Anandkumar]{bernstein2018signsgd}
Jeremy Bernstein, Yu-Xiang Wang, Kamyar Azizzadenesheli, and Animashree Anandkumar.
\newblock sign{SGD}: Compressed optimisation for non-convex problems.
\newblock In \emph{International Conference on Machine Learning}, pages 560--569. PMLR, 2018.
\newblock URL \url{https://arxiv.org/abs/1802.04434}.

\bibitem[Bj\"{o}rck and Bowie(1971)]{bjorck1971iterative}
\r{A}. Bj\"{o}rck and C.~Bowie.
\newblock An iterative algorithm for computing the best estimate of an orthogonal matrix.
\newblock \emph{SIAM Journal on Numerical Analysis}, 8\penalty0 (2):\penalty0 358--364, 1971.
\newblock URL \url{https://doi.org/10.1137/0708036}.

\bibitem[Chen et~al.(2023{\natexlab{a}})Chen, Li, and Lu]{ChenLiLu23}
Ziang Chen, Yingzhou Li, and Jianfeng Lu.
\newblock On the global convergence of randomized coordinate gradient descent for nonconvex optimization.
\newblock \emph{SIAM Journal on Optimization}, 33\penalty0 (2):\penalty0 713--738, 2023{\natexlab{a}}.
\newblock \doi{10.1137/21M1460375}.
\newblock URL \url{https://doi.org/10.1137/21M1460375}.

\bibitem[Chen et~al.(2023{\natexlab{b}})Chen, Zhou, Liang, and Lu]{chen2023generalized}
Ziyi Chen, Yi~Zhou, Yingbin Liang, and Zhaosong Lu.
\newblock Generalized-smooth nonconvex optimization is as efficient as smooth nonconvex optimization.
\newblock In \emph{International Conference on Machine Learning}, pages 5396--5427. PMLR, 2023{\natexlab{b}}.
\newblock URL \url{https://arxiv.org/abs/2303.02854}.

\bibitem[Crawshaw et~al.(2022)Crawshaw, Liu, Orabona, Zhang, and Zhuang]{crawshaw2022robustness}
Michael Crawshaw, Mingrui Liu, Francesco Orabona, Wei Zhang, and Zhenxun Zhuang.
\newblock Robustness to unbounded smoothness of generalized sign{SGD}.
\newblock \emph{Advances in neural information processing systems}, 35:\penalty0 9955--9968, 2022.
\newblock URL \url{https://arxiv.org/abs/2208.11195}.

\bibitem[Cutkosky and Mehta(2020)]{cutkosky2020momentum}
Ashok Cutkosky and Harsh Mehta.
\newblock Momentum improves normalized {SGD}.
\newblock In \emph{International conference on machine learning}, pages 2260--2268. PMLR, 2020.
\newblock URL \url{https://arxiv.org/abs/2002.03305}.

\bibitem[Dang and Lan(2015)]{DangLan15}
Cong~D. Dang and Guanghui Lan.
\newblock Stochastic block mirror descent methods for nonsmooth and stochastic optimization.
\newblock \emph{SIAM Journal on Optimization}, 25\penalty0 (2):\penalty0 856--881, 2015.
\newblock \doi{10.1137/130936361}.
\newblock URL \url{https://doi.org/10.1137/130936361}.

\bibitem[Dinkelbach(1967)]{dinkelbach1967nonlinear}
Werner Dinkelbach.
\newblock On nonlinear fractional programming.
\newblock \emph{Management Science}, 13:\penalty0 492--498, 1967.
\newblock URL \url{https://api.semanticscholar.org/CorpusID:119939254}.

\bibitem[Fercoq and Richt{\'a}rik(2015)]{fer15}
Olivier Fercoq and Peter Richt{\'a}rik.
\newblock Accelerated, parallel, and proximal coordinate descent.
\newblock \emph{SIAM Journal on Optimization}, 25\penalty0 (4):\penalty0 1997--2023, 2015.
\newblock URL \url{https://arxiv.org/abs/1312.5799}.

\bibitem[Goodfellow et~al.(2014)Goodfellow, Pouget-Abadie, Mirza, Xu, Warde-Farley, Ozair, Courville, and Bengio]{goodfellow2014generative}
Ian~J Goodfellow, Jean Pouget-Abadie, Mehdi Mirza, Bing Xu, David Warde-Farley, Sherjil Ozair, Aaron Courville, and Yoshua Bengio.
\newblock Generative adversarial nets.
\newblock \emph{Advances in neural information processing systems}, 27, 2014.
\newblock URL \url{https://arxiv.org/abs/1406.2661}.

\bibitem[Jiang et~al.(2024)Jiang, Maladkar, and Mokhtari]{jiang2024convergence}
Ruichen Jiang, Devyani Maladkar, and Aryan Mokhtari.
\newblock Convergence analysis of adaptive gradient methods under refined smoothness and noise assumptions.
\newblock \emph{arXiv preprint arXiv:2406.04592}, 2024.
\newblock URL \url{https://arxiv.org/abs/2406.04592}.

\bibitem[Jordan et~al.(2024)Jordan, Jin, Boza, You, Cesista, Newhouse, and Bernstein]{jordan2024muon}
Keller Jordan, Yuchen Jin, Vlado Boza, Jiacheng You, Franz Cesista, Laker Newhouse, and Jeremy Bernstein.
\newblock Muon: An optimizer for hidden layers in neural networks, 2024.
\newblock URL \url{https://kellerjordan.github.io/posts/muon/}.

\bibitem[Karras et~al.(2017)Karras, Aila, Laine, and Lehtinen]{karras2017progressive}
Tero Karras, Timo Aila, Samuli Laine, and Jaakko Lehtinen.
\newblock Progressive growing of gans for improved quality, stability, and variation.
\newblock \emph{arXiv preprint arXiv:1710.10196}, 2017.
\newblock URL \url{https://arxiv.org/abs/1710.10196}.

\bibitem[Kelner et~al.(2014)Kelner, Lee, Orecchia, and Sidford]{kelner2014almost}
Jonathan~A Kelner, Yin~Tat Lee, Lorenzo Orecchia, and Aaron Sidford.
\newblock An almost-linear-time algorithm for approximate max flow in undirected graphs, and its multicommodity generalizations.
\newblock In \emph{Proceedings of the twenty-fifth annual ACM-SIAM symposium on Discrete algorithms}, pages 217--226. SIAM, 2014.
\newblock URL \url{https://arxiv.org/abs/1304.2338}.

\bibitem[Kingma and Ba(2015)]{kingma2015adam}
Diederik~P Kingma and Jimmy Ba.
\newblock Adam: A method for stochastic optimization.
\newblock In \emph{International Conference on Learning Representations}, 2015.
\newblock URL \url{https://arxiv.org/abs/1412.6980}.

\bibitem[Kovalev(2025)]{kovalev2025understanding}
Dmitry Kovalev.
\newblock Understanding gradient orthogonalization for deep learning via non-{E}uclidean trust-region optimization, 2025.
\newblock URL \url{https://arxiv.org/abs/2503.12645}.

\bibitem[Kovarik(1970)]{kovarik1970some}
Zdislav Kovarik.
\newblock Some iterative methods for improving orthonormality.
\newblock \emph{SIAM Journal on Numerical Analysis}, 7\penalty0 (3):\penalty0 386--389, 1970.
\newblock URL \url{https://doi.org/10.1137/0707031}.

\bibitem[Li et~al.(2023)Li, Qian, Tian, Rakhlin, and Jadbabaie]{li2023convex}
Haochuan Li, Jian Qian, Yi~Tian, Alexander Rakhlin, and Ali Jadbabaie.
\newblock Convex and non-convex optimization under generalized smoothness.
\newblock \emph{Advances in Neural Information Processing Systems}, 36:\penalty0 40238--40271, 2023.
\newblock URL \url{https://arxiv.org/abs/2306.01264}.

\bibitem[Li and Hong(2025)]{li2025note}
Jiaxiang Li and Mingyi Hong.
\newblock A note on the convergence of {M}uon and further, 2025.
\newblock URL \url{https://arxiv.org/abs/2502.02900}.

\bibitem[Liu et~al.(2025)Liu, Su, Yao, Jiang, Lai, Du, Qin, Xu, Lu, Yan, et~al.]{liu2025muon}
Jingyuan Liu, Jianlin Su, Xingcheng Yao, Zhejun Jiang, Guokun Lai, Yulun Du, Yidao Qin, Weixin Xu, Enzhe Lu, Junjie Yan, et~al.
\newblock Muon is scalable for {LLM} training.
\newblock \emph{arXiv preprint arXiv:2502.16982}, 2025.
\newblock URL \url{https://arxiv.org/abs/2502.16982}.

\bibitem[Liu et~al.(2024)Liu, Pan, and Zhang]{liu2024AdaGrad}
Yuxing Liu, Rui Pan, and Tong Zhang.
\newblock {AdaGrad} under anisotropic smoothness, 2024.
\newblock URL \url{https://arxiv.org/abs/2406.15244}.

\bibitem[Loshchilov and Hutter(2019)]{loshchilov2019decoupled}
Ilya Loshchilov and Frank Hutter.
\newblock Decoupled weight decay regularization.
\newblock In \emph{International Conference on Learning Representations}, 2019.
\newblock URL \url{https://arxiv.org/abs/1711.05101}.

\bibitem[{Moonshot AI}(2025)]{moonshotai2025}
{Moonshot AI}.
\newblock Kimi {K2}: Open agentic intelligence, 2025.
\newblock URL \url{https://moonshotai.github.io/Kimi-K2/}.

\bibitem[Nesterov(2012)]{nesterov2012efficiency}
Yu~Nesterov.
\newblock Efficiency of coordinate descent methods on huge-scale optimization problems.
\newblock \emph{SIAM Journal on Optimization}, 22\penalty0 (2):\penalty0 341--362, 2012.
\newblock URL \url{https://epubs.siam.org/doi/10.1137/100802001}.

\bibitem[Patrascu and Necoara(2015)]{PatrascuNecoara15}
Andrei Patrascu and Ion Necoara.
\newblock Efficient random coordinate descent algorithms for large-scale structured nonconvex optimization.
\newblock \emph{Journal of Global Optimization}, 61\penalty0 (1):\penalty0 19--46, 2015.
\newblock \doi{10.1007/s10898-014-0151-9}.
\newblock URL \url{https://doi.org/10.1007/s10898-014-0151-9}.

\bibitem[Pethick et~al.(2025)Pethick, Xie, Antonakopoulos, Zhu, Silveti-Falls, and Cevher]{pethick2025training}
Thomas Pethick, Wanyun Xie, Kimon Antonakopoulos, Zhenyu Zhu, Antonio Silveti-Falls, and Volkan Cevher.
\newblock Training deep learning models with norm-constrained {LMO}s.
\newblock \emph{arXiv preprint arXiv:2502.07529}, 2025.
\newblock URL \url{https://arxiv.org/abs/2502.07529}.

\bibitem[Powell(1973)]{powell1973search}
Michael J.~D. Powell.
\newblock On search directions for minimization algorithms.
\newblock \emph{Mathematical Programming}, 4:\penalty0 193--201, 1973.
\newblock URL \url{https://link.springer.com/article/10.1007/BF01584660}.

\bibitem[Qu and Richt{\'a}rik(2016)]{qu16}
Zheng Qu and Peter Richt{\'a}rik.
\newblock Coordinate descent with arbitrary sampling {I}: {Algorithms and complexity}.
\newblock \emph{Optimization Methods and Software}, 31\penalty0 (5):\penalty0 829--857, 2016.
\newblock URL \url{https://arxiv.org/abs/1412.8060}.

\bibitem[Riabinin et~al.(2025)Riabinin, Shulgin, Gruntkowska, and Richt{\'a}rik]{riabinin2025gluon}
Artem Riabinin, Egor Shulgin, Kaja Gruntkowska, and Peter Richt{\'a}rik.
\newblock Gluon: Making {M}uon \& {S}cion great again! ({B}ridging theory and practice of {LMO}-based optimizers for {LLM}s).
\newblock \emph{arXiv preprint arXiv:2505.13416}, 2025.
\newblock URL \url{https://arxiv.org/abs/2505.13416}.

\bibitem[Richt{\'a}rik and Tak{\'a}{\v{c}}(2014{\natexlab{a}})]{rich14}
Peter Richt{\'a}rik and Martin Tak{\'a}{\v{c}}.
\newblock Iteration complexity of randomized block-coordinate descent methods for minimizing a composite function.
\newblock \emph{Mathematical Programming}, 144\penalty0 (1-2):\penalty0 1--38, 2014{\natexlab{a}}.
\newblock URL \url{https://arxiv.org/abs/1107.2848}.

\bibitem[Richt{\'a}rik and Tak{\'a}{\v{c}}(2014{\natexlab{b}})]{richtarik2014iteration}
Peter Richt{\'a}rik and Martin Tak{\'a}{\v{c}}.
\newblock Iteration complexity of randomized block-coordinate descent methods for minimizing a composite function.
\newblock \emph{Mathematical Programming}, 144\penalty0 (1):\penalty0 1--38, 2014{\natexlab{b}}.
\newblock URL \url{https://arxiv.org/abs/1107.2848}.

\bibitem[Richt{\'a}rik and Tak{\'a}{\v{c}}(2016)]{rich16}
Peter Richt{\'a}rik and Martin Tak{\'a}{\v{c}}.
\newblock Parallel coordinate descent methods for big data optimization.
\newblock \emph{Mathematical Programming}, 156\penalty0 (1-2):\penalty0 433--484, 2016.
\newblock URL \url{https://arxiv.org/abs/1212.0873}.

\bibitem[Shah et~al.(2025)Shah, Polloreno, Stratos, Monk, Chaluvaraju, Hojel, Ma, Thomas, Tanwer, Shah, et~al.]{shah2025practical}
Ishaan Shah, Anthony~M Polloreno, Karl Stratos, Philip Monk, Adarsh Chaluvaraju, Andrew Hojel, Andrew Ma, Anil Thomas, Ashish Tanwer, Darsh~J Shah, et~al.
\newblock Practical efficiency of {M}uon for pretraining.
\newblock \emph{arXiv preprint arXiv:2505.02222}, 2025.
\newblock URL \url{https://arxiv.org/abs/2505.02222}.

\bibitem[Southwell(1940)]{southwell1940relaxation}
Richard~V. Southwell.
\newblock \emph{Relaxation Methods in Engineering Science: A Treatise on Approximate Computation}.
\newblock Read Books, 1940.

\bibitem[Sun et~al.(2023)Sun, Wang, Li, and Wang]{sun2023momentum}
Tao Sun, Qingsong Wang, Dongsheng Li, and Bao Wang.
\newblock Momentum ensures convergence of {SignSGD} under weaker assumptions.
\newblock In Andreas Krause, Emma Brunskill, Kyunghyun Cho, Barbara Engelhardt, Sivan Sabato, and Jonathan Scarlett, editors, \emph{Proceedings of the 40th International Conference on Machine Learning}, volume 202 of \emph{Proceedings of Machine Learning Research}, pages 33077--33099. PMLR, 23--29 Jul 2023.
\newblock URL \url{https://proceedings.mlr.press/v202/sun23l.html}.

\bibitem[Szlendak et~al.(2024)Szlendak, Gasanov, and Richtarik]{szlendak2024understanding}
Rafa{\l} Szlendak, Elnur Gasanov, and Peter Richtarik.
\newblock Understanding progressive training through the framework of randomized coordinate descent.
\newblock In \emph{International Conference on Artificial Intelligence and Statistics}, pages 2161--2169. PMLR, 2024.
\newblock URL \url{https://arxiv.org/abs/2306.03626}.

\bibitem[Th{\'e}rien et~al.(2025)Th{\'e}rien, Huang, Rish, and Belilovsky]{therien2025muloco}
Benjamin Th{\'e}rien, Xiaolong Huang, Irina Rish, and Eugene Belilovsky.
\newblock {MuLoCo}: {M}uon is a practical inner optimizer for {DiLoCo}.
\newblock \emph{arXiv preprint arXiv:2505.23725}, 2025.
\newblock URL \url{https://arxiv.org/abs/2505.23725}.

\bibitem[Wang et~al.(2022)Wang, Stich, He, and Fritz]{wang2022progfed}
Hui-Po Wang, Sebastian Stich, Yang He, and Mario Fritz.
\newblock {P}rog{F}ed: Effective, communication, and computation efficient federated learning by progressive training.
\newblock In Kamalika Chaudhuri, Stefanie Jegelka, Le~Song, Csaba Szepesvari, Gang Niu, and Sivan Sabato, editors, \emph{Proceedings of the 39th International Conference on Machine Learning}, volume 162 of \emph{Proceedings of Machine Learning Research}, pages 23034--23054. PMLR, 17--23 Jul 2022.
\newblock URL \url{https://proceedings.mlr.press/v162/wang22y.html}.

\bibitem[Wen et~al.(2025)Wen, Hall, Ma, and Liang]{wen2025fantastic}
Kaiyue Wen, David Hall, Tengyu Ma, and Percy Liang.
\newblock Fantastic pretraining optimizers and where to find them.
\newblock \emph{arXiv preprint arXiv:2509.02046}, 2025.
\newblock URL \url{https://arxiv.org/abs/2509.02046}.

\bibitem[Wright(2015)]{wright2015coordinate}
Stephen~J. Wright.
\newblock Coordinate descent algorithms.
\newblock \emph{Mathematical programming}, 151\penalty0 (1):\penalty0 3--34, 2015.
\newblock URL \url{https://arxiv.org/abs/1502.04759}.

\bibitem[Xie et~al.(2024)Xie, Mohamadi, and Li]{xie2024adam}
Shuo Xie, Mohamad~Amin Mohamadi, and Zhiyuan Li.
\newblock Adam exploits $\ell_{\infty}$-geometry of loss landscape via coordinate-wise adaptivity.
\newblock \emph{arXiv preprint arXiv:2410.08198}, 2024.
\newblock URL \url{https://arxiv.org/abs/2410.08198}.

\bibitem[Zhang et~al.(2020)Zhang, He, Sra, and Jadbabaie]{zhang2020why}
Jingzhao Zhang, Tianxing He, Suvrit Sra, and Ali Jadbabaie.
\newblock Why gradient clipping accelerates training: A theoretical justification for adaptivity.
\newblock In \emph{International Conference on Learning Representations}, 2020.
\newblock URL \url{https://arxiv.org/abs/1905.11881}.

\end{thebibliography}

\newpage

\appendix

\section*{Appendix}

\tableofcontents

\newpage

\section{Additional Literature Review}

\subsection{Muon (and Friends)}\label{sec:muon}

Recent work has revisited how neural network parameters are updated, moving beyond simple element-wise gradient steps toward more structured update rules. Notable examples include the \algnamesmall{Muon} optimizer \citep{jordan2024muon}, as well as the closely related \algnamesmall{Scion} \citep{pethick2025training} and \algnamesmall{Gluon} \citep{riabinin2025gluon}. All of these methods can be interpreted as instances of algorithms driven by a \emph{linear minimization oracle} (LMO) over norm balls.

\algnamesmall{Muon}, introduced by \citet{jordan2024muon}, is an optimizer for the hidden layers of neural networks (with first and last layers typically trained using \algnamesmall{AdamW} \citep{loshchilov2019decoupled}). Given a layer $X_i$ and the associated (stochastic) gradient $G_i$, the update is obtained by solving the constrained optimization problem
\begin{align}\label{eq:muon_lmo}
    \min_{\Delta X_i}\ \inp{G_i}{\Delta X_i}
    \quad \text{subject to} \quad
    \|\Delta X_i\|_{2\to 2} \le t_i,
\end{align}
where $t_i > 0$ plays the role of a trust-region radius or stepsize. The solution is characterized by the singular vectors of $G_i$: if $G_i = U_i \Sigma_i V_i^\top$ is its singular value decomposition, then 
\begin{align}\label{eq:muon_sol}
    \Delta X_i = -t_i U_i V_i^\top,
    \qquad
    X_i^{k+1} = X_i^k + \Delta X_i.
\end{align}
In other words, \algnamesmall{Muon} moves in the direction of steepest descent measured in the spectral norm.  

Since computing an exact SVD can be expensive, \algnamesmall{Muon} uses the Newton-Schulz method \citep{kovarik1970some,bjorck1971iterative} to approximate the orthogonalization. In practice, the algorithm also incorporates momentum, yielding the update
\begin{align*}
    M_i^k &= (1-\beta_i) M_i^{k-1} + \beta_i G_i^k, \\
    X_i^{k+1} &= X_i^k - t_i^k \mathrm{NewtonSchulz}(M_i^k),
\end{align*}
where $\beta_i \in (0,1]$ is the momentum parameter and $M_i^k$ is the running average of past gradients.

The update rule in \eqref{eq:muon_sol} can be interpreted as a special case of the more general LMO-based update described in \Cref{sec:intro}. By selecting different norms for the LMO ball constraint, one obtains different algorithmic variants. Since our primary focus is on deep learning applications, we are especially interested in matrix norms. Within this setting, a particularly important family is that of \emph{operator norms}, defined for any $A \in \R^{m \times n}$ as $$\norm{A}_{\alpha \to \beta} \eqdef \sup _{\norm{Z}_\alpha=1} \norm{A Z}_\beta,$$ where  $\norm{\cdot}_{\alpha}$ and $\norm{\cdot}_{\beta}$ denote norms on $\R^{n}$ and $\R^{m}$, respectively.

In particular, \algnamesmall{Muon}'s update in \eqref{eq:muon_lmo} corresponds to an LMO over the spectral norm ball $\cB_i^{2\to2}(0,t_i) \eqdef \{Z_i \in \cX_i : \norm{Z_i}_{2\to 2} \leq t_i\}$, leading to the equivalent update
\begin{align*}
    X_i^{k+1} = X_i^k + \lmo{\cB_i^{2\to2}(0,t_i)}{G_i}.
\end{align*}
Hence, \algnamesmall{Muon} is just one member of a broader family of methods parameterized by the geometry of the update set.

Building on this perspective, \citet{pethick2025training} introduced \algnamesmall{Scion}, which extends LMO-based updates to \emph{all} network layers, rather than being limited to the hidden matrix-shaped layers as in \algnamesmall{Muon}. To better capture layer-specific behavior, \algnamesmall{Scion} employs different norms depending on layer type (scaled spectral norms for the weight matrices of transformer blocks and the $\norm{\cdot}_{1\to \infty}$ norm for embedding and output layers).

Finally, \algnamesmall{Gluon} \citep{riabinin2025gluon} generalizes the theory behind both \algnamesmall{Muon} and \algnamesmall{Scion}. It provides a convergence framework for updates over \emph{arbitrary} norm balls. The analysis relies on a novel layer-wise smoothness model (closely related to our Assumptions \ref{as:arbitrary_layer_smoothness}, \ref{as:arbitrary_layer_gen_smoothness}, and \ref{as:arbitrary_layer_gen_smoothness2}), capturing heterogeneity through $(L^0,L^1)$ parameters, that more accurately reflects the non-uniform characteristics of deep learning models (see more details in \Cref{sec:gen_smooth}). As such, \algnamesmall{Gluon} can be seen as the theoretical foundation unifying the entire class of LMO-based optimizers for training deep networks.

\subsection{Generalized Smoothness}\label{sec:gen_smooth}

Gradient-based methods are traditionally analyzed under the assumption of Lipschitz smoothness of the gradient. Yet, many modern deep learning tasks violate this assumption \citep{zhang2020why, crawshaw2022robustness}, prompting the development of alternative smoothness models. One such model is $(L^0, L^1)$--smoothness, initially introduced by \citet{zhang2020why} for twice continuously differentiable functions and later generalized beyond this setting \citep{li2023convex, chen2023generalized}.

This framework has been further tailored to deep learning through the non-Euclidean layer-wise $(L^0, L^1)$–smoothness assumption \citep{riabinin2025gluon}, which closely relates to our Assumptions \ref{as:arbitrary_layer_smoothness}, \ref{as:arbitrary_layer_gen_smoothness}, and \ref{as:arbitrary_layer_gen_smoothness2}. Concretely, \citet{riabinin2025gluon} assume that
\begin{align}\label{eq:oagsbv}
   \norm{\nabla _i f(X) - \nabla _i f(Y)}_{(i) \star} \leq \parens{L^0_i + L^1_i \| \nabla _i f(X) \|_{(i) \star}} \norm{X_i - Y_i}_{(i)}
\end{align}
for all $i\in[b]$ and all $X = [X_1, \ldots, X_p]\in \cX$, $Y = [Y_1, \ldots, Y_p] \in \cX$. Condition \eqref{eq:oagsbv} can be interpreted as a global version of our \Cref{as:arbitrary_layer_gen_smoothness2}. However, it is less precise because it does not allow each subset of layers to have its own effective smoothness constants. By adopting \Cref{as:arbitrary_layer_gen_smoothness2} instead of \eqref{eq:oagsbv}, we can derive tighter bounds without restricting the function class under consideration.

Accounting for heterogeneous parameter structures is not novel and has been studied in the context of coordinate descent \citep{nesterov2012efficiency, richtarik2014iteration} and in the analysis of algorithms such as \algnamesmall{signSGD} \citep{bernstein2018signsgd, crawshaw2022robustness}, \algnamesmall{AdaGrad} \citep{jiang2024convergence, liu2024AdaGrad}, and \algnamesmall{Adam} \citep{xie2024adam}.

\subsection{Randomized Block Coordinate Descent}\label{sec:cd}

\begin{algorithm}[t]
    \begin{algorithmic}[1]
    \State {\bf Input:} Initial iterate $X^0 = [X_1^0, \ldots, X_b^0] \in \cX$; stepsizes $\gamma_i^k > 0$, $i\in[b]$, $k\geq0$
    \For{$k = 0, \dots, K-1$}
        \State Choose index $i_k\in [b]$ randomly \Comment{Serial sampling}
        \State Set $X_i^{k+1} = X_i^k$ for $i\neq i_k$
        \State Update
        $
        X_{i_k}^{k+1} = X_{i_k}^k - \gamma_{i_k}^k \nabla_{i_k} f(X^k)
        $
    \EndFor
    \caption{Randomized Block Coordinate Descent}
    \label{alg:bcd}
    \end{algorithmic}
\end{algorithm}

Coordinate Descent (\algnamesmall{CD}) algorithms have a long history \citep{southwell1940relaxation, powell1973search, nesterov2012efficiency, wright2015coordinate} and are among the most widely studied methods for large-scale optimization. Their randomized variants, known as Randomized Block Coordinate Descent (\algnamesmall{RBCD}), have emerged as a powerful class of first-order methods, particularly in high-dimensional problems. Instead of updating the entire parameter vector, \algnamesmall{RBCD} updates only a subset (``block'') of variables at each iteration, reducing per-iteration computational cost and improving scalability.

In the context of \algnamesmall{CD}, the parameters are viewed as a flat vector in $\R^d$. 
The general procedure of \algnamesmall{RBCD} is summarized in \Cref{alg:bcd}. At each iteration, one block is selected at random and updated, while the remaining blocks remain unchanged. 
Clearly, \Cref{alg:bcd} is a special case of \Cref{alg:rt_arbitrary} with the \emph{serial sampling} strategy (i.e., sampling a single coordinate per iteration) and standard Euclidean updates.

We define a family of linear maps $\{\mathbf{U}_i: \cX_i\rightarrow \cX : i\in [b]\}$ (where $\cX=\R^d$ and $\cX_i=\R^{d_i}$, with $\sum_{i=1}^b d_i = d$), forming a partition of the identity map, i.e., $X = \sum_{i\in [b]} \mathbf{U}_i \left(X_i\right)$.
The standard analysis of \algnamesmall{RBCD} relies on the assumption of \emph{blockwise Lipschitz continuity} of the gradient:
\begin{equation}\label{eq:sdfewr}
\left\|
\nabla_i f(X+\mathbf{U}_i (\Gamma_i)) - \nabla_i f(X)
\right\|_{(i)\star}\leq L_i \|\Gamma_i\|_{(i)},\quad \forall i\in [b], \,  X \in \cX, \, \Gamma_i \in \cX_i.
\end{equation}
From~\eqref{eq:sdfewr}, one obtains the standard quadratic upper bound
\begin{equation}\label{eq:fnf}
    f(X+ \mathbf{U}_i (\Gamma_i))-f(X)- \left\langle \nabla_i f(X), \Gamma_i \right\rangle \leq \frac{L_i}{2} \|\Gamma_i\|^2_{(i)}.
\end{equation}
This inequality plays a pivotal role in establishing the iteration complexity of \algnamesmall{RBCD}. Moreover, it explains why \algnamesmall{CD}-type methods can be advantageous compared to full gradient descent: they allow for larger stepsizes, which in turn leads to faster convergence. Indeed, the standard analysis of \algnamesmall{GD} relies on \emph{global} $L$-smoothness, enforcing stepsizes on the order of $\nicefrac{1}{L}$. By contrast, the blockwise smoothness assumption \eqref{eq:sdfewr} bounds the function's variation when only a single block of coordinates is updated. Consequently, in \algnamesmall{CD}, the safe stepsize for block $i$ scales with $\nicefrac{1}{L_i}$. Since $L$ must capture worst-case variations across all directions (including cross-block interactions), it is often much larger than the individual $L_i$'s.

Complexity analyses of \algnamesmall{RBCD} in the convex setting can be found in~\citet{nesterov2012efficiency,wright2015coordinate}. Generalizations to parallel and arbitrary sampling variants were developed by~\citet{rich16,qu16}. 
For the nonconvex case considered in this paper, iteration complexity results of \Cref{alg:rt_arbitrary} are only known for serial sampling~\citep{PatrascuNecoara15,DangLan15}. Otherwise, some work has focused on the global convergence properties of randomized \algnamesmall{CD} methods~\citep{ChenLiLu23}, rather than the iteration complexity analysis we develop here. It is worth noting that serial sampling is not suitable in the deep neural network context due to the special entanglement structure and the requirements of backpropagation for gradient computation.

\newpage

\section{More on Smoothness Assumptions}

Recall that a function $f: \cX \mapsto \R$ is $\supp(\cD)$--layer-wise $L^0$--smooth with constants $L^0 \eqdef \{(L_{1,S}^0, \ldots, L_{b,S}^0)\}_{S\in\supp(\cD)}$, $(L_{1,S}^0, \ldots, L_{b,S}^0) \in \R^b_+$ (\Cref{as:arbitrary_layer_smoothness}) if for any $S\in\supp(\cD)$,
\begin{eqnarray*}
    f(X + \Gamma) - f(X) - \inp{\nabla f(X)}{\Gamma}
    \leq \sum_{i\in S} \frac{L_{i,S}^0}{2}\norm{\Gamma_i}_{(i)}^2
\end{eqnarray*}
for all $X = [X_1, \ldots, X_b]\in \cX$ and $\Gamma = [\Gamma_1, \ldots, \Gamma_b] \in \cX$ such that $\Gamma_i = 0$ for all $i\not\in S$.

\begin{lemma}\label{lemma:smooth_subs}
    Let $S_1, S_2 \in \supp(\cD)$ satisfy $S_1 \subseteq S_2$, and suppose \Cref{as:arbitrary_layer_smoothness} holds. 
    Then one can choose the constants $\{L_{i,S_1}^0\}_{i\in S_1}$ so that
    \begin{eqnarray*}
        L_{i,S_1}^0 \leq L_{i,S_2}^0.
    \end{eqnarray*}
\end{lemma}

\begin{proof}
    Let $X = [X_1, \ldots, X_b]\in \cX$ be arbitrary. By \Cref{as:arbitrary_layer_smoothness}, for any $\Gamma = [\Gamma_1, \ldots, \Gamma_b] \in \cX$ such that $\Gamma_i = 0$ for all $i\not\in S_1$, we have
    \begin{eqnarray*}
		f(X + \Gamma) - f(X) - \inp{\nabla f(X)}{\Gamma}
		\leq \sum_{i\in S_1} \frac{L_{i,S_1}^0}{2}\norm{\Gamma_i}_{(i)}^2.
	\end{eqnarray*}
    Now, since $S_1 \subseteq S_2$ and $\Gamma_i = 0$ for all $i \in S_2\setminus S_1$, \Cref{as:arbitrary_layer_smoothness} also gives
    \begin{eqnarray*}
		f(X + \Gamma) - f(X) - \inp{\nabla f(X)}{\Gamma}
		\leq \sum_{i\in S_2} \frac{L_{i,S_2}^0}{2}\norm{\Gamma_i}_{(i)}^2
		= \sum_{i\in S_1} \frac{L_{i,S_2}^0}{2}\norm{\Gamma_i}_{(i)}^2.
	\end{eqnarray*}
    Therefore, one can always choose $L_{i,S_1}^0 \leq L_{i,S_2}^0$.
\end{proof}

\begin{remark}
    We can show that, without further assumptions, there is no finite function $C: [b] \to (0, \infty)$ such that 
    \begin{eqnarray}\label{eq:oaivn}
        L_{i,S}^0 \leq C(|S|) L_{i,\{i\}}^0
    \end{eqnarray}
    holds for all $f$.
    Consider the case $b=2$ with scalar blocks and the Euclidean norm, suppose that $\supp(\cD) = \brac{\{1\}, \{2\}, \{1,2\}}$, and define
    \begin{eqnarray*}
        f(x,y) = \alpha x y
    \end{eqnarray*}
    for some $\alpha > 0$. For a single-block perturbation in the $x$-coordinate only, we have
    \begin{eqnarray*}
        f(x+\gamma_1,y) - f(x,y) - \partial_x f(x,y) \gamma_1
        = \alpha (x+\gamma_1) y - \alpha x y - \alpha \gamma_1 y
        = 0,
    \end{eqnarray*}
    which shows that $L_{1,\{1\}}^0 = 0$. By symmetry, $L_{2,\{2\}}^0 = 0$ as well. However, for a joint perturbation $\Gamma = (\gamma_1,\gamma_2)$, we obtain
    \begin{eqnarray*}
        f(x+\gamma_1, y+\gamma_2) - f(x,y) - \inp{\nabla f(x,y)}{\Gamma}
        &=& \alpha (x+\gamma_1)(y+\gamma_2) - \alpha x y - \alpha \gamma_1 y - \alpha \gamma_2 x \\
        &=& \alpha \gamma_1 \gamma_2.
    \end{eqnarray*}
    Now, suppose for a contradiction that \eqref{eq:oaivn} holds for some function $C:[b]\to\R$. Then, since $L_{1,\{1\}}^0 = 0$, it must be that $L_{1,\{1,2\}}^0 = 0$. By \Cref{as:arbitrary_layer_smoothness}, this requires
    \begin{eqnarray*}
        \alpha \gamma_1 \gamma_2
        \leq \frac{L_{1,\{1,2\}}^0}{2} \gamma_1^2 + \frac{L_{2,\{1,2\}}^0}{2} \gamma_2^2
        = \frac{L_{2,\{1,2\}}^0}{2} \gamma_2^2.
    \end{eqnarray*}
    But the right-hand side is independent of $\gamma_1$, whereas the left-hand side grows without bound as $\gamma_1\to\infty$ whenever $\gamma_2\neq 0$. Hence no such constant $L_{2,\{1,2\}}^0$ can exist. This proves that no uniform bound of the form \eqref{eq:oaivn} can hold without additional assumptions.
\end{remark}

\begin{remark}
    The above argument extends to any two sets with one contained in the other (not just singletons). Let $S_1, S_2 \in \supp(\cD)$ satisfy $S_1 \subseteq S_2$. We show that, without further assumptions, there is no finite function $C:[b]\times[b]\to(0,\infty)$ such that
    \begin{eqnarray}\label{eq:qwioavnf}
        L_{i,S_2}^0 \leq C(|S_1|,|S_2|) L_{i,S_1}^0.
    \end{eqnarray}    
    Without loss of generality we can restrict attention to the coordinates indexed by $S_2$ (so that $S_2=[b]$). Let $S_1 = \{1,\dots,s\}$ and $S_2 \setminus S_1 = \{s+1,\dots,s+u\}$, so $|S_1|=s$, $|S_2|=s+u$, and $|S_2 \setminus S_1|=u$. Again, consider scalar blocks and define a function $f: \R^b \to \R$ via
    \begin{eqnarray*}
        f(x_1,\dots,x_s,y_{1},\dots,y_{u}) = \alpha \parens{\sum_{i=1}^s x_i} \parens{\sum_{j=1}^u y_j}
    \end{eqnarray*}
    for some $\alpha > 0$.
    For any increment $\Gamma_{S_1}$ supported on $S_1$ only (i.e., $\gamma_j=0$ for all $j \not\in S_1$) the function~$f$ is affine in the $x$-variables with coefficients given by $\alpha \sum_j y_j$. Therefore, for any $X\in\R^b$,
    \begin{eqnarray*}
        &&\hspace{-8mm}f(X+\Gamma_{S_1})-f(X)-\inp{\nabla f(X)}{\Gamma_{S_1}} \\
        &=& \alpha \parens{\sum_{i=1}^s (x_i + \gamma_i)} \parens{\sum_{j=1}^u y_j}
        - \alpha \parens{\sum_{i=1}^s x_i} \parens{\sum_{j=1}^u y_j}
        - \alpha \sum_{i=1}^s \parens{\sum_{j=1}^u y_j} \gamma_i
        = 0,
    \end{eqnarray*}
    and hence $L_{i,S_1}^0 = 0$ for every $i\in S_1$.
    For a joint perturbation $\Gamma_{S_2}=(\gamma_1,\dots,\gamma_s,\eta_1,\dots,\eta_u)$ supported on $S_2$, we have
    \begin{eqnarray*}
        &&\hspace{-8mm}f(X+\Gamma_{S_2})-f(X)-\inp{\nabla f(X)}{\Gamma_{S_2}} \\
        &=& \alpha \parens{\sum_{i=1}^s (x_i + \gamma_i)} \parens{\sum_{j=1}^u (y_j + \eta_j)}
        - \alpha \parens{\sum_{i=1}^s x_i} \parens{\sum_{j=1}^u y_j} \\
        &&- \alpha \sum_{i=1}^s \parens{\sum_{j=1}^u y_j} \gamma_i
        - \alpha \sum_{i=1}^u \parens{\sum_{j=1}^s x_j} \eta_i \\
        &=& \alpha \parens{\sum_{i=1}^s \gamma_i} \parens{\sum_{j=1}^u \eta_j}.
    \end{eqnarray*}
    In particular, fix any nonzero vectors $\tilde{\gamma} =(\tilde{\gamma}_1,\dots,\tilde{\gamma}_s)$ and $\eta = (\eta_1,\dots,\eta_u)$ and let the $x$-perturbation scale by a factor $\lambda$, i.e. $\gamma_i=\lambda \tilde{\gamma}_i$. Then
    \begin{eqnarray*}
        f(X+\Gamma)-f(X)-\inp{\nabla f(X)}{\Gamma}
        = \alpha \lambda \parens{\sum_{i=1}^s \tilde{\gamma}_i} \parens{\sum_{j=1}^u \eta_j},
    \end{eqnarray*}
    which grows linearly in $\lambda$.
    Now, suppose, for contradiction that \eqref{eq:qwioavnf} holds with some finite $C(|S_1|,|S_2|)$. Since $L_{i,S_1}^0=0$ for every $i\in S_1$, the inequality forces $L_{i,S_2}^0=0$ for all $i\in S_1$. But then \Cref{as:arbitrary_layer_smoothness} would yield
    \begin{eqnarray*}
        \alpha \lambda \parens{\sum_{i=1}^s \tilde{\gamma}_i} \parens{\sum_{j=1}^u \eta_j}
        &=& \alpha \parens{\sum_{i=1}^t \gamma_i} \parens{\sum_{j=1}^u \eta_j} \\
        &\leq& \sum_{i\in S_1} \frac{L_{i,S_2}^0}{2} \gamma_i^2
        + \sum_{i\in S_2 \setminus S_1} \frac{L_{i,S_2}^0}{2} \eta_i^2 \\
        &=& \sum_{i\in S_2 \setminus S_1} \frac{L_{i,S_2}^0}{2} \eta_i^2,
    \end{eqnarray*}
    where the right-hand side is independent of the scaling factor $\lambda$. Taking $\lambda\to\infty$ makes the left-hand side arbitrarily large, resulting in a contradiction. Hence no such $C$ can exist.
\end{remark}

\newpage

\section{Arbitrary Sampling}\label{sec:arbitrary_sampling}

One might wonder why the cost model in \Cref{sec:cost_model} distinguishes between two sets of constants $\{c_i\}_{i\in[b]}$ and $\{c_i^\sharp\}_{i\in[b]}$, given that in the \algnamesmall{RPT} case, $c_i$ and $c_i^\sharp$ could be combined into a single constant. This is because Algorithms \ref{alg:rt_arbitrary_stoch} and \ref{alg:rt_arbitrary} do not restrict us to the layer sampling scheme studied in \Cref{sec:rpt}. One could argue that deviating from this progressive training framework is inefficient, as it discards gradients obtained ``for free'' during backpropagation: by design, at iteration~$k$, the gradients $[\nabla_{s^k} f(X^k), \ldots, \nabla_b f(X^k)]$, where $s^k \eqdef \min S^k$, are necessarily computed, and thus available for the update (if, for some block $i \in [b]$, the cost $c_i^{\sharp}$ of applying the sharp operator is large, one can simply bypass this step and instead use the standard gradient update; this corresponds to setting $\norm{\cdot}_{(i)}=\norm{\cdot}_2$ in the algorithm). Nevertheless, since the general iteration complexity results in Theorems \ref{thm:rt_smooth_iter}, \ref{thm:rt_l0l1_iter} and \ref{thm:stoch_rt_l0l1_iter} hold for any distribution~$\cD$, the framework naturally allows experimentation with alternative sampling schemes, even if only out of theoretical curiosity.

With this in mind, we generalize the results from \Cref{sec:rpt} to \emph{arbitrary} layer samplings.
Formally, at each iteration $k$, Algorithms \ref{alg:rt_arbitrary_stoch} and \ref{alg:rt_arbitrary} sample a random subset of layers $S^k \subseteq [b]$ from a distribution
\begin{align*}
	\cD: \mathfrak{P}([b]) \to [0,1], \qquad \sum_{S\subseteq[b]} \cD(S) = 1,
\end{align*}
with support $\supp(\cD) \eqdef \{S\in\mathfrak{P}([b]): \cD(S)>0\}$, where $\mathfrak{P}([b])$ denotes the power set of $[b]$. Only the layers in $S^k$ are updated, while the rest remain fixed. We write $\hat{S} \sim \cD$ for a set-valued random variable with the same distribution as $S^k$.
This formulation defines a broad family of algorithms, parameterized jointly by the sampling distribution $\cD$ and the choice of layer-wise norms~$\norm{\cdot}_{(i)}$.

Mirroring the earlier discussion, we organize the analysis around two smoothness regimes--layer-wise smooth (\Cref{as:arbitrary_layer_smoothness}) and generalized layer-wise smooth (Assumptions \ref{as:arbitrary_layer_gen_smoothness} and \ref{as:arbitrary_layer_gen_smoothness2}).

In what follows, we present convergence results for Algorithms \ref{alg:rt_arbitrary_stoch} (\Cref{thm:stoch_rt_l0l1_iter}) and \ref{alg:rt_arbitrary} (Theorems \ref{thm:rt_smooth_iter} and \ref{thm:rt_l0l1_iter}). When specialized to the \algnamesmall{RPT} setting, these general results recover the guarantees stated in Theorems \ref{thm:rpt_smooth_iter}, \ref{thm:rpt_l0l1_iter}, and \ref{thm:stoch_rpt_l0l1_iter}, as detailed in Remarks \ref{rem:rpt_smooth_iter}, \ref{rem:rpt_l0l1_iter}, and \ref{rem:stoch_rpt_l0l1_iter}.

\newpage

\section{Convergence Results -- Deterministic Gradient Case}

\subsection{Layer-Wise Smooth Case}

\begin{theorem}\label{thm:rt_smooth_iter}
	Let Assumptions \ref{as:lower_bound} and \ref{as:arbitrary_layer_smoothness} hold, and let $\{X^k\}_{k=0}^{K-1}$ be the iterates of \Cref{alg:rt_arbitrary} run with stepsizes $\gamma_i^k = \nicefrac{1}{L_{i,S^k}^0}$. Then
	\begin{eqnarray*}
        \squeeze \frac{1}{K} \sum\limits_{k=0}^{K-1} \sum\limits_{i=1}^b \frac{w_i}{\frac{1}{b} \sum_{j=1}^b w_j} \Exp{\norm{\nabla_i f(X^k)}^2_{(i) \star}}
		\leq \frac{f(X^0) - f^{\star}}{K \parens{\frac{1}{b} \sum_{j=1}^b w_j}},
    \end{eqnarray*}
	where $w_i \eqdef \Exp{\frac{\I{i\in \hat{S}}}{2 L_{i,\hat{S}}^0}}$.
\end{theorem}

\begin{remark}
    For \Cref{thm:rt_smooth_iter} to be meaningful, the sampling must ensure that $w_i > 0$ for every $i\in[b]$, i.e.,
    \begin{align*}
        \Exp{\frac{\I{i\in \hat{S}}}{2 L_{i,\hat{S}}^0}} > 0.
    \end{align*}
    Otherwise, some layers would receive zero weight and the bound would provide no control over them. This is a natural condition, requiring that every layer is sampled with positive probability. Indeed, if $\Prob{i\in \hat{S}} = 0$, then the expectation vanishes, so $w_i>0$ necessarily implies $\Prob{i\in \hat{S}} > 0$ for all $i\in[b]$.
\end{remark}

\begin{remark}\label{rem:rpt_smooth_iter}
    In the case of \algnamesmall{RPT} (see \Cref{sec:rpt}), the weights from \Cref{thm:rt_smooth_iter} become
    \begin{eqnarray*}
        \Exp{\frac{\I{i\in \hat{S}}}{2 L_{i,\hat{S}}^0}}
        = \sum_{s=1}^b \frac{\I{i \in \{s,\dots,b\}}}{2 L_{i, \{s,\dots,b\}}^0} p_s
        = \sum_{s=1}^i \frac{p_s}{2 L_{i, \{s,\dots,b\}}^0}.
    \end{eqnarray*}
    Substituting this expression into the rate yields the result in \Cref{thm:rpt_smooth_iter}.
\end{remark}

\begin{proof}[Proof of \Cref{thm:rt_smooth_iter}]
	Using \Cref{as:arbitrary_layer_smoothness}, we get
    \begin{eqnarray*}
        &&\hspace{-10mm}f(X^{k+1}) \\
        &\leq& f(X^k) + \inp{\nabla f(X^k)}{X^{k+1} - X^k} + \sum_{i\in S^k} \frac{L^0_{i,S^k}}{2} \norm{X_i^k - X_i^{k+1}}_{(i)}^2 \\
        &=& f(X^k) + \sum_{i\in S^k} \parens{\inp{\nabla_i f(X^k)}{X_i^{k+1} - X_i^k}_{(i)} + \frac{L^0_{i,S^k}}{2} \norm{X_i^k - X_i^{k+1}}_{(i)}^2} \\
        &=& f(X^k) + \sum_{i\in S^k} \parens{- \gamma_i^k \inp{\nabla_i f(X^k)}{\parens{\nabla_i f(X^k)}^{\sharp}}_{(i)} + \frac{L^0_{i,S^k} (\gamma_i^k)^2}{2} \norm{\parens{\nabla_i f(X^k)}^{\sharp}}_{(i)}^2} \\
        &\overset{\eqref{eq:inpsharp}, \eqref{eq:normsharp}}{=}& f(X^k) + \sum_{i\in S^k} \parens{- \gamma_i^k \norm{\nabla_i f(X^k)}^2_{(i) \star} + \frac{L^0_{i,S^k} (\gamma_i^k)^2}{2} \norm{\nabla_i f(X^k)}_{(i) \star}^2}.
    \end{eqnarray*}
	Choosing $\gamma_i^k = \nicefrac{1}{L^0_{i,S^k}}$ and rearranging, we get
	\begin{eqnarray*}
        \sum_{i\in S^k} \frac{\norm{\nabla_i f(X^k)}^2_{(i) \star}}{2 L^0_{i,S^k}}
		\leq f(X^k) - f(X^{k+1}).
    \end{eqnarray*}
	Taking expectation conditional on $X^k$ (denoted as $\ExpSub{k}{\cdot}$) gives
	\begin{eqnarray*}
        f(X^k) - \ExpSub{k}{f(X^{k+1})}
		&\geq& \ExpSub{k}{\sum_{i\in S^k} \frac{\norm{\nabla_i f(X^k)}^2_{(i) \star}}{2 L^0_{i,S^k}}} \\
		&=& \ExpSub{k}{\sum_{i=1}^b \I{i\in S^k} \frac{\norm{\nabla_i f(X^k)}^2_{(i) \star}}{2 L^0_{i,S^k}}} \\
		&=& \sum_{i=1}^b \Exp{\frac{\I{i\in \hat{S}}}{2 L_{i,\hat{S}}^0}} \norm{\nabla_i f(X^k)}^2_{(i) \star} \\
		&=& \sum_{i=1}^b w_i \norm{\nabla_i f(X^k)}^2_{(i) \star},
    \end{eqnarray*}
	where we denoted $w_i \eqdef \Exp{\frac{\I{i\in \hat{S}}}{2 L_{i,\hat{S}}^0}}$ and $\I{\cdot}$ is the indicator function (i.e., for any event $A$, $\I{A} = 1$ if $A$ and $\I{A} = 0$ otherwise).
	Taking full expectation, summing over the first $K$ iterations and dividing by $\frac{K}{b} \sum_{j=1}^b w_j$, we get
	\begin{align*}
        \frac{1}{K} \sum_{k=0}^{K-1} \sum_{i=1}^b \frac{w_i}{\frac{1}{b} \sum_{j=1}^b w_j} \Exp{\norm{\nabla_i f(X^k)}^2_{(i) \star}}
		&\leq \frac{1}{\frac{K}{b} \sum_{j=1}^b w_j} \sum_{k=0}^{K-1} \parens{\Exp{f(X^k)} - \Exp{f(X^{k+1})}} \\
		&\leq \frac{f(X^0) - f^{\star}}{K \parens{\frac{1}{b} \sum_{j=1}^b w_j}}.
    \end{align*}
\end{proof}

\subsection{Layer-Wise Generalized Smooth Case}

\begin{theorem}\label{thm:rt_l0l1_iter}
	Let Assumptions \ref{as:lower_bound} and \ref{as:arbitrary_layer_gen_smoothness} hold, fix $\varepsilon>0$, and let $\{X^k\}_{k=0}^{K-1}$ be the iterates of \Cref{alg:rt_arbitrary} run with stepsizes $\gamma_i^k = \parens{L_{i,S^k}^0 + L_{i,S^k}^1 \norm{\nabla_i f(X^k)}_{(i) \star}}^{-1}$.
    Then, to guarantee that
    \begin{eqnarray*}
        \squeeze \min_{k=0,\ldots,K-1} \sum\limits_{i=1}^b \left[\frac{w_i}{\frac{1}{b} \sum_{l=1}^b w_l} \Exp{\norm{\nabla _i f(X^k)}_{(i) \star}}\right] \leq \varepsilon,
    \end{eqnarray*}
    it suffices to run the algorithm for
	\begin{eqnarray*}
        \squeeze K = \left\lceil \frac{2 \delta^0 \sum\limits_{i=1}^b \frac{\Prob{i\in \hat{S}} \ExpCond{L^0_{i,\hat{S}}}{i\in \hat{S}}}{\parens{\ExpCond{L^1_{i,\hat{S}}}{i\in \hat{S}}}^2}}{\varepsilon^2 \parens{\frac{1}{b} \sum_{l=1}^b w_l}^2}
        + \frac{2 \delta^0}{\varepsilon \parens{\frac{1}{b} \sum_{l=1}^b w_l}} \right\rceil
    \end{eqnarray*}
	iterations, where $\delta^0 \eqdef f(X^0) - f^{\star}$ and $w_i \eqdef \frac{\Prob{i\in \hat{S}}}{\ExpCond{L^1_{i,\hat{S}}}{i\in \hat{S}}}$.
\end{theorem}

\begin{remark}\label{rem:rt_l0l1_wi}
    The guarantee in \Cref{thm:rt_l0l1_iter} is meaningful only if $w_i > 0$ for all $i\in[b]$, which is equivalent to requiring that $\Prob{i\in \hat{S}} > 0$ for all $i\in[b]$.
\end{remark}

\begin{remark}\label{rem:rpt_l0l1_iter}
    For the \algnamesmall{RPT} case (see \Cref{sec:rpt}), we have
    \begin{eqnarray*}
    	\Prob{i\in \hat{S}}
    	= \Prob{s \leq i}
    	= \sum_{s=1}^i p_s
    \end{eqnarray*}
    and
    \begin{eqnarray*}
    	\ExpCond{L^{\alpha}_{i,\hat{S}}}{i\in \hat{S}}
    	= \frac{\Exp{L^{\alpha}_{i,\hat{S}} \I{i\in \hat{S}}}}{\Prob{i\in \hat{S}}}
        = \frac{\sum_{s=1}^i p_s L^{\alpha}_{i, \{s,\dots,b\}}}{\sum_{s=1}^i p_s}
    \end{eqnarray*}
    for $\alpha \in \{0,1\}$. Hence the weights from \Cref{thm:rt_l0l1_iter} become
    \begin{eqnarray*}
        \frac{\Prob{i\in \hat{S}}}{\ExpCond{L^1_{i,\hat{S}}}{i\in \hat{S}}}
        = \frac{\parens{\sum_{s=1}^i p_s}^2}{\sum_{s=1}^i p_s L^{\alpha}_{i, \{s,\dots,b\}}}.
    \end{eqnarray*}
    Substituting this expression into the rate yields the result in \Cref{thm:rpt_l0l1_iter}.
\end{remark}

\begin{remark}
    When $\hat{S} = [b]$ with probability $1$, the weights become $w_i = \nicefrac{1}{L^1_{i,[b]}}$, and hence \Cref{thm:rt_l0l1_iter} guarantees that
    \begin{eqnarray*}
        \min_{k=0,\ldots,K-1} \sum_{i=1}^b \left[\frac{\nicefrac{1}{L^1_{i,[b]}}}{\frac{1}{b} \sum_{l=1}^b \nicefrac{1}{L^1_{l,[b]}}} \norm{\nabla _i f(X^k)}_{(i) \star}\right] \leq \varepsilon,
    \end{eqnarray*}
    after
	\begin{eqnarray*}
        K &=& \left\lceil \frac{2 \delta^0 \sum_{i=1}^b \nicefrac{L^0_{i,[b]}}{\parens{L^1_{i,[b]}}^2}}{\varepsilon^2 \parens{\frac{1}{b} \sum_{l=1}^b \nicefrac{1}{L^1_{i,[b]}}}^2}
        + \frac{2 \delta^0}{\varepsilon \parens{\frac{1}{b} \sum_{l=1}^b \nicefrac{1}{L^1_{i,[b]}}}} \right\rceil
    \end{eqnarray*}
	iterations, recovering the rate of \algnamesmall{Gluon} \citep{riabinin2025gluon}.
\end{remark}

\begin{proof}[Proof of \Cref{thm:rt_l0l1_iter}]
    Starting with \Cref{as:arbitrary_layer_gen_smoothness}, we have
    \begin{eqnarray*}
        &&\hspace{-1cm}f(X^{k+1}) \\
		&\leq& f(X^k) + \sum_{i\in S^k} \Bigg(\inp{\nabla_i f(X^k)}{X_i^{k+1} - X_i^k}_{(i)} \\
        &&\qquad\qquad\qquad\quad+ \frac{L^0_{i,S^k} + L^1_{i,S^k} \norm{\nabla_i f(X^k)}_{(i) \star}}{2} \norm{X_i^{k+1} - X_i^k}_{(i)}^2\Bigg) \\
		&=& f(X^k) + \sum_{i\in S^k} \Bigg(- \gamma_i^k \inp{\nabla_i f(X^k)}{\parens{\nabla_i f(X^k)}^{\sharp}}_{(i)} \\
        &&\qquad\qquad\qquad\quad+ \frac{L^0_{i,S^k} + L^1_{i,S^k} \norm{\nabla_i f(X^k)}_{(i) \star}}{2} (\gamma_i^k)^2 \norm{\parens{\nabla_i f(X^k)}^{\sharp}}_{(i)}^2\Bigg) \\
		&\overset{\eqref{eq:inpsharp}, \eqref{eq:normsharp}}{=}& f(X^k) + \sum_{i\in S^k} \Bigg(- \gamma_i^k \norm{\nabla_i f(X^k)}^2_{(i) \star} \\
        &&\qquad\qquad\qquad\quad+ \frac{L^0_{i,S^k} + L^1_{i,S^k} \norm{\nabla_i f(X^k)}_{(i) \star}}{2} (\gamma_i^k)^2 \norm{\nabla_i f(X^k)}^2_{(i) \star}\Bigg).
    \end{eqnarray*}
	Taking $\gamma_i^k = \frac{1}{L^0_{i,S^k} + L^1_{i,S^k} \norm{\nabla_i f(X^k)}_{(i) \star}}$ gives
	\begin{eqnarray*}
        f(X^{k+1}) \leq f(X^k) - \sum_{i\in S^k} \frac{\norm{\nabla_i f(X^k)}^2_{(i) \star}}{2 \parens{L^0_{i,S^k} + L^1_{i,S^k} \norm{\nabla_i f(X^k)}_{(i) \star}}},
    \end{eqnarray*}
	and hence
	\begin{eqnarray*}
        \sum_{k=0}^{K-1} \sum_{i\in S^k} \frac{\norm{\nabla_i f(X^k)}^2_{(i) \star}}{2 \parens{L^0_{i,S^k} + L^1_{i,S^k} \norm{\nabla_i f(X^k)}_{(i) \star}}}
		&\leq& \sum_{k=0}^{K-1} \parens{f(X^k) - f(X^{k+1})} \\
		&\leq& f(X^0) - f^\star
		\eqdef \delta^0.
    \end{eqnarray*}
    Taking expectation
    \begin{align}\label{eq:naovaoqwf}
        \delta^0
        &\geq \sum_{k=0}^{K-1} \Exp{\sum_{i\in S^k} \frac{\norm{\nabla_i f(X^k)}^2_{(i) \star}}{2 \parens{L^0_{i,S^k} + L^1_{i,S^k} \norm{\nabla_i f(X^k)}_{(i) \star}}}} \nonumber \\
        &= \sum_{k=0}^{K-1} \Exp{\ExpCond{\sum_{i=1}^b \I{i\in S^k} \frac{\norm{\nabla_i f(X^k)}^2_{(i) \star}}{2 \parens{L^0_{i,S^k} + L^1_{i,S^k} \norm{\nabla_i f(X^k)}_{(i) \star}}}}{X^k}} \nonumber \\
        &= \sum_{k=0}^{K-1} \sum_{i=1}^b \Exp{\Prob{i\in S^k} \ExpCond{\frac{\norm{\nabla_i f(X^k)}^2_{(i) \star}}{2 \parens{L^0_{i,S^k} + L^1_{i,S^k} \norm{\nabla_i f(X^k)}_{(i) \star}}}}{X^k, \brac{i\in S^k}}} \nonumber \\
        &\overset{(i)}{\geq} \sum_{k=0}^{K-1} \sum_{i=1}^b \Exp{\frac{\norm{\nabla_i f(X^k)}^2_{(i) \star} \Prob{i\in S^k}}{2 \parens{\ExpCond{L^0_{i,S^k}}{X^k, \brac{i\in S^k}} + \ExpCond{L^1_{i,S^k}}{X^k, \brac{i\in S^k}} \norm{\nabla_i f(X^k)}_{(i) \star}}}} \nonumber \\
        &\overset{(ii)}{=} \frac{1}{2} \sum_{k=0}^{K-1} \sum_{i=1}^b \Exp{\frac{\norm{\nabla_i f(X^k)}^2_{(i) \star} \Prob{i\in \hat{S}}}{\ExpCond{L^0_{i,\hat{S}}}{i\in \hat{S}} + \ExpCond{L^1_{i,\hat{S}}}{i\in \hat{S}} \norm{\nabla_i f(X^k)}_{(i) \star}}} \nonumber \\
        &\overset{(iii)}{\geq} \frac{1}{2} \sum_{k=0}^{K-1} \sum_{i=1}^b \frac{\Exp{\norm{\nabla_i f(X^k)}_{(i) \star}}^2 \Prob{i\in \hat{S}}}{\ExpCond{L^0_{i,\hat{S}}}{i\in \hat{S}} + \ExpCond{L^1_{i,\hat{S}}}{i\in \hat{S}} \Exp{\norm{\nabla_i f(X^k)}_{(i) \star}}},
    \end{align}
    where in $(i)$ we used Jensen's inequality and convexity of the function $t \mapsto \frac{1}{t}$ for $t>0$, $(ii)$ follows from independence of $\hat{S}$ and $X^k$, and $(iii)$ is a consequence of convexity of the function $t \mapsto \frac{t^2 \Prob{i\in S^k}}{\ExpCond{L^0_{i,\hat{S}}}{i\in \hat{S}} + \ExpCond{L^1_{i,\hat{S}}}{i\in \hat{S}} t}$ and Jensen's inequality.

    Now, using \Cref{lemma:ineq} with $x_i = \nicefrac{\sqrt{\Prob{i\in \hat{S}}}}{\ExpCond{L^1_{i,\hat{S}}}{i\in \hat{S}}}$, $y_i = \sqrt{\Prob{i\in \hat{S}}} \Exp{\norm{\nabla_i f(X^k)}_{(i) \star}}$ and $z_i = \ExpCond{L^0_{i,\hat{S}}}{i\in \hat{S}} + \ExpCond{L^1_{i,\hat{S}}}{i\in \hat{S}} \Exp{\norm{\nabla_i f(X^k)}_{(i) \star}}$, we obtain
    \begin{align*}
        &\sum_{i=1}^b \frac{\Prob{i\in \hat{S}} \Exp{\norm{\nabla_i f(X^k)}_{(i) \star}}^2}{\ExpCond{L^0_{i,\hat{S}}}{i\in \hat{S}} + \ExpCond{L^1_{i,\hat{S}}}{i\in \hat{S}} \Exp{\norm{\nabla_i f(X^k)}_{(i) \star}}} \\
        &\geq \frac{\parens{\sum_{i=1}^b \frac{\Prob{i\in \hat{S}}}{\ExpCond{L^1_{i,\hat{S}}}{i\in \hat{S}}} \Exp{\norm{\nabla_i f(X^k)}_{(i) \star}}}^2}{\sum_{i=1}^b \parens{\frac{\Prob{i\in \hat{S}}}{\parens{\ExpCond{L^1_{i,\hat{S}}}{i\in \hat{S}}}^2} \ExpCond{L^0_{i,\hat{S}}}{i\in \hat{S}} + \frac{\Prob{i\in \hat{S}}}{\ExpCond{L^1_{i,\hat{S}}}{i\in \hat{S}}} \Exp{\norm{\nabla_i f(X^k)}_{(i) \star}}}}.
    \end{align*}
    Applying this in \eqref{eq:naovaoqwf}, we get
    \begin{eqnarray*}
        \delta^0
        &\geq& \frac{1}{2} \sum_{k=0}^{K-1}
        \sum_{i=1}^b \frac{\Exp{\norm{\nabla_i f(X^k)}_{(i) \star}}^2 \Prob{i\in \hat{S}}}{\ExpCond{L^0_{i,\hat{S}}}{i\in \hat{S}} + \ExpCond{L^1_{i,\hat{S}}}{i\in \hat{S}} \Exp{\norm{\nabla_i f(X^k)}_{(i) \star}}} \\
        &\geq& \frac{1}{2} \sum_{k=0}^{K-1} \frac{\parens{\sum_{i=1}^b \frac{\Prob{i\in \hat{S}}}{\ExpCond{L^1_{i,\hat{S}}}{i\in \hat{S}}} \Exp{\norm{\nabla_i f(X^k)}_{(i) \star}}}^2}{\sum_{i=1}^b \parens{\frac{\Prob{i\in \hat{S}}}{\parens{\ExpCond{L^1_{i,\hat{S}}}{i\in \hat{S}}}^2} \ExpCond{L^0_{i,\hat{S}}}{i\in \hat{S}} + \frac{\Prob{i\in \hat{S}}}{\ExpCond{L^1_{i,\hat{S}}}{i\in \hat{S}}} \Exp{\norm{\nabla_i f(X^k)}_{(i) \star}}}} \\
        &=& \frac{1}{2} \sum_{k=0}^{K-1} \psi \parens{\sum_{i=1}^b \frac{\Prob{i\in \hat{S}}}{\ExpCond{L^1_{i,\hat{S}}}{i\in \hat{S}}} \Exp{\norm{\nabla_i f(X^k)}_{(i) \star}}},
    \end{eqnarray*}
	where $\psi(t) \eqdef \frac{t^2}{\sum_{i=1}^b \frac{\Prob{i\in \hat{S}}}{\parens{\ExpCond{L^1_{i,\hat{S}}}{i\in \hat{S}}}^2} \ExpCond{L^0_{i,\hat{S}}}{i\in \hat{S}} + t}$. Since $\psi$ is increasing for $t>0$, we have
    \begin{eqnarray*}
        \delta^0
        &\geq& \frac{1}{2} \sum_{k=0}^{K-1} \psi \parens{\sum_{i=1}^b \frac{\Prob{i\in \hat{S}}}{\ExpCond{L^1_{i,\hat{S}}}{i\in \hat{S}}} \Exp{\norm{\nabla_i f(X^k)}_{(i) \star}}} \\
        &\geq& \frac{K}{2} \psi \parens{\min_{k=0,\ldots,K-1} \sum_{i=1}^b \frac{\Prob{i\in \hat{S}}}{\ExpCond{L^1_{i,\hat{S}}}{i\in \hat{S}}} \Exp{\norm{\nabla_i f(X^k)}_{(i) \star}}}.
    \end{eqnarray*}
    Moreover, since $\psi$ is monotonic, it has an inverse $\psi^{-1}$. Thus
    \begin{eqnarray*}
        \psi^{-1}\parens{\frac{2 \delta^0}{K}}
        &\geq& \min_{k=0,\ldots,K-1} \sum_{i=1}^b \frac{\Prob{i\in \hat{S}}}{\ExpCond{L^1_{i,\hat{S}}}{i\in \hat{S}}} \Exp{\norm{\nabla_i f(X^k)}_{(i) \star}} \\
        &=& \min_{k=0,\ldots,K-1} \sum_{i=1}^b w_i \Exp{\norm{\nabla_i f(X^k)}_{(i) \star}},
    \end{eqnarray*}
    where $w_i \eqdef \frac{\Prob{i\in \hat{S}}}{\ExpCond{L^1_{i,\hat{S}}}{i\in \hat{S}}}$.
    This in turn means that to reach the precision
    \begin{eqnarray*}
        \min_{k=0,\ldots,K-1} \sum_{i=1}^b \left[\frac{w_i}{\frac{1}{b} \sum_{l=1}^b w_l} \Exp{\norm{\nabla _i f(X^k)}_{(i) \star}}\right] \leq \varepsilon,
    \end{eqnarray*}
    it suffices to run the algorithm for
	\begin{eqnarray*}
        K &=& \left\lceil \frac{2 \delta^0}{\psi\parens{\varepsilon \parens{\frac{1}{b} \sum_{l=1}^b w_l}}} \right\rceil
        = \left\lceil \frac{2 \delta^0 \sum_{i=1}^b \frac{\Prob{i\in \hat{S}} \ExpCond{L^0_{i,\hat{S}}}{i\in \hat{S}}}{\parens{\ExpCond{L^1_{i,\hat{S}}}{i\in \hat{S}}}^2} + 2 \delta^0 \parens{\varepsilon \parens{\frac{1}{b} \sum_{l=1}^b w_l}}}{\parens{\varepsilon \parens{\frac{1}{b} \sum_{l=1}^b w_l}}^2} \right\rceil \\
        &=& \left\lceil \frac{2 \delta^0 \sum_{i=1}^b \frac{\Prob{i\in \hat{S}} \ExpCond{L^0_{i,\hat{S}}}{i\in \hat{S}}}{\parens{\ExpCond{L^1_{i,\hat{S}}}{i\in \hat{S}}}^2}}{\varepsilon^2 \parens{\frac{1}{b} \sum_{l=1}^b \frac{\Prob{l\in \hat{S}}}{\ExpCond{L^1_{l,\hat{S}}}{l\in \hat{S}}}}^2}
        + \frac{2 \delta^0}{\varepsilon \parens{\frac{1}{b} \sum_{l=1}^b \frac{\Prob{l \in \hat{S}}}{\ExpCond{L^1_{l,\hat{S}}}{l\in \hat{S}}}}} \right\rceil
    \end{eqnarray*}
    iterations.
\end{proof}

\newpage

\section{Optimizing the Cost -- Deterministic Gradient Case}\label{sec:cost_opt}

Let $S^k\subseteq\{1,\dots,b\}$ be the random subset sampled at iteration $k$, and
$s^k \eqdef \min S^k$ its smallest index. Recall the per-iteration cost
\begin{eqnarray*}
    \mathrm{cost}(S^k)
    = c_{\mathrm{ov}} + \sum_{i=s^k}^b c_i + \sum_{i\in S^k} c_i^{\sharp}.
\end{eqnarray*}
Define the two marginal probabilities
\begin{eqnarray*}
    F_i \eqdef \Prob{\hat{s} \leq i},\qquad
    Q_i \eqdef \Prob{i \in \hat{S}},
\end{eqnarray*}
where $\hat{s}$ is a random variable following the same distribution as $s^k$ (since $S^k \sim \cD$, $k \geq 0$, are i.i.d., the same holds for $s^k$). Since
\begin{align*}
    \Exp{\sum_{i=s^k}^b c_i}
    = \sum_{j=1}^b \Prob{s^k = j} \sum_{i=j}^b c_i
    = \sum_{i=1}^b c_i \sum_{j=1}^i \Prob{s^k = j}
    = \sum_{i=1}^b c_i \Prob{s^k \leq i}
\end{align*}
and
\begin{align*}
    \Exp{\sum_{i\in S^k} c_i^{\sharp}}
    = \Exp{\sum_{i=1}^b \I{i\in S^k} c_i^{\sharp}}
    = \sum_{i=1}^b c_i^{\sharp} \Prob{i\in S^k},
\end{align*}
the expected cost of one iteration is
\begin{eqnarray}\label{eq:cost_gen}
    \Exp{\mathrm{cost}(\hat{S})}
    = c_{\mathrm{ov}} + \sum_{i=1}^b c_i F_i + \sum_{i=1}^b c_i^{\sharp} Q_i.
\end{eqnarray}
Hence, to evaluate the expected cost, it suffices to compute the two marginals $F_i$ and $Q_i$ under the chosen sampling scheme.

In the sequel, we describe a few example sampling strategies considered in this work and analyze their costs under the layer-wise smooth setting (see \Cref{sec:cost_opt_smooth}) and the generalized layer-wise smooth setting (see \Cref{sec:cost_opt_l0l1}).
\begin{enumerate}
    \item \textbf{\algname{RPT}.}
    Sample $\hat{s}\in\{1,\dots,b\}$, where $p_i = \Prob{\hat{s}=i}$, and set $\hat{S}=\{\hat{s},\dots,b\}$. Then
    \begin{eqnarray*}
        F_i = \Prob{\hat{s}\leq i} = \sum_{j=1}^i p_j,
        \qquad
        Q_i = \Prob{i\in \hat{S}} = \Prob{\hat{s}\leq i} = F_i.
    \end{eqnarray*}
    Therefore,
    \begin{eqnarray*}
        \Exp{\mathrm{cost}(\hat{S})}
        = c_{\mathrm{ov}} + \sum_{i=1}^b (c_i + c_i^{\sharp}) F_i
        = c_{\mathrm{ov}} + \sum_{i=1}^b (c_i + c_i^{\sharp}) \sum_{j=1}^i p_j.
    \end{eqnarray*}

    \item \textbf{$\tau$-nice sampling.}
    Choose $\hat{S}$ uniformly from all subsets of $[b]$ of size $\tau$. Then
    \begin{eqnarray*}
        F_i &=& 1 - \Prob{\hat{s} > i}
        = 1 - \frac{\binom{b-i}{\tau}}{\binom{b}{\tau}}, \\
        Q_i &=& \frac{\binom{b-1}{\tau-1}}{\binom{b}{\tau}}
        = \frac{\tau}{b},
    \end{eqnarray*}
    where we use the convention $\binom{n}{k}=0$ for $n<k$ (so the formula for $F_i$ automatically gives $F_i=1$ when $b-i<\tau$).
    Hence
    \begin{eqnarray*}
        \Exp{\mathrm{cost}(\hat{S})}
        = c_{\mathrm{ov}} + \sum_{i=1}^b c_i \parens{1 - \frac{\binom{b-i}{\tau}}{\binom{b}{\tau}}}
        + \frac{\tau}{b} \sum_{i=1}^b c_i^{\sharp}.
    \end{eqnarray*}

    \item \textbf{$\tau$-submodel sampling.}
    Sample a starting index $\hat{s}\in\{1,\dots,b-\tau+1\}$ with probability $p_i = \Prob{\hat{s}=i}$ (where $\tau\in[b]$ is fixed), and set $\hat{S}=\{\hat{s},\dots,\hat{s}+\tau-1\}$ (i.e., a block of $\tau$ consecutive layers). Then the marginals are
    \begin{eqnarray*}
        F_i &=& \Prob{\hat{s} \leq i}
        = \sum_{j=1}^{\min\{i,b-\tau+1\}} p_j, \\
        Q_i &=& \Prob{i\in \hat{S}}
        = \sum_{j=\max\{1,i-\tau+1\}}^{\min\{i,b-\tau+1\}} p_j.
    \end{eqnarray*}
    The expected per-iteration cost is therefore
    \begin{eqnarray*}
        \Exp{\mathrm{cost}(\hat{S})}
        = c_{\mathrm{ov}} + \sum_{i=1}^b c_i \parens{\sum_{j=1}^{\min\{i,b-\tau+1\}} p_j}
        + \sum_{i=1}^b c_i^{\sharp} \parens{\sum_{j=\max\{1,i-\tau+1\}}^{\min\{i,b-\tau+1\}} p_j}.
    \end{eqnarray*}

    \item \textbf{Arbitrary submodel sampling.}
    Let $\{B_1,\dots,B_m\}$ be a partition of $[b]$ into disjoint blocks of arbitrary indices, i.e.,
    \begin{align*}
    	B_1 \cup \cdots \cup B_m = [b], \quad 
    	B_k \cap B_l = \emptyset \quad\textnormal{for } k\neq l.
    \end{align*}
    At each iteration, pick block $B_i$ with probability $p_i$ (where $\sum_{i=1}^m p_i = 1$) and set $\hat{S} = B_i$.
    
    For $i\in[b]$, let $k(i)$ denote the unique block with $i\in B_{k(i)}$, and let $\underline{b}_k \eqdef \min B_k$. Then
    \begin{align*}
    	F_i &= \Prob{\hat{s} \leq i}
    	= \sum_{j: \underline{b}_j \leq i} p_j, \\
    	Q_i &= \Prob{i\in \hat{S}}
    	= p_{k(i)}.
    \end{align*}
    The expected cost per iteration is
    \begin{align*}
    	\Exp{\mathrm{cost}(\hat{S})}
    	= c_{\mathrm{ov}} + \sum_{i=1}^b c_i \parens{\sum_{k: \underline{b}_k \leq i} p_k} + \sum_{i=1}^b c_i^{\sharp} p_{k(i)}.
    \end{align*}
\end{enumerate}

We now consider the algorithm's performance under the two smoothness regimes.

\subsection{Smooth Case}\label{sec:cost_opt_smooth}

According to \Cref{thm:rt_smooth_iter}, under \Cref{as:arbitrary_layer_smoothness}, \Cref{alg:rt_arbitrary} run with stepsizes $\gamma_i^k = \nicefrac{1}{L_{i,S^k}^0}$ guarantees that
\begin{eqnarray*}
    \parens{\min_{i\in[b]} w_i} \frac{1}{K} \sum_{k=0}^{K-1} \sum_{i=1}^b \Exp{\norm{\nabla_i f(X^k)}^2_{(i) \star}}
    &\leq& \frac{1}{K} \sum_{k=0}^{K-1} \sum_{i=1}^b w_i \Exp{\norm{\nabla_i f(X^k)}^2_{(i) \star}} \\
    &\leq& \frac{f(X^0) - f^{\star}}{K},
\end{eqnarray*}
where $w_i \eqdef \Exp{\frac{\I{i\in \hat{S}}}{2 L_{i,\hat{S}}^0}}$. Thus, to ensure that $\frac{1}{K} \sum_{k=0}^{K-1} \sum_{i=1}^b \Exp{\norm{\nabla_i f(X^k)}^2_{(i) \star}} \leq \varepsilon$, it suffices to run it for $$K = \ceil{\frac{f(X^0) - f^{\star}}{\varepsilon \parens{\min_{i\in[b]} w_i}}}$$ iterations. Now, recall from \eqref{eq:cost_gen} that the expected cost of a single iteration is
\begin{eqnarray*}
    \Exp{\mathrm{cost}(S^k)}
    = c_{\mathrm{ov}} + \sum_{i=1}^b c_i \Prob{\min S^k \leq i} + \sum_{i=1}^b c_i^{\sharp} \Prob{i\in S^k}.
\end{eqnarray*}
Hence, the expected cost of the entire optimization procedure can be written as
\begin{eqnarray*}
    \mathrm{cost}_{\varepsilon}(\cD)
    &=& K \times \Exp{\mathrm{cost}(\hat{S})} \\
    &=& K \times \parens{c_{\mathrm{ov}} + \sum_{i=1}^b c_i \Prob{\min \hat{S} \leq i} + \sum_{i=1}^b c_i^{\sharp} \Prob{i\in \hat{S}}} \\
    &\propto& \frac{c_{\mathrm{ov}} + \sum_{i=1}^b c_i \Prob{\min \hat{S} \leq i} + \sum_{i=1}^b c_i^{\sharp} \Prob{i\in \hat{S}}}{\min_{i\in[b]} \Exp{\frac{\I{i\in \hat{S}}}{2 L_{i,\hat{S}}^0}}},
\end{eqnarray*}
and the cost minimization problem to be solved is
\begin{eqnarray}\label{eq:cost_opt_problem}
    \min_{\cD: \mathfrak{P}([b]) \to [0,1], \sum_{S\subseteq[b]} \cD(S) = 1} \frac{c_{\mathrm{ov}} + \sum_{i=1}^b c_i \Prob{\min \hat{S} \leq i} + \sum_{i=1}^b c_i^{\sharp} \Prob{i\in \hat{S}}}{\min_{i\in[b]} \Exp{\frac{\I{i\in \hat{S}}}{2 L_{i,\hat{S}}^0}}}.
\end{eqnarray}

The task above is an optimization over probability distributions on the power set $\mathfrak{P}([b])$, which has dimension $2^b$, making a direct solution intractable for large $b$.
Instead of tackling it in full generality, we can restrict $\cD$ to some parametric family. For certain such families, the ratio objective simplifies to a linear–fractional program in the cumulative marginals, which has a closed form solution or can be solved efficiently (e.g., via the Dinkelbach algorithm \citep{dinkelbach1967nonlinear}).

Let us now consider some specific examples, starting with the procedure considered in the main part of this paper.

\subsubsection{Randomized Progressive Training}\label{sec:rpt_sol}

\begin{mdframed}[linecolor=lightgray!12, backgroundcolor=lightgray!12]
    Sample $\hat{s}\in\{1,\dots,b\}$, where $p_i = \Prob{\hat{s}=i}$, and set $\hat{S}=\{\hat{s},\dots,b\}$.
\end{mdframed}

We first consider the randomized progressive training setting introduced in \Cref{sec:rpt}.
Under this sampling strategy, we have
\begin{eqnarray*}
    F_i = \Prob{\hat{s}\leq i} = \sum_{j=1}^i p_j,
    \qquad
    Q_i = \Prob{i\in \hat{S}} = \Prob{\hat{s}\leq i} = F_i,
\end{eqnarray*}
and hence
\begin{eqnarray*}
    \Exp{\mathrm{cost}(\hat{S})}
    = c_{\mathrm{ov}} + \sum_{i=1}^b (c_i + c_i^{\sharp}) F_i
    = c_{\mathrm{ov}} + \sum_{i=1}^b (c_i + c_i^{\sharp}) \sum_{j=1}^i p_j.
\end{eqnarray*}
Combining it with the fact that
\begin{eqnarray*}
    \Exp{\frac{\I{i\in \hat{S}}}{2 L_{i,\hat{S}}^0}}
    = \sum_{s=1}^b \frac{\I{i \in \{s,\dots,b\}}}{2 L_{i, \{s,\dots,b\}}^0} p_s
    = \sum_{s=1}^i \frac{p_s}{2 L_{i, \{s,\dots,b\}}^0},
\end{eqnarray*}
we get
\begin{eqnarray}\label{eq:costK}
    \mathrm{cost}_{\varepsilon}(\cD)
    \propto \frac{\Exp{\mathrm{cost}(\hat{S})}}{\min_{i\in[b]} \Exp{\frac{\I{i\in \hat{S}}}{2 L_{i,\hat{S}}^0}}}
    = \frac{c_{\mathrm{ov}} + \sum_{i=1}^b (c_i + c_i^{\sharp}) \sum_{j=1}^i p_j}{\min_{i\in[b]} \brac{\sum_{s=1}^i \frac{p_s}{2 L_{i, \{s,\dots,b\}}^0}}}.
\end{eqnarray}

\paragraph{Optimal probabilities.}

First, note that the numerator can be rewritten as
\begin{align*}
    c_{\mathrm{ov}} + \sum_{i=1}^b (c_i + c_i^{\sharp}) \sum_{j=1}^i p_j
    =c_{\mathrm{ov}}+\sum_{j=1}^b \left[\sum_{i\geq j} (c_i + c_i^{\sharp})\right] p_j
    =\sum_{j=1}^b \left[c_{\mathrm{ov}}+\sum_{i\geq j} (c_i + c_i^{\sharp})\right] p_j,
\end{align*}
where the second equality follows from $\sum_{j=1}^b p_j=1$. Denote $$d_j \eqdef c_{\mathrm{ov}}+\sum_{i\geq j} (c_i + c_i^{\sharp})$$ for $j \in [b]$, and
let $$\delta_{i,s}\eqdef\frac{1}{2L_{i, \{s,\dots,b\}}}$$ for $i\in [b]$ and $s\leq i$.
Clearly 
\begin{align}\label{eq:orderofd}
    d_1 > d_2 > \ldots > d_b
\end{align}
and
\begin{align}\label{eq:orderofdelta}
    \delta_{i,1} \leq  \delta_{i,2} \leq  \ldots \leq  \delta_{i,i} \qquad \forall i\in [b].
\end{align}
Based on~\eqref{eq:costK}, the search for optimal probabilities reduces to solving the following linear fractional program:
\begin{equation}\label{eq:fl1}
    \begin{aligned}
        \min_{p, t} \quad & \frac{d_1p_1+\ldots+d_bp_b}{t} \\
        \text{s.t.} \quad & p_1,\ldots,p_b\geq 0\\
        & p_1+\ldots+p_b=1\\
        & t\geq 0 \\
        & t\leq \delta_{1,1}p_1 \\
        & \vdots\\
        & t\leq \delta_{i,1}p_1+\ldots+\delta_{i,i}p_i\\
        & \vdots\\
        & t\leq \delta_{b,1}p_1+\ldots+\delta_{b,b}p_b.
    \end{aligned}
\end{equation}
This program can be written equivalently as
\begin{equation}\label{eq:fl2}
    \begin{aligned}
        \min_{q} \quad & {d_1 q_1+\ldots+d_b q_b} \\
        \text{s.t.} \quad & q_1,\ldots,q_b\geq 0\\
        &  \delta_{1,1}q_1\geq 1 \\
        & \vdots\\
        &  \delta_{i,1}q_1+\ldots+\delta_{i,i}q_i\geq 1\\
        & \vdots\\
        &  \delta_{b,1}q_1+\ldots+\delta_{b,b}q_b \geq 1.
    \end{aligned}
\end{equation}

Based on that, we can derive a recursive prescription for the optimal probabilities.

First, we show that if $\sum_{s=1}^i \delta_{i,s} q_s>1$ and $q_i>0$ for some $i$, then one can shift a small amount of mass from $q_i$ to $q_{i+1}$ and strictly reduce the objective, without violating any constraints.

\begin{lemma}\label{lemma:exchange}
    Let $q$ be an optimal point of \eqref{eq:fl2} and fix $i \in [b-1]$.
    If $\sum_{s=1}^i \delta_{i,s} q_s>1$, then $q_i = 0$. Equivalently,
    \begin{align}\label{eq:exch}
        q_i > 0 \qquad\implies\qquad \sum_{s=1}^i \delta_{i,s} q_s = 1.
    \end{align}
\end{lemma}

\begin{proof}
    Suppose that $i$ is such that $\sum_{s=1}^i \delta_{i,s} q_s>1$, but $q_i > 0$. Then there exists $\varepsilon>0$ such that $q_i \geq \varepsilon$ and
    \begin{align*}
        0<\varepsilon \leq \frac{\sum_{s=1}^i \delta_{i,s} q_s - 1}{\delta_{i,i}}.
    \end{align*}
    Define
    \begin{align*}
        \tilde{q_i} = q_i-\varepsilon,\quad \tilde{q}_{i+1} = q_{i+1} + \varepsilon, \quad \tilde{q}_s = q_s \text{ for } s\notin\{i,i+1\}.
    \end{align*}
    Then $\{\tilde{q}_i\}_{i\in[b]}$ satisfy all constraints:
    \begin{itemize}
        \item Constraint $k<i$: $\sum_{s=1}^k \delta_{k,s} \tilde{q}_s = \sum_{s=1}^k \delta_{k,s} q_s \geq 1$,
        
        \item Constraint $i$: $\sum_{s=1}^i \delta_{i,s} \tilde{q}_s = \sum_{s=1}^i \delta_{i,s} q_s - \delta_{i,i} \varepsilon \geq 1$ by the choice of $\varepsilon$,
        
        \item Constraint $k>i$: $\sum_{s=1}^k \delta_{k,s} \tilde{q}_s = \sum_{s=1}^k \delta_{k,s} q_s + \varepsilon(\delta_{k,i+1}-\delta_{k,i}) \geq \sum_{s=1}^k \delta_{k,s} q_s \geq 1$ using the monotonicity $\delta_{k,i} \leq \delta_{k,i+1}$.
    \end{itemize}
    At the same time, the value of the objective decreases, as $\sum_j d_j \tilde q_j - \sum_j d_j q_j = (d_{i+1}-d_i)\varepsilon < 0$, contradicting optimality of $q$. Therefore, we must have $q_i = 0$.
\end{proof}

Let us now use \Cref{lemma:exchange} to derive a recursive construction of the optimal probabilities. Let $q^\star = (q_1^\star, \ldots, q_b^\star)$ be a solution of \eqref{eq:fl2}.
As established in \eqref{eq:exch}, any positive coordinate forces its constraint to be tight.
Constraint $1$ is $\delta_{1,1} q_1^\star \geq 1$, meaning that $q_1^\star > 0$, and so $\delta_{1,1} q_1^\star = 1$. Thus
\begin{align*}
    q_1^\star = \frac{1}{\delta_{1,1}} = 2 L_{1, \{1,\dots,b\}}.
\end{align*}
Now, suppose that $q_1^\star,\dots,q_{i-1}^\star$ have already been determined. Then, constraint $i$ reads
\begin{align*}
    \sum_{s=1}^{i-1} \delta_{i,s} q_s^\star + \delta_{i,i} q_i^\star \geq 1.
\end{align*}
Define the residual
\begin{align*}
    r_i^\star \eqdef 1 - \sum_{s=1}^{i-1} \delta_{i,s} q_s^\star
    = 1 - \sum_{s=1}^{i-1} \frac{q_s^\star}{2L_{i, \{s,\dots,b\}}}.
\end{align*}
If $r_i^\star \le 0$, then $(q_1^\star,\dots,q_{i-1}^\star)$ already satisfy the constraint, so by \eqref{eq:exch}, we must have $q_i^\star=0$.
If $r_i^\star >0$, then for $(q_1^\star,\dots,q_i^\star)$ to satisfy the constraint, we need $q_i^\star>0$, and hence by \Cref{lemma:exchange}, we have $\sum_{s=1}^i \delta_{i,s} q_s^\star = 1$, meaning that
\begin{align*}
    q_i^\star = \frac{r_i^\star}{\delta_{i,i}}
    = 2 r_i^\star L_{i, \{i,\dots,b\}}.
\end{align*}
Combining the above yields the recursion
\begin{align*}
    q_1^\star &= 2 L_{1, \{1,\dots,b\}}, \\
    q_i^\star &= 2 [r_i^\star]_+ L_{i, \{i,\dots,b\}},\qquad r_i^\star = 1-\sum_{s=1}^{i-1} \frac{q_s^\star}{2L_{i, \{s,\dots,b\}}},\quad i\in\{2,\ldots,b\},
\end{align*}
where $[x]_+ \eqdef \max\{x,0\}$. The optimal probabilities are finally recovered by normalization
\begin{align*}
    p_i^\star = \frac{q_i^\star}{\sum_{j=1}^b q_j^\star}.
\end{align*}

\begin{remark}\label{rem:rpt_p1}
    Note that $(p_1^\star, p_2^\star, \ldots, p_b^\star)=(1, 0, \ldots, 0)$ if and only if $q_1^\star = \sum_{j=1}^b q_j^\star$, i.e., $q_2^\star = \ldots = q_b^\star = 0$. But for that to be the case, we need $r_i^\star \leq 0$ for all $i\in\{2,\ldots,b\}$, so
    \begin{align*}
        1 \leq \sum_{s=1}^{i-1} \delta_{i,s} q_s^\star
        = \delta_{i,1} q_1^\star
        = \frac{\delta_{i,1}}{\delta_{1,1}}
        = \frac{L_{1, \{1,\dots,b\}}}{L_{i, \{1,\dots,b\}}} \qquad \forall i\in\{2,\ldots,b\}.
    \end{align*}
    Therefore, choosing $p_1=1$ is optimal if and only if
    \begin{equation*}
        L_{1,\{1,\ldots,b\}}=\max_{i\in[b]} L_{i,\{1,\ldots,b\}},
    \end{equation*}
    which proves \Cref{thm:rpt_p1}.
\end{remark}

\subsubsection{$\tau$-nice Sampling}\label{sec:tau_nice_smooth}

\begin{mdframed}[linecolor=lightgray!12, backgroundcolor=lightgray!12]
    Choose $\hat{S}$ uniformly from all subsets of $[b]$ of size $\tau$.
\end{mdframed}

Every $\tau$-subset has probability $\binom{b}{\tau}^{-1}$. Thus
\begin{eqnarray*}
    \Exp{\frac{\I{i\in\hat{S}}}{2L^0_{i,\hat{S}}}}
    = \frac{1}{\binom{b}{\tau}} \sum_{\substack{S\subseteq[b]\\ |S|=\tau}} \frac{\I{i\in S}}{2L^0_{i,S}},
\end{eqnarray*}
and hence
\begin{eqnarray*}
    \mathrm{cost}_{\varepsilon}(\cD)
    &\propto& \frac{\Exp{\mathrm{cost}(\hat{S})}}{\min_{i\in[b]} \Exp{\frac{\I{i\in \hat{S}}}{2 L_{i,\hat{S}}^0}}}
    = \frac{c_{\mathrm{ov}} + \sum_{j=1}^b c_j \parens{1 - \frac{\binom{b-j}{\tau}}{\binom{b}{\tau}}} + \frac{\tau}{b} \sum_{j=1}^b c_j^{\sharp}}{\min_{i\in[b]} \frac{1}{\binom{b}{\tau}} \sum_{\substack{S\subseteq[b]\\ |S|=\tau}} \frac{\I{i\in S}}{2L^0_{i,S}}} \\
    &=& \frac{c_{\mathrm{ov}} \binom{b}{\tau} + \sum_{j=1}^b c_j \parens{\binom{b}{\tau} - \binom{b-j}{\tau}} + \binom{b}{\tau} \frac{\tau}{b} \sum_{j=1}^b c_j^{\sharp}}{\min_{i\in[b]} \brac{\sum_{\substack{S\subseteq[b]\\ |S|=\tau}} \frac{\I{i\in S}}{2L^0_{i,S}}}}.
\end{eqnarray*}
In general there is no closed-form expression for $\tau$ minimizing this cost because the denominator depends on $\{L^0_{i,S}\}_{S\subseteq[b]}$ in a highly  problem-specific way. That said, it can be shown that $\tau=b$ need not be optimal (indeed, it can be very sub-optimal).
For $\tau=b$, we always have $S=[b]$, and the cost becomes
\begin{eqnarray*}
    \mathrm{cost}_{\varepsilon}(\cD)
    &\propto& \frac{c_{\mathrm{ov}} + \sum_{j=1}^b c_j \parens{1 - \binom{b-j}{b}} + \sum_{j=1}^b c_j^{\sharp}}{\min_{i\in[b]} \frac{1}{2L^0_{i,[b]}}} \\
    &=& 2 \max_{i\in[b]} L^0_{i,[b]} \parens{c_{\mathrm{ov}} + \sum_{j=1}^b \parens{c_j + c_j^{\sharp}}}
\end{eqnarray*}
(recall that we use the convention $\binom{n}{k}=0$ for $n<k$).

Now, let us make a simplifying assumption that $L^0_{i,S}$ depends only on $i$ and $|S|$ and define $L^0_{i,\tau} \eqdef L^0_{i,S}$ for $|S|=\tau$. Then
\begin{eqnarray*}
    \mathrm{cost}_{\varepsilon}(\cD)
    &\propto& \frac{c_{\mathrm{ov}} \binom{b}{\tau} + \sum_{j=1}^b c_j \parens{\binom{b}{\tau} - \binom{b-j}{\tau}} + \binom{b}{\tau} \frac{\tau}{b} \sum_{j=1}^b c_j^{\sharp}}{\min_{i\in[b]} \brac{\sum_{\substack{S\subseteq[b]\\ |S|=\tau}} \frac{\I{i\in S}}{2L^0_{i,\tau}}}} \\
    &=& \frac{c_{\mathrm{ov}} \binom{b}{\tau} + \sum_{j=1}^b c_j \parens{\binom{b}{\tau} - \binom{b-j}{\tau}} + \binom{b}{\tau} \frac{\tau}{b} \sum_{j=1}^b c_j^{\sharp}}{\min_{i\in[b]} \brac{\frac{1}{2L^0_{i,\tau}} \binom{b-1}{\tau-1}}} \\
    &=& 2 \max_{i\in[b]} L^0_{i,\tau} \parens{\frac{b}{\tau} c_{\mathrm{ov}} + \sum_{j=1}^b c_j \parens{\frac{b}{\tau} - \frac{\binom{b-j}{\tau}}{\binom{b-1}{\tau-1}}} + \sum_{j=1}^b c_j^{\sharp}}.
\end{eqnarray*}
Define
\begin{align*}
	A(\tau) \eqdef \max_{i\in[b]} L^0_{i,\tau}, 
	\qquad 
	B(\tau) \eqdef \frac{b}{\tau} c_{\mathrm{ov}} + \sum_{j=1}^b c_j \parens{\frac{b}{\tau} - \frac{\binom{b-j}{\tau}}{\binom{b-1}{\tau-1}}} + \sum_{j=1}^b c_j^{\sharp}.
\end{align*}
Then, the objective function to be minimized is
\begin{align*}
	f(\tau) \eqdef A(\tau) B(\tau).
\end{align*}
By definition, $\tau=b$ is optimal if and only if $f(b) \leq f(\tau)$ for all $\tau \in \{1,\dots,b-1\}$.

First, we show that $B$ is decreasing in $\tau$. To this end, let $1 \leq \tau_1 < \tau_2 \leq b$ and consider the difference
\begin{eqnarray*}
	B(\tau_1) - B(\tau_2)
	&=& \parens{\frac{b}{\tau_1} - \frac{b}{\tau_2}} c_{\mathrm{ov}} + \sum_{j=1}^b c_j \parens{\frac{b}{\tau_1} - \frac{\binom{b-j}{\tau_1}}{\binom{b-1}{\tau_1-1}} - \frac{b}{\tau_2} + \frac{\binom{b-j}{\tau_2}}{\binom{b-1}{\tau_2-1}}} \\
	&=& \parens{\frac{b}{\tau_1} - \frac{b}{\tau_2}} c_{\mathrm{ov}}
	+ \sum_{j=1}^b c_j \parens{h_j(\tau_1) - h_j(\tau_2)},
\end{eqnarray*}
where $h_j(\tau) \eqdef \frac{b}{\tau} - \frac{\binom{b-j}{\tau}}{\binom{b-1}{\tau-1}}$. Clearly, the first term is positive. Let us focus on the second term.
Using Pascal's identity, we have
\begin{align*}
	h_{j+1}(\tau)-h_j(\tau)
	&=\frac{\binom{b-j}{\tau}-\binom{b-j-1}{\tau}}{\binom{b-1}{\tau-1}}
	=\frac{\binom{b-j-1}{\tau-1}}{\binom{b-1}{\tau-1}}.
\end{align*}
Therefore,
\begin{align*}
	h_j(\tau)
	= h_j(\tau) - h_0(\tau)
	= \sum_{m=0}^{j-1} \parens{h_{m+1}(\tau)-h_m(\tau)}
	= \sum_{m=0}^{j-1} \frac{\binom{b-m-1}{\tau-1}}{\binom{b-1}{\tau-1}},
\end{align*}
where
\begin{align*}
	\frac{\binom{b-m-1}{\tau-1}}{\binom{b-1}{\tau-1}}
	= \frac{b - \tau}{b - m - \tau} \frac{\binom{b-m-1}{\tau}}{\binom{b-1}{\tau}}
	\geq \frac{\binom{b-m-1}{\tau}}{\binom{b-1}{\tau}},
\end{align*}
and the inequality is strict for $m>0$. It follows that for any $j \in [b]$
\begin{align*}
	h_j(\tau)
	= \sum_{m=0}^{j-1} \frac{\binom{b-m-1}{\tau-1}}{\binom{b-1}{\tau-1}}
	\geq \sum_{m=0}^{j-1} \frac{\binom{b-m-1}{\tau}}{\binom{b-1}{\tau}}
	= h_j(\tau+1).
\end{align*}
Thus, except for the trivial case when $j=1$ (where $h_1(\tau)\equiv1$), the function $h_j(\tau)$ is strictly decreasing in $\tau$, implying that $h_j(\tau_1) - h_j(\tau_2) > 0$, which proves that $B(\tau_1) > B(\tau_2)$. Thus
\begin{align*}
	B(\tau) \geq B(b) \quad \forall \tau \leq b,
\end{align*}
with strict inequality if $c_{\mathrm{ov}} > 0$ and $\tau < b$.

Now, note that in general larger $\tau$ leads to a higher Lipschitz constant, and hence one may expect $A$ to be an increasing function of $\tau$. If $A(b)$ is much larger than $A(\tau)$, it may compensate for the decrease in $B(\tau)$, resulting in $f(\tau) = A(\tau) B(\tau) < A(b) B(b) = f(b)$, and so $\tau=b$ may not be optimal.
\\

As an example, suppose that $L^0_{i,\tau}$ scales linearly with $\tau$, i.e, $A(\tau) = \max_{i\in[b]} \tau L^0_{i}$ for some $L^0_{i} \geq 0$. Then
\begin{eqnarray*}
    \mathrm{cost}_{\varepsilon}(\cD)
    &\propto& \max_{i\in[b]} L^0_{i} \parens{b c_{\mathrm{ov}} + \sum_{j=1}^b c_j \parens{b - \tau \frac{\binom{b-j}{\tau}}{\binom{b-1}{\tau-1}}} + \tau \sum_{j=1}^b c_j^{\sharp}} \\
	&=& \max_{i\in[b]} L^0_{i} \parens{b c_{\mathrm{ov}}
	+ b \sum_{j=1}^b c_j
	+ \tau \sum_{j=1}^b \parens{c_j^{\sharp} - c_j \frac{\binom{b-j}{\tau}}{\binom{b-1}{\tau-1}}}}.
\end{eqnarray*}
Now, consider the function
\begin{eqnarray*}
	\Phi(\tau) \eqdef \tau \sum_{j=1}^b \parens{c_j^{\sharp} - c_j \frac{\binom{b-j}{\tau}}{\binom{b-1}{\tau-1}}}.
\end{eqnarray*}
Note that
\begin{eqnarray*}
	\Phi(\tau+1) - \Phi(\tau)
    &=& (\tau+1) \sum_{j=1}^b \parens{c_j^{\sharp} - c_j \frac{\binom{b-j}{\tau+1}}{\binom{b-1}{\tau}}}
    - \tau \sum_{j=1}^b \parens{c_j^{\sharp} - c_j \frac{\binom{b-j}{\tau}}{\binom{b-1}{\tau-1}}} \\
    &=& \sum_{j=1}^b c_j \parens{\tau \frac{\binom{b-j}{\tau}}{\binom{b-1}{\tau-1}} - (\tau+1) \frac{\binom{b-j}{\tau+1}}{\binom{b-1}{\tau}}}
    + \sum_{j=1}^b c_j^{\sharp} \\
    &=& \sum_{j=1}^b c_j \frac{j (b-j)! (b-\tau-1)!}{(b-1)! (b-j-\tau)!}
    + \sum_{j=1}^b c_j^{\sharp}
    \geq 0
\end{eqnarray*}
(with the convention that the right-hand side is $0$ when $b-j-\tau<0$).
Moreover, the increment is strictly positive if either $\sum_{j=1}^b c_j^{\sharp}>0$ or there exists $j$ with $c_j>0$ and $b-j-\tau\geq 0$. Hence $\Phi(\tau)$ is non-decreasing in $\tau$, and strictly increasing whenever one of these conditions holds, and the optimal choice is $\tau^{\star}=1$.

\subsubsection{$\tau$-Submodel Sampling}

\begin{mdframed}[linecolor=lightgray!12, backgroundcolor=lightgray!12]
    Sample a starting index $\hat{s}\in\{1,\dots,b-\tau+1\}$ with probability $p_i = \Prob{\hat{s}=i}$ and set $\hat{S}=\{\hat{s},\dots,\hat{s}+\tau-1\}$.
\end{mdframed}

The denominator is
\begin{eqnarray*}
    \Exp{\frac{\I{i\in\hat{S}}}{2L^0_{i,\hat{S}}}}
    = \sum_{j=1}^{b-\tau+1} p_j \frac{\I{i\in\{j,\dots,j+\tau-1\}}}{2L^0_{i,\{j,\dots,j+\tau-1\}}}
    = \sum_{j=\max\{1,i-\tau+1\}}^{\min\{i,b-\tau+1\}} \frac{p_j}{2L^0_{i,\{j,\dots,j+\tau-1\}}}.
\end{eqnarray*}
Hence the total expected cost is proportional to
\begin{eqnarray*}
    \mathrm{cost}_{\varepsilon}(\cD)
    &\propto& \frac{\Exp{\mathrm{cost}(\hat{S})}}
    {\min_{i\in[b]} \Exp{\frac{\I{i\in\hat{S}}}{2L^0_{i,\hat{S}}}}} \\
    &=& \frac{c_{\mathrm{ov}} + \sum_{j=1}^b c_j \left(\sum_{i=1}^{\min\{j,b-\tau+1\}} p_i\right)
    + \sum_{j=1}^b c_j^{\sharp} \left(\sum_{i=\max\{1,j-\tau+1\}}^{\min\{j,b-\tau+1\}} p_i\right)}
    {\min_{i\in[b]} \brac{\sum_{j=\max\{1,i-\tau+1\}}^{\min\{i,b-\tau+1\}} \frac{p_j}{2L^0_{i,\{j,\dots,j+\tau-1\}}}}}.
\end{eqnarray*}

\paragraph{Partitioned $\tau$-submodel sampling.}\label{sec:part_tau_submodel}

Let us consider a special case of the above sampling scheme where the submodels assigned non-zero probability \emph{partition} $[b]$. For simplicity, suppose that $b$ is divisible by $\tau$, let $m=\nicefrac{b}{\tau}$, and define the block start indices via
\begin{eqnarray*}
    s_k=(k-1)\tau+1,\qquad k=1,\dots,m.
\end{eqnarray*}
The algorithm then picks block $B_k \eqdef \{s_k, \dots, s_k+\tau-1\}$ with probability $p_{s_k}> 0$ (where $\sum_{k=1}^m p_{s_k} = 1$). 
This is equivalent to the submodel sampling with starting-index distribution satisfying
\begin{eqnarray*}
p_{s_k}=p_{(k-1)\tau+1} > 0 \quad (k=1,\dots,m), \qquad p_j=0 \text{ otherwise}.
\end{eqnarray*}
Plugging this choice into the general submodel expressions immediately gives
\begin{eqnarray*}
    F_i &=& \sum_{j=1}^{\min\{i, b-\tau+1\}} p_j
    = \sum_{k=1}^{\min\{m, \flr{\nicefrac{(i-1)}{\tau}} + 1\}} p_{s_k}, \\
    Q_i &=& \sum_{j=\max\{1, i-\tau+1\}}^{\min\{i, b-\tau+1\}} p_j
    = p_{s_{\ceil{\nicefrac{i}{\tau}}}}, \\
    \Exp{\frac{\I{i\in\hat{S}}}{2L^0_{i,\hat{S}}}}
    &=& \sum_{j=\max\{1,i-\tau+1\}}^{\min\{i,b-\tau+1\}} \frac{p_j}{2L^0_{i,\{j,\dots,j+\tau-1\}}}
    = \frac{p_{s_{\ceil{\nicefrac{i}{\tau}}}}}{2L^0_{i,B_{\ceil{\nicefrac{i}{\tau}}}}}.
\end{eqnarray*}
Therefore the full cost reduces to
\begin{eqnarray}\label{eq:part_tau_submodel_cost}
    \mathrm{cost}_{\varepsilon}(\cD)
    &\propto& \frac{c_{\mathrm{ov}} + \sum_{j=1}^b c_j \parens{\sum_{k=1}^{\min\{m, \flr{\nicefrac{(j-1)}{\tau}} + 1\}} p_{s_k}}
    + \sum_{j=1}^b c_j^{\sharp} p_{s_{\ceil{\nicefrac{j}{\tau}}}}}{\min_{i\in[b]} \brac{\frac{p_{s_{\ceil{\nicefrac{i}{\tau}}}}}{2L^0_{i,B_{\ceil{\nicefrac{i}{\tau}}}}}}} \nonumber \\
    &=& \frac{c_{\mathrm{ov}} + \sum_{j=1}^b c_j \parens{\sum_{k=1}^{\ceil{\nicefrac{j}{\tau}}} p_{s_k}}
    + \sum_{j=1}^b c_j^{\sharp} p_{s_{\ceil{\nicefrac{j}{\tau}}}}}{\min_{i\in[b]} \brac{\frac{p_{s_{\ceil{\nicefrac{i}{\tau}}}}}{2L^0_{i,B_{\ceil{\nicefrac{i}{\tau}}}}}}}.
\end{eqnarray}
Now, note that
\begin{align*}
    \sum_{j=1}^b c_j \parens{\sum_{k=1}^{\ceil{\nicefrac{j}{\tau}}} p_{s_k}}
    &= \sum_{j=1}^\tau c_j p_{s_1} + \sum_{j=\tau+1}^{2\tau} c_j \parens{p_{s_1} + p_{s_2}} + \ldots + \sum_{j=(m-1)\tau+1}^{m\tau} c_j \parens{p_{s_1} + \ldots + p_{s_m}} \\
    &= p_{s_1} \sum_{j=1}^\tau c_j + \parens{p_{s_1} + p_{s_2}} \sum_{j=\tau+1}^{2\tau} c_j + \ldots + \parens{p_{s_1} + \ldots + p_{s_m}} \sum_{j=(m-1)\tau+1}^{m\tau} c_j \\
    &= p_{s_1} \sum_{j=1}^{m \tau} c_j + p_{s_2} \sum_{j=\tau+1}^{m\tau} c_j + \ldots + p_{s_m} \sum_{j=(m-1)\tau+1}^{m\tau} c_j
\end{align*}
and
\begin{eqnarray*}
    \sum_{j=1}^b c_j^{\sharp} p_{s_{\ceil{\nicefrac{j}{\tau}}}}
    = \sum_{j=1}^\tau c_j^{\sharp} p_{s_1} + \sum_{j=\tau+1}^{2\tau} c_j^{\sharp} p_{s_2} + \ldots + \sum_{j=(m-1)\tau+1}^{m\tau} c_j^{\sharp} p_{s_m}.
\end{eqnarray*}
Hence
\begin{align*}
    &\mathrm{cost}_{\varepsilon}(\cD) \\
    &\propto \frac{1}{\min_{i\in[b]} \brac{\frac{p_{s_{\ceil{\nicefrac{i}{\tau}}}}}{2L^0_{i,B_{\ceil{\nicefrac{i}{\tau}}}}}}} \left[c_{\mathrm{ov}} + p_{s_1} \parens{\sum_{j=1}^{m \tau} c_j + \sum_{j=1}^\tau c_j^{\sharp}} + p_{s_2} \parens{\sum_{j=\tau+1}^{m\tau} c_j + \sum_{j=\tau+1}^{2\tau} c_j^{\sharp}} + \ldots \right. \\
    &\qquad\qquad\qquad\qquad\qquad\qquad\left.+ p_{s_m} \parens{\sum_{j=(m-1)\tau+1}^{m\tau} c_j + \sum_{j=(m-1)\tau+1}^{m\tau} c_j^{\sharp}}\right] \\
    &= \frac{p_{s_1} d_1 + p_{s_2} d_2 + \ldots + p_{s_m} d_m}{\min_{i\in[b]} \brac{\frac{p_{s_{\ceil{\nicefrac{i}{\tau}}}}}{2L^0_{i,B_{\ceil{\nicefrac{i}{\tau}}}}}}},
\end{align*}
where we used the fact that $\sum_{j=1}^{m} p_{s_j} = 1$ and denoted $d_i \eqdef c_{\mathrm{ov}} + \sum_{j=(i-1)\tau+1}^{m \tau} c_j + \sum_{j=(i-1)\tau+1}^{i \tau} c_j^{\sharp}$.

\paragraph{Optimal probabilities for fixed $\tau$.}
We follow an approach similar to that in \Cref{sec:rpt_sol} to find the optimal probabilities for a fixed choice of~$\tau$. For $i\in B_k$, define
\begin{eqnarray*}
	\delta_{i,k}\eqdef\frac{1}{2L^0_{i,B_k}}
\end{eqnarray*}
and set
\begin{eqnarray*}
	\delta_k^{\min}\eqdef \min_{i\in B_k}\delta_{i,k}
	= \frac{1}{2\max_{i\in B_k} L^0_{i,B_k}}.
\end{eqnarray*}
Then the expected cost can be represented as
\begin{eqnarray*}
	\mathrm{cost}_{\varepsilon}(\cD)
	\propto \frac{d_1 p_{s_1} + \cdots + d_m p_{s_m}}{\min_{i\in[b]} \brac{\frac{p_{s_{\ceil{\nicefrac{i}{\tau}}}}}{2L^0_{i,B_{\ceil{\nicefrac{i}{\tau}}}}}}}
	= \frac{d_1 p_{s_1} + \cdots + d_m p_{s_m}}{\min_{k\in[m]} \min_{i\in B_k} \{\delta_{i,k} p_{s_k}\}}.
\end{eqnarray*}
This corresponds to the linear fractional program
\begin{equation}\label{eq:partitioned_submod_lfp}
    \begin{aligned}
        \min_{p,t}\quad & \frac{d_1 p_{s_1} + \cdots + d_m p_{s_m}}{t} \\
        \text{s.t.}\quad & p_{s_1}, \ldots, p_{s_m} \geq 0, \\
        & p_{s_1}+\cdots+p_{s_m} = 1, \\
        & t \geq 0, \\
        & t \leq \delta_{i,k} p_{s_k}, \qquad k\in[m], i\in B_k.
    \end{aligned}
\end{equation}
The standard change of variables $q_k = p_{s_k} / t$ yields the equivalent linear program
\begin{equation}\label{eq:fl2_partitioned}
    \begin{aligned}
        \min_{q} \quad & d_1 q_1 + \cdots + d_m q_m \\
        \text{s.t.} \quad & q_1, \ldots, q_m \geq 0, \\
        & \delta_{i,k} q_k \geq 1 \qquad k\in[m], i\in B_k.
    \end{aligned}
\end{equation}
Because each block $k$ only appears in constraints of the form $\delta_{i,k} q_k \geq 1$ (for $i\in B_k$), the constraints for block $k$ reduce to the single constraint
\begin{eqnarray*}
    \delta_k^{\min} q_k \geq 1
    \quad\Longleftrightarrow\quad
    q_k \geq \frac{1}{\delta_k^{\min}}.
\end{eqnarray*}
Hence \eqref{eq:fl2_partitioned} is separable and its optimal solution is
\begin{eqnarray*}
    q_k^\star = \frac{1}{\delta_k^{\min}}, \qquad k\in[m],
\end{eqnarray*}
with optimal objective value
\begin{eqnarray*}
    \sum_{k=1}^m d_k q_k^\star = \sum_{k=1}^m \frac{d_k}{\delta_k^{\min}}.
\end{eqnarray*}

Now, let us introduce dual variables $\lambda_i\ge0$ for the constraints $\delta_{i,k} q_k \geq 1$, $i\in B_k$. The dual of \eqref{eq:fl2_partitioned} is
\begin{equation*}
    \begin{aligned}
        \max_{\lambda}\quad & \lambda_1 + \ldots + \lambda_b \\
        \text{s.t.}\quad & \lambda_1, \ldots, \lambda_b \geq 0, \\
        & d_k \geq \sum_{i\in B_k} \lambda_i \delta_{i,k}, \qquad k\in[m].
    \end{aligned}
\end{equation*}
To certify optimality of $q^\star$, for each block $k$, choose an index $i_k\in B_k$ attaining the minimum $\delta_k^{\min}$ and set
\begin{eqnarray*}
	\lambda_{i_k} = \frac{d_k}{\delta_{i_k,k}}, \qquad
	\lambda_i=0 \text{ for } i\not\in\{i_1,\dots,i_m\}.
\end{eqnarray*}
Then for each block $k$,
\begin{eqnarray*}
	\sum_{i\in B_k} \delta_{i,k} \lambda_i
	= \delta_{i_k,k} \lambda_{i_k}
	= d_k,
\end{eqnarray*}
so the dual constraints hold with equality and the dual objective equals
\begin{eqnarray*}
	\sum_{i=1}^b \lambda_i = \sum_{k=1}^m \frac{d_k}{\delta_k^{\min}},
\end{eqnarray*}
matching the primal objective. By strong duality $q^\star$ is optimal.

From $q_k^\star=\nicefrac{1}{\delta_k^{\min}}$ and $p_{s_k} = t q_k = q_k / \sum_{l=1}^m q_l$ we obtain the optimal block probabilities
\begin{eqnarray*}
    p_{s_k}^\star
    = \frac{\nicefrac{1}{\delta_k^{\min}}}{\sum_{l=1}^m \nicefrac{1}{\delta_l^{\min}}}
    = \frac{\max_{i\in B_k} L^0_{i,B_k}}{\sum_{l=1}^m \max_{i\in B_l} L^0_{i,B_l}}.
\end{eqnarray*}
so that each block's probability is proportional to the worst-case local smoothness constant inside that block.
For this choice, the minimal expected cost is
\begin{eqnarray*}
	\mathrm{cost}_{\varepsilon}(\cD) \propto \sum_{k=1}^m \frac{d_k}{\delta_k^{\min}}
	= 2\sum_{k=1}^m d_k \max_{i\in B_k} L^0_{i,B_k}.
\end{eqnarray*}

\paragraph{Choosing $\tau$.}
We now show that the cost is not necessarily minimized by choosing $\tau = b$. To this end, we want to minimize the function
\begin{eqnarray*}
    \Phi(\tau) \eqdef 2\sum_{k=1}^{m} d_k(\tau) \max_{i\in B_k} L^0_{i,B_k}
\end{eqnarray*}
(where we explicitly emphasize the dependence of $d_k$ on $\tau$).

To gain intuition about which extreme ($\tau = 1$ or $\tau = b$) may be preferable, assume that the costs are constant, i.e.,
\begin{eqnarray*}
    c_j \equiv c, \qquad c_j^\sharp \equiv c^\sharp,
\end{eqnarray*}
in which case
\begin{align*}
    \sum_{k=1}^m d_k(\tau)
    &= \sum_{k=1}^m \parens{c_{\mathrm{ov}} + \sum_{j=(k-1)\tau+1}^{m \tau} c + \sum_{j=(k-1)\tau+1}^{k\tau} c^{\sharp}} \\
	&= \sum_{k=1}^m \parens{c_{\mathrm{ov}} + c \parens{m \tau - (k-1) \tau} + c^{\sharp} \parens{k\tau - (k-1)\tau}} \\
	&= m c_{\mathrm{ov}} + c \tau \frac{m(m+1)}{2} + m \tau c^\sharp.
\end{align*}
Suppose in addition that the worst-case local smoothness per block does not depend on the block index, i.e.,
\begin{eqnarray*}
    \max_{i\in B_k} L^0_{i,B_k} \equiv L^0(\tau) \quad \forall k\in[m]
\end{eqnarray*}
for some non-decreasing function $L$. Then
\begin{align*}
    \Phi(\tau)
    = 2 L^0(\tau)\sum_{k=1}^{m} d_k(\tau)
	= 2 L^0(\tau) \parens{\frac{b}{\tau} \parens{c_{\mathrm{ov}}+\frac{bc}{2}} + \frac{bc}{2} + b c^\sharp}.
\end{align*}
Thus, the $\tau$-dependence of $\Phi$ is the product of a non-decreasing factor $L_0(\tau)$ and a factor that decreases like $1/\tau$ plus additive constants. Consequently:
\begin{itemize}
    \item If $L^0(\tau)$ is constant in $\tau$ (no worsening with larger blocks), then $\Phi(\tau)$ is strictly decreasing in $\tau$ and the minimizer is $\tau^\star = b$.

	\item If $L^0(\tau)$ grows sublinearly, then $L^0(\tau) \frac{b}{\tau} \parens{c_{\mathrm{ov}}+\frac{bc}{2}}$ is decreasing in $\tau$, while $L^0(\tau) \parens{\frac{c b}{2} + b c^\sharp}$ is increasing. Hence, the optimal $\tau$ may lie strictly between $1$ and $b$, depending on the relative magnitudes of the costs.

    \item If $L^0(\tau)$ increases at least linearly in $\tau$, then $\Phi(\tau)$ is increasing in $\tau$, and hence $\tau^\star = 1$.
\end{itemize}

\subsubsection{Arbitrary Submodel Sampling}\label{sec:arbitrary_submodel}

\begin{mdframed}[linecolor=lightgray!12, backgroundcolor=lightgray!12]
    Let $\{B_1,\dots,B_m\}$ be a partition of $[b]$. Set $\hat{S} = B_i$ with probability $p_i$ (where $\sum_{i=1}^m p_i = 1$).
    For $j\in[b]$, $k(j)$ denotes the unique block with $j\in B_{k(j)}$ and $\underline{b}_k \eqdef \min B_k$.
\end{mdframed}

First, note that the cost can be expressed block-wise as
\begin{align*}
	\Exp{\mathrm{cost}(\hat{S})}
    &= c_{\mathrm{ov}} + \sum_{i=1}^b c_i \parens{\sum_{k:\underline{b}_k \leq i} p_k} + \sum_{i=1}^b c_i^{\sharp} p_{k(i)} \\
    &= c_{\mathrm{ov}} + \sum_{k=1}^m \sum_{i \geq \underline{b}_k} c_i p_k + \sum_{k=1}^m \sum_{i\in B_k} c_i^{\sharp} p_{k} \\
    &= \sum_{k=1}^m d_k p_k,
\end{align*}
where $d_k \eqdef c_{\mathrm{ov}} + \sum_{j\geq \underline{b}_k} c_j + \sum_{j\in B_k} c_j^{\sharp}$.
We also have 
\begin{align*}
	\Exp{\frac{\I{i\in \hat{S}}}{2 L^0_{i,\hat{S}}}}
    = \frac{p_{k(i)}}{2 L^0_{i,B_{k(i)}}},
\end{align*}
and hence
\begin{align*}
	\mathrm{cost}_{\varepsilon}(\cD)
    \propto \frac{\Exp{\mathrm{cost}(\hat{S})}}
    {\min_{i\in[b]} \Exp{\frac{\I{i\in\hat{S}}}{2L^0_{i,\hat{S}}}}}
    = \frac{\sum_{k=1}^m d_k p_k}{\min_{k\in[m]} \min_{i\in B_k} \brac{\frac{p_k}{2 L^0_{i,B_k}}}}.
\end{align*}
Now, define 
\begin{align*}
	\delta_{i,k} \eqdef \tfrac{1}{2 L^0_{i,B_k}}, 
    \qquad 
    \delta_k^{\min} \eqdef \min_{i\in B_k} \delta_{i,k}.
\end{align*}
By the same linear fractional reduction as in \eqref{eq:partitioned_submod_lfp}, the unique optimal solution is
\begin{align*}
	p_k^\star = \frac{\nicefrac{1}{\delta_k^{\min}}}{\sum_{l=1}^m \nicefrac{1}{\delta_l^{\min}}} 
    = \frac{\max_{i\in B_k} L^0_{i,B_k}}{\sum_{l=1}^m \max_{i\in B_l} L^0_{i,B_l}}.
\end{align*}
With this choice, the objective is
\begin{eqnarray}\label{eq:cost_arbitrary_submod}
	\mathrm{cost}_{\varepsilon}(\cD)
    &\propto& 2 \sum_{k=1}^m d_k \max_{i\in B_k} L^0_{i,B_k}
    = 2\sum_{k=1}^m \parens{c_{\mathrm{ov}} + \sum_{j\geq \underline{b}_k} c_j + \sum_{j\in B_k} c_j^{\sharp}} \max_{i\in B_k} L^0_{i,B_k}.
\end{eqnarray}

Let us compare it with the cost for full model training
\begin{eqnarray*}
    \Exp{\mathrm{cost}_{\textnormal{full}}(K)}
    \propto 2 \max_{i\in[b]} L^0_{i,[b]} \parens{c_{\mathrm{ov}} + \sum_{j=1}^b \parens{c_j + c_j^{\sharp}}}.
\end{eqnarray*}

\begin{example}[Grouping layers by cost similarity]
    Suppose that the layers are partitioned into groups according to their similarity, so that within each block, all per-layer costs are (approximately) the same. Concretely, for every $k\in[m]$ and every $j\in B_k$ assume that
    \begin{align*}
        c_j \equiv \underline{c}_k, \qquad c_j^{\sharp}\equiv \underline{c}_k^{\sharp}.
    \end{align*}
    Then
    \begin{eqnarray*}
    	\mathrm{cost}_{\varepsilon}(\cD)
        &\propto& 2\sum_{k=1}^m \parens{c_{\mathrm{ov}} + \sum_{j\geq \underline{b}_k} c_j + \sum_{j\in B_k} c_j^{\sharp}} \max_{i\in B_k} L^0_{i,B_k} \\
        &=& 2 \sum_{k=1}^m \parens{c_{\mathrm{ov}} + \sum_{j\geq \underline{b}_k} c_j + |B_k| \underline{c}_k^{\sharp}} \max_{i\in B_k} L^0_{i,B_k}
    \end{eqnarray*}
    and
    \begin{eqnarray*}
        \Exp{\mathrm{cost}_{\textnormal{full}}(K)}
        \propto 2 \max_{i\in[b]} L^0_{i,[b]} \parens{c_{\mathrm{ov}} + \sum_{l=1}^m |B_l| \parens{\underline{c}_l + \underline{c}_l^{\sharp}}}.
    \end{eqnarray*}
    Therefore, partitioned arbitrary submodel sampling with the optimal probabilities is better than full model training if and only if
    \begin{eqnarray}\label{eq:partition_better_explicit}
        \sum_{k=1}^m \max_{i\in B_k} L^0_{i,B_k} \parens{c_{\mathrm{ov}} + \sum_{j\geq \underline{b}_k} c_j + |B_k| \underline{c}_k^{\sharp}}
        \leq \max_{i\in[b]} L^0_{i,[b]} \parens{c_{\mathrm{ov}} + \sum_{k=1}^m |B_k| \parens{\underline{c}_k + \underline{c}_k^{\sharp}}}.
    \end{eqnarray}
    If each block is much better conditioned than the full model, i.e. $\max_{i\in B_k} L^0_{i,B_k} \ll \max_{i\in[b]} L^0_{i,[b]}$ for all $k$, then the left side of \eqref{eq:partition_better_explicit} can be much smaller than the right side and partitioned submodel sampling will be advantageous even if the block-wise tail sums are moderately large. Indeed, if $\max_{i\in B_k} L^0_{i,B_k} \approx \frac{1}{m} \max_{i\in[b]} L^0_{i,[b]}$, then
    \begin{eqnarray*}
        &&\hspace{-1cm}\sum_{k=1}^m \max_{i\in B_k} L^0_{i,B_k} \parens{c_{\mathrm{ov}} + \sum_{j\geq \underline{b}_k} c_j + |B_k| \underline{c}_k^{\sharp}} \\
        &\approx& \frac{1}{m} \max_{i\in[b]} L^0_{i,[b]} \sum_{k=1}^m \parens{c_{\mathrm{ov}} + \sum_{j\geq \underline{b}_k} c_j + |B_k| \underline{c}_k^{\sharp}} \\
        &=& \max_{i\in[b]} L^0_{i,[b]} \parens{c_{\mathrm{ov}} + \frac{1}{m} \sum_{k=1}^m \sum_{j\geq \underline{b}_k} c_j + \frac{1}{m} \sum_{k=1}^m |B_k| \underline{c}_k^{\sharp}},
    \end{eqnarray*}
    which improves upon full model training if $\frac{1}{m} \sum_{k=1}^m \sum_{j\geq \underline{b}_k} c_j + \frac{1}{m} \sum_{k=1}^m |B_k| \underline{c}_k^{\sharp} \leq \sum_{k=1}^m |B_k| \parens{\underline{c}_k + \underline{c}_k^{\sharp}}$.
\end{example}

\begin{example}[Grouping layers across transformer blocks]
    In their general forms, \eqref{eq:cost_arbitrary_submod} and \eqref{eq:partition_better_explicit} are hard to interpret. Let us consider a simplified language model motivating example behind considering this sampling strategy. Suppose the network is composed of $T$ identical transformer blocks, each containing $L$ layers, so that $b = T L$, with the natural block structure: layer index $j\in[b]$ corresponds to position $l\in\{1,\dots,L\}$ inside transformer block $t\in[T]$
    via $j = l + (t-1) L$.
    
    Consider the partition that groups together \emph{same-position} layers across transformer blocks:
    \begin{align*}
        B_l = \{l, l+L, l+2L, \dots, l+(T-1)L\}, \qquad \forall l\in[L],
    \end{align*}
    so $m=L$ and $|B_l| = T$ for every $l$. For each block $B_l$ we have $\underline{b}_{l}=\min B_l = l$.
    
    In this setting, for block $B_l$ the tail sum that appears in $d_l$ simplifies to
    \begin{align*}
        \sum_{j\geq \underline b_l} c_j = \sum_{j=l}^{b} c_j.
    \end{align*}
    Note that we grouped the layers is such a way that the layers within each block are of the same type, and hence we may expect the costs to be roughly the same within each block. That is, we assume that
    \begin{align*}
        c_{l + tL} \equiv \underline{c}_l,
        \qquad c_{l + tL}^{\sharp}\equiv \underline{c}_l^{\sharp},
        \qquad \forall l\in[L],\ t\in\{0,\ldots,T-1\},
    \end{align*}
    in which case the tail sum further simplifies to
    \begin{align*}
        \sum_{j=l}^{b} c_j
        = \sum_{j=l}^{L} \underline{c}_j + (T-1) \sum_{j=1}^{L} \underline{c}_j.
    \end{align*}
    Hence, under the periodic costs assumption the block-wise cost constants become
    \begin{align*}
        d_l = c_{\mathrm{ov}} + \sum_{j\geq \underline b_l} c_j + \sum_{j\in B_l} c_j^{\sharp}
        = c_{\mathrm{ov}} + \sum_{j=l}^{L} \underline{c}_j + (T-1) \sum_{j=1}^{L} \underline{c}_j + T \underline{c}_l^{\sharp}.
    \end{align*}
    
    If the local smoothness within a block are also approximately homogeneous, i.e.,
    \begin{align*}
        L^0_{i,B_l} \equiv L^0_l \qquad \forall i\in B_l,
    \end{align*}
    then
    \begin{align*}
        \max_{i\in B_l} L^0_{i,B_l} = L^0_l,
    \end{align*}
    and the previously derived expected cost (for the optimal block probabilities) reduces to
    \begin{align*}
        \mathrm{cost}_{\varepsilon}(\cD)
        \propto 2 \sum_{l=1}^L L^0_l \parens{c_{\mathrm{ov}} + \sum_{j=l}^{L} \underline{c}_j + (T-1) \sum_{j=1}^{L} \underline{c}_j + T \underline{c}_l^{\sharp}}.
    \end{align*}
    
    Under the same periodicity assumptions, the full-model expected cost is
    \begin{eqnarray*}
        \Exp{\mathrm{cost}_{\textnormal{full}}(K)}
        \propto 2 \max_{i\in[b]} L^0_{i,[b]} \parens{c_{\mathrm{ov}} + T \sum_{j=1}^L \underline c_j + T \sum_{j=1}^L \underline c_j^\sharp}.
    \end{eqnarray*}
    Comparing the two, partitioned sampling is better if
    \begin{align}\label{eq:compare_exact}
        \sum_{l=1}^L L^0_l \parens{c_{\mathrm{ov}} + \sum_{j=l}^{L} \underline{c}_j + (T-1) \sum_{j=1}^{L} \underline{c}_j + T \underline{c}_l^{\sharp}}
        \leq \max_{i\in[b]} L^0_{i,[b]} \parens{c_{\mathrm{ov}} + T \sum_{j=1}^L \underline c_j + T \sum_{j=1}^L \underline c_j^\sharp}.
    \end{align}
    Now, let us denote
    \begin{align*}
        \underline{C} \eqdef \sum_{j=1}^L \underline{c}_j,\qquad
        \underline{C}^{\sharp} \eqdef \sum_{j=1}^L \underline c_j^\sharp.
    \end{align*}
    Then \eqref{eq:compare_exact} is equivalent to
    \begin{align}\label{eq:compare_rewrite}
        \sum_{l=1}^L L^0_l \parens{c_{\mathrm{ov}} + \sum_{j=l}^L \underline{c}_j}
        + T \sum_{l=1}^L L^0_l \parens{\frac{T-1}{T} \underline{C} + \underline{c}_l^{\sharp}}
        \leq T \parens{\underline{C} + \underline{C}^{\sharp}} \max_{i\in[b]} L^0_{i,[b]}
        + c_{\mathrm{ov}} \max_{i\in[b]} L^0_{i,[b]}.
    \end{align}
    For large $T$, the dominant terms are those multiplied by $T$. Then, the leading-order criterion becomes
    \begin{align}\label{eq:compare_asymp}
        \sum_{l=1}^L L^0_l \parens{\underline{C} + \underline{c}_l^\sharp}
        \leq \max_{i\in[b]} L^0_{i,[b]}(\underline{C} + \underline{C}^{\sharp}),
    \end{align}
    which can hold when $L^0_l \ll \max_{i\in[b]} L^0_{i,[b]}$.
\end{example}

\subsection{Generalized Smooth Case}\label{sec:cost_opt_l0l1}

According to \Cref{thm:rt_l0l1_iter}, under \Cref{as:arbitrary_layer_gen_smoothness}, \Cref{alg:rt_arbitrary} run with stepsizes $\gamma_i^k = \big(L_{i,S^k}^0 + L_{i,S^k}^1 \norm{\nabla_i f(X^k)}_{(i) \star}\big)^{-1}$ guarantees that
\begin{eqnarray}\label{eq:avndion}
	\min_{k=0,\ldots,K-1} \sum_{i=1}^b \left[\frac{w_i}{\frac{1}{b} \sum_{l=1}^b w_l} \Exp{\norm{\nabla _i f(X^k)}_{(i) \star}}\right] \leq \underline{\varepsilon},
\end{eqnarray}
after
\begin{eqnarray}\label{eq:oinqvr}
	K &=& \left\lceil \frac{2 \delta^0 \sum_{i=1}^b \frac{\Prob{i\in \hat{S}} \ExpCond{L^0_{i,\hat{S}}}{i\in \hat{S}}}{\parens{\ExpCond{L^1_{i,\hat{S}}}{i\in \hat{S}}}^2}}{\underline{\varepsilon}^2 \parens{\frac{1}{b} \sum_{l=1}^b \frac{\Prob{l\in \hat{S}}}{\ExpCond{L^1_{l,\hat{S}}}{l\in \hat{S}}}}^2}
	+ \frac{2 \delta^0}{\underline{\varepsilon} \parens{\frac{1}{b} \sum_{l=1}^b \frac{\Prob{l \in \hat{S}}}{\ExpCond{L^1_{l,\hat{S}}}{l\in \hat{S}}}}} \right\rceil
\end{eqnarray}
iterations, where $\delta^0 \eqdef f(X^0) - f^{\star}$ and $w_i \eqdef \frac{\Prob{i\in \hat{S}}}{\ExpCond{L^1_{i,\hat{S}}}{i\in \hat{S}}}$. To obtain a guarantee on an unweighted gradient sum, note that
\begin{eqnarray*}
	\underline{\varepsilon} &\geq& \min_{k=0,\ldots,K-1} \sum_{i=1}^b \left[\frac{w_i}{\frac{1}{b} \sum_{l=1}^b w_l} \Exp{\norm{\nabla _i f(X^k)}_{(i) \star}}\right] \\
	&\geq& \frac{\min_{i\in[b]} w_i}{\frac{1}{b} \sum_{l=1}^b w_l} \min_{k=0,\ldots,K-1} \sum_{i=1}^b \Exp{\norm{\nabla _i f(X^k)}_{(i) \star}},
\end{eqnarray*}
and hence, substituting in \eqref{eq:avndion} and \eqref{eq:oinqvr}, we have
\begin{eqnarray*}
	\min_{k=0,\ldots,K-1} \sum_{i=1}^b \Exp{\norm{\nabla _i f(X^k)}_{(i) \star}}
	\leq \frac{\underline{\varepsilon} \frac{1}{b} \sum_{l=1}^b w_l}{\min_{i\in[b]} w_i} \eqdef \varepsilon
\end{eqnarray*}
after
\begin{eqnarray*}
	K &=& \left\lceil \frac{2 \delta^0 \sum_{i=1}^b \frac{\Prob{i\in \hat{S}} \ExpCond{L^0_{i,\hat{S}}}{i\in \hat{S}}}{\parens{\ExpCond{L^1_{i,\hat{S}}}{i\in \hat{S}}}^2}}{\varepsilon^2 \min_{i\in[b]} \parens{\frac{\Prob{i\in \hat{S}}}{\ExpCond{L^1_{i,\hat{S}}}{i\in \hat{S}}}}^2}
	+ \frac{2 \delta^0}{\varepsilon \min_{i\in[b]} \frac{\Prob{i\in \hat{S}}}{\ExpCond{L^1_{i,\hat{S}}}{i\in \hat{S}}}} \right\rceil
\end{eqnarray*}
iterations. Moreover, recall from \eqref{eq:cost_gen} that the expected cost of a single iteration is
\begin{eqnarray*}
    \Exp{\mathrm{cost}(\hat{S})}
    = c_{\mathrm{ov}} + \sum_{i=1}^b c_i F_i + \sum_{i=1}^b c_i^{\sharp} Q_i.
\end{eqnarray*}
Hence, using the fact that
\begin{eqnarray*}
	\ExpCond{L^{\alpha}_{i,\hat{S}}}{i\in \hat{S}}
	= \frac{\Exp{L^{\alpha}_{i,\hat{S}} \I{i\in \hat{S}}}}{\Prob{i\in \hat{S}}}
\end{eqnarray*}
for $\alpha \in \{0,1\}$, the expected cost of the entire optimization procedure is
\begin{align*}
    &\mathrm{cost}_{\varepsilon}(\cD)
    = K \times \Exp{\mathrm{cost}(\hat{S})} \\
    &= \left\lceil \frac{2 \delta^0 \sum_{i=1}^b \frac{\Prob{i\in \hat{S}} \ExpCond{L^0_{i,\hat{S}}}{i\in \hat{S}}}{\parens{\ExpCond{L^1_{i,\hat{S}}}{i\in \hat{S}}}^2}}{\varepsilon^2 \min_{i\in[b]} \parens{\frac{\Prob{i\in \hat{S}}}{\ExpCond{L^1_{i,\hat{S}}}{i\in \hat{S}}}}^2}
	+ \frac{2 \delta^0}{\varepsilon \min_{i\in[b]} \frac{\Prob{i\in \hat{S}}}{\ExpCond{L^1_{i,\hat{S}}}{i\in \hat{S}}}} \right\rceil
    \parens{c_{\mathrm{ov}} + \sum_{i=1}^b c_i F_i + \sum_{i=1}^b c_i^{\sharp} Q_i} \\
    &= \left\lceil \frac{2 \delta^0 \sum_{i=1}^b \frac{\Prob{i\in \hat{S}}^2 \Exp{L^0_{i,\hat{S}} \I{i\in \hat{S}}}}{\parens{\Exp{L^1_{i,\hat{S}} \I{i\in \hat{S}}}}^2}}{\varepsilon^2 \min_{i\in[b]} \parens{\frac{\Prob{i\in \hat{S}}^2}{\Exp{L^1_{i,\hat{S}} \I{i\in \hat{S}}}}}^2}
	+ \frac{2 \delta^0}{\varepsilon \min_{i\in[b]} \frac{\Prob{i\in \hat{S}}^2}{\Exp{L^1_{i,\hat{S}} \I{i\in \hat{S}}}}} \right\rceil
    \parens{c_{\mathrm{ov}} + \sum_{i=1}^b c_i F_i + \sum_{i=1}^b c_i^{\sharp} Q_i}.
\end{align*}
We will consider two regimes.

\textbf{(1) The $\cO\parens{\nicefrac{1}{\varepsilon^2}}$ term dominates.}
Then the problem to be solved is
\begin{eqnarray}\label{eq:cost_opt_prob_e2}
    \min_{\cD: \mathfrak{P}([b]) \to [0,1], \sum_{S\subseteq[b]} \cD(S) = 1} \underbrace{\frac{\sum_{i=1}^b \frac{\Prob{i\in \hat{S}}^2 \Exp{L^0_{i,\hat{S}} \I{i\in \hat{S}}}}{\parens{\Exp{L^1_{i,\hat{S}} \I{i\in \hat{S}}}}^2}}{\min_{i\in[b]} \parens{\frac{\Prob{i\in \hat{S}}^2}{\Exp{L^1_{i,\hat{S}} \I{i\in \hat{S}}}}}^2}
	\parens{c_{\mathrm{ov}} + \sum_{i=1}^b c_i F_i + \sum_{i=1}^b c_i^{\sharp} Q_i}}_{\propto \, \Exp{\mathrm{cost}_{\varepsilon^2}(K)}}.
\end{eqnarray}

\textbf{(2) The $\cO\parens{\nicefrac{1}{\varepsilon}}$ term dominates.}
Then the problem to be solved is
\begin{eqnarray}\label{eq:cost_opt_prob_e}
    \min_{\cD: \mathfrak{P}([b]) \to [0,1], \sum_{S\subseteq[b]} \cD(S) = 1} \underbrace{\frac{1}{\min_{i\in[b]} \frac{\Prob{i\in \hat{S}}^2}{\Exp{L^1_{i,\hat{S}} \I{i\in \hat{S}}}}}
	\parens{c_{\mathrm{ov}} + \sum_{i=1}^b c_i F_i + \sum_{i=1}^b c_i^{\sharp} Q_i}}_{\propto \, \Exp{\mathrm{cost}_{\varepsilon}(K)}}.
\end{eqnarray}
As in \Cref{sec:cost_opt_smooth}, both tasks above involve optimization over probability distributions, which is intractable in general. Again, for certain parametric families, the objective simplifies to a linear-fractional program, which can be solved efficiently. Let us consider some specific examples.

\subsubsection{Randomized Progressive Training}\label{sec:rpt_sol_l0l1}

\begin{mdframed}[linecolor=lightgray!12, backgroundcolor=lightgray!12]
    Sample $\hat{s}\in\{1,\dots,b\}$, where $p_i = \Prob{\hat{s}=i}$, and set $\hat{S}=\{\hat{s},\dots,b\}$.
\end{mdframed}

In this case, $\hat{S} = \{s,\ldots,b\}$ with probability $p_s$ for all $s\in[b]$, and hence
\begin{eqnarray*}
	\Prob{i\in \hat{S}}
	= \Prob{s \leq i}
	= \sum_{s=1}^i p_s,
\end{eqnarray*}
and
\begin{eqnarray*}
	\Exp{L^{\alpha}_{i,\hat{S}} \I{i\in \hat{S}}}
	= \sum_{s=1}^i p_s L^{\alpha}_{i, \{s,\dots,b\}}
\end{eqnarray*}
for $\alpha\in\{0,1\}$.
Moreover, following the same steps as in \Cref{sec:rpt_sol},
\begin{eqnarray*}
    \Exp{\mathrm{cost}(\hat{S})}
    = c_{\mathrm{ov}} + \sum_{j=1}^b (c_j + c_j^{\sharp}) \sum_{i=1}^j p_i
    = \sum_{i=1}^b \left[c_{\mathrm{ov}}+\sum_{j\geq i} (c_j + c_j^{\sharp})\right] p_i.
\end{eqnarray*}
Hence, substituting into \eqref{eq:cost_opt_prob_e2} and \eqref{eq:cost_opt_prob_e}, the respective optimization problems reduce to minimizing
\begin{eqnarray*}
    \Exp{\mathrm{cost}_{\varepsilon^2}(K)}
	&\propto& \frac{\sum_{i=1}^b \frac{\Prob{i\in \hat{S}}^2 \Exp{L^0_{i,\hat{S}} \I{i\in \hat{S}}}}{\parens{\Exp{L^1_{i,\hat{S}} \I{i\in \hat{S}}}}^2}}{\min_{i\in[b]} \parens{\frac{\Prob{i\in \hat{S}}^2}{\Exp{L^1_{i,\hat{S}} \I{i\in \hat{S}}}}}^2}
	\parens{\sum_{i=1}^b \left[c_{\mathrm{ov}}+\sum_{j\geq i} (c_j + c_j^{\sharp})\right] p_i} \\
	&=& \frac{\sum_{i=1}^b \frac{\parens{\sum_{s=1}^i p_s}^2 \sum_{s=1}^i p_s L^0_{i, \{s,\dots,b\}}}{\parens{\sum_{s=1}^i p_s L^1_{i, \{s,\dots,b\}}}^2}}{\min_{i\in[b]} \parens{\frac{\parens{\sum_{s=1}^i p_s}^2}{\sum_{s=1}^i p_s L^1_{i, \{s,\dots,b\}}}}^2}
	\parens{\sum_{i=1}^b \left[c_{\mathrm{ov}}+\sum_{j\geq i} (c_j + c_j^{\sharp})\right] p_i}
\end{eqnarray*}
and
\begin{eqnarray*}
	\Exp{\mathrm{cost}_{\varepsilon}(K)}
	&\propto& \frac{1}{\min_{i\in[b]} \frac{\Prob{i\in \hat{S}}^2}{\Exp{L^1_{i,\hat{S}} \I{i\in \hat{S}}}}}
	\parens{\sum_{i=1}^b \left[c_{\mathrm{ov}}+\sum_{j\geq i} (c_j + c_j^{\sharp})\right] p_i} \\
	&=& \frac{1}{\min_{i\in[b]} \frac{\parens{\sum_{s=1}^i p_s}^2}{\sum_{s=1}^i p_s L^1_{i, \{s,\dots,b\}}}}
	\parens{\sum_{i=1}^b \left[c_{\mathrm{ov}}+\sum_{j\geq i} (c_j + c_j^{\sharp})\right] p_i}.
\end{eqnarray*}

Let us first focus on $\Exp{\mathrm{cost}_{\varepsilon}(K)}$, following a similar reasoning to that in \Cref{sec:rpt_sol}. 
Denote $d_i \eqdef c_{\mathrm{ov}}+\sum_{j\geq i} (c_j + c_j^{\sharp})$ for $i\in [b]$ and consider the linear fractional program
\begin{equation}\label{eq:fl1l0l1}
    \begin{aligned}
        \min_{p, t} \quad & \frac{d_1 p_1+\cdots+d_b p_b}{t} \\
        \text{s.t.} \quad & p_1,\ldots,p_b\geq 0\\
        & p_1+\cdots+p_b=1\\
        & t\geq 0 \\
        & t\leq \frac{\parens{\sum_{s=1}^i p_s}^2}{\sum_{s=1}^i p_s L^1_{i, \{s,\dots,b\}}}, \quad i \in [b].
    \end{aligned}
\end{equation}
This program can be written equivalently as
\begin{equation}\label{eq:fl2l0l1}
    \begin{aligned}
        \min_{q} \quad & {d_1 q_1+\cdots+d_b q_b} \\
        \text{s.t.} \quad & q_1,\ldots,q_b\geq 0\\
        & g_i(q) \eqdef \parens{\sum_{s=1}^i q_s}^2 - \sum_{s=1}^i q_s L^1_{i, \{s,\dots,b\}} \geq 0, \quad i \in [b].
    \end{aligned}
\end{equation}
We now write the KKT conditions for \eqref{eq:fl2l0l1}.
Introduce multipliers $\lambda_i\ge0$ for the constraints $g_i(q)\ge0$ and multipliers $\eta_k\ge0$ for the non-negativity constraints $q_k\ge0$. The Lagrangian is
\begin{eqnarray*}
    \cL(q,\lambda,s) = \sum_{i=1}^b d_i q_i - \sum_{i=1}^b \lambda_i g_i(q) - \sum_{k=1}^b \eta_k q_k.
\end{eqnarray*}
Fix $k\in[b]$. Differentiating $g_i$ with respect to $q_k$ yields
\begin{equation}\label{eq:dgi_dqk}
	\frac{\partial g_i(q)}{\partial q_k}
	= \begin{cases}
		2 \sum_{s=1}^i q_s - L^1_{i,\{k,\dots,b\}}, & k\leq i,\\[4pt]
		0, & k>i.
	\end{cases}
\end{equation}
Thus, stationarity $\nabla_q \cL(q,\lambda,s)=0$ yields, component-wise for $k\in[b]$,
\begin{eqnarray}\label{eq:stationarity}
	0 = d_k - \sum_{i=k}^b \lambda_i \parens{2 \sum_{s=1}^i q_s - L^1_{i,\{k,\dots,b\}}} - \eta_k,
\end{eqnarray}
and complementarity and sign conditions are
\begin{eqnarray*}
	\lambda_i g_i(q) = 0, \qquad
	\eta_i q_i = 0, \qquad
	\lambda_i, \eta_i \geq 0 \qquad i\in[b].
\end{eqnarray*}

\begin{lemma}\label{lem:active_exists}
	Let $q^\star = (q_1^\star, \ldots, q_p^\star)$ be a local minimizer of \eqref{eq:fl2l0l1} and let $\supp(q^\star)\eqdef\{k:q^\star_k>0\}$. Then, for every index $k\in\supp(q^\star)$, there exists at least one index $i$ with $i\geq k$ such that $\lambda_i>0$. In particular, not all $\lambda_i$ can be zero.
\end{lemma}

\begin{proof}
	Fix $k\in\supp(q^\star)$. If $\lambda_i=0$ for all $i\geq k$, then \eqref{eq:stationarity} at $k$ reduces to
	\begin{eqnarray*}
		0 = d_k - \eta_k.
	\end{eqnarray*}
	By complementary slackness $\eta_k q_k^\star = 0$, and since $q_k^\star>0$, we must have $\eta_k=0$, so $d_k=0$ as well. But by definition $d_k = c_{\mathrm{ov}}+\sum_{j\geq k}(c_j+c_j^\sharp)$, which is (under the model assumptions on the costs) strictly positive. Hence there must exist $i\geq k$ with $\lambda_i>0$. In particular not all $\lambda_i$ are zero.
\end{proof}

\begin{lemma}\label{lem:single_feasibility}
	If $\supp(q)=\{k\}$, then
	\begin{eqnarray*}
		q_k \geq \max_{i\in\{k,\dots,b\}} L^1_{i,\{k,\dots,b\}}.
	\end{eqnarray*}
\end{lemma}

\begin{proof}
	When $q_j=0$ for $j \neq k$, we have
	\begin{eqnarray*}
		\sum_{s=1}^i q_s =
		\begin{cases}
		0 & i<k,\\
		q_k & i\geq k.
		\end{cases}
	\end{eqnarray*}
	For $i<k$, trivially $g_i(q)=0$. For $i \geq k$, we have
	\begin{eqnarray*}
		g_i(q) = q_k^2 - q_k\,L^1_{i,\{k,\dots,b\}} = q_k \parens{q_k - L^1_{i,\{k,\dots,b\}}}.
	\end{eqnarray*}
	Since $q_k>0$, the constraint $g_i(q) \geq 0$ is equivalent to $q_k \geq L^1_{i,\{k,\dots,b\}}$. This must hold for every $i\geq k$, and hence $q_k$ is at least the stated maximum.
\end{proof}

\begin{lemma}\label{lem:single_KKT}
	Let $q^\star$ be a KKT point of \eqref{eq:fl2l0l1} such that $\supp(q^\star)=\{k\}$. Then
	\begin{eqnarray*}
		q_k^\star = \max_{i\geq k} L^1_{i,\{k,\dots,b\}}.
	\end{eqnarray*}
\end{lemma}

\begin{proof}
	By Lemma \ref{lem:single_feasibility} we already have $q_k^\star \geq \max_{i\geq k} L^1_{i,\{k,\dots,b\}}$. Assume that the inequality is strict, that is,
	\begin{eqnarray*}
		q_k^\star > \max_{i\geq k} L^1_{i,\{k,\dots,b\}}.
	\end{eqnarray*}
	Then for every $i\geq k$ we have
	\begin{eqnarray*}
		g_i(q^\star) = q^\star_k \parens{q^\star_k - L^1_{i,\{k,\dots,b\}}} > 0.
	\end{eqnarray*}
	Since by complementary slackness $\lambda_i g_i(q^\star)=0$, we get $\lambda_i = 0$ for all $i\geq k$. Plugging this into the stationarity condition \eqref{eq:stationarity} yields
	\begin{eqnarray*}
	0 = d_k - \eta_k.
	\end{eqnarray*}
	But $q_k^\star>0$ forces $\eta_k=0$, and thus $d_k=0$, yielding a contradiction. Therefore, the strict inequality cannot hold, and we must have
	\begin{eqnarray*}
		q_k^\star = \max_{i\geq k} L^1_{i,\{k,\dots,b\}}.
	\end{eqnarray*}
\end{proof}

\begin{theorem}\label{thm:rpt_p1_l0l1}
    Unless
    \begin{equation*}
        L^1_{1, [b]} = \max_{i\in[b]} L^1_{i, [b]},
    \end{equation*}
    $(p_1, p_2, \ldots, p_b)=(1,0,\ldots,0)$ is \emph{not} an optimal solution of~\eqref{eq:fl1l0l1}.
\end{theorem}

\begin{proof}
    According to \Cref{rem:rt_l0l1_wi}, the sampling must be such that
    \begin{align*}
    	\Prob{i\in \hat{S}} = \sum_{s=1}^i p_s > 0 \qquad \forall i\in[b],
    \end{align*}
    and hence we must have $p_1>0$. Thus, the support of the solution $p^\star$ of \eqref{eq:fl1l0l1} (which coincides with the support of the solution $q^\star$ of \eqref{eq:fl2l0l1}) can only be a singleton if $\supp(p^\star) = \supp(q^\star) = \{1\}$. But then, by \Cref{lem:single_KKT}, we must have
    \begin{eqnarray*}
        q_1^\star = \max_{i\in[b]} L^1_{i,[b]}, \quad q_2^\star = \ldots = q_p^\star = 0,
    \end{eqnarray*}
    which in turn implies that
    \begin{align*}
        0 \leq g_i(q^\star)
        = \parens{\sum_{s=1}^i q_s^\star}^2 - \sum_{s=1}^i q_s^\star L^1_{i, \{s,\dots,b\}}
        = q_1^\star \parens{q_1^\star - L^1_{i, [b]}}
    \end{align*}
    for all $i \in [b]$, so in particular $q_1^\star \parens{q_1^\star - L^1_{1, [b]}} \geq 0$. Now, suppose $L^1_{1, [b]} \neq \max_{i\in[b]} L^1_{i, [b]}$. Then the inequality is strict and, by complementary slackness, $\lambda_1=\eta_1=0$, so
    \begin{align*}
    	d_1 = \sum_{i=2}^b \lambda_i \parens{2 q_1^\star - L^1_{i, \{1,\dots,b\}}}
    	\leq \sum_{i=2}^b \lambda_i \parens{2 q_1^\star - L^1_{i, \{2,\dots,b\}}}
    	= d_2 - \eta_2
    	\leq d_2,
    \end{align*}
    which is a contradiction. Thus, similar to the result in \Cref{rem:rpt_p1}, we must have $L^1_{1, [b]} = \max_{i\in[b]} L^1_{i, [b]}$.
\end{proof}

\subsubsection{$\tau$-Nice Sampling}

\begin{mdframed}[linecolor=lightgray!12, backgroundcolor=lightgray!12]
    Choose $S$ uniformly from all subsets of $[b]$ of size $\tau$.
\end{mdframed}

For every $i\in[b]$ we have
\begin{eqnarray*}
	\Prob{i \in \hat{S}}=\frac{\binom{b-1}{\tau-1}}{\binom{b}{\tau}}=\frac{\tau}{b}
\end{eqnarray*}
and
\begin{eqnarray*}
	\Exp{L^{\alpha}_{i,\hat{S}} \I{i\in\hat{S}}}
    = \frac{1}{\binom{b}{\tau}} \sum_{\substack{S\subseteq[b]\\|S|=\tau}} L^{\alpha}_{i,S} \I{i\in S}.
\end{eqnarray*}
for $\alpha\in\{0,1\}$. Recall that
\begin{eqnarray*}
	\Exp{\mathrm{cost}(S)}
	= c_{\mathrm{ov}} + \sum_{j=1}^b c_j \parens{1 - \frac{\binom{b-j}{\tau}}{\binom{b}{\tau}}} + \frac{\tau}{b} \sum_{j=1}^b c_j^{\sharp}.
\end{eqnarray*}
Hence, substituting into \eqref{eq:cost_opt_prob_e2} and \eqref{eq:cost_opt_prob_e}, the objective functions to minimize are
\begin{eqnarray*}
    &&\hspace{-8mm}\Exp{\mathrm{cost}_{\varepsilon^2}(K)} \\
	&\propto& \frac{\sum_{i=1}^b \frac{\Prob{i \in \hat{S}}^2 \Exp{L^0_{i,\hat{S}} \I{i \in \hat{S}}}}{\parens{\Exp{L^1_{i,\hat{S}} \I{i \in \hat{S}}}}^2}}{\min_{i\in[b]} \parens{\frac{\Prob{i \in \hat{S}}^2}{\Exp{L^1_{i,\hat{S}} \I{i \in \hat{S}}}}}^2}
	\parens{c_{\mathrm{ov}} + \sum_{j=1}^b c_j \parens{1 - \frac{\binom{b-j}{\tau}}{\binom{b}{\tau}}} + \frac{\tau}{b} \sum_{j=1}^b c_j^{\sharp}} \\
	&=& \frac{\max_{i\in[b]} \parens{\sum_{\substack{S\subseteq[b]\\|S|=\tau}} L^1_{i,S} \I{i\in S}}^2}{\frac{\tau}{b} \binom{b-1}{\tau-1}} \\
    &&\qquad\times\sum_{i=1}^b \frac{\sum_{\substack{S\subseteq[b]\\|S|=\tau}} L^0_{i,S} \I{i\in S}}{\parens{\sum_{\substack{S\subseteq[b]\\|S|=\tau}} L^1_{i,S} \I{i\in S}}^2}
	\parens{c_{\mathrm{ov}} + \sum_{j=1}^b c_j \parens{1 - \frac{\binom{b-j}{\tau}}{\binom{b}{\tau}}} + \frac{\tau}{b} \sum_{j=1}^b c_j^{\sharp}}
\end{eqnarray*}
and
\begin{eqnarray*}
	\Exp{\mathrm{cost}_{\varepsilon}(K)}
	&\propto& \frac{1}{\min_{i\in[b]} \frac{\Prob{i \in \hat{S}}^2}{\Exp{L^1_{i,\hat{S}} \I{i \in \hat{S}}}}}
	\parens{c_{\mathrm{ov}} + \sum_{j=1}^b c_j \parens{1 - \frac{\binom{b-j}{\tau}}{\binom{b}{\tau}}} + \frac{\tau}{b} \sum_{j=1}^b c_j^{\sharp}} \\
	&=& \frac{\max_{i\in[b]} \sum_{\substack{S\subseteq[b]\\|S|=\tau}} L^1_{i,S} \I{i\in S}}{\parens{\frac{\tau}{b}}^2 \binom{b}{\tau}}
	\parens{c_{\mathrm{ov}} + \sum_{j=1}^b c_j \parens{1 - \frac{\binom{b-j}{\tau}}{\binom{b}{\tau}}} + \frac{\tau}{b} \sum_{j=1}^b c_j^{\sharp}}.
\end{eqnarray*}
As in \Cref{sec:tau_nice_smooth}, in general there is no closed-form expression for $\tau$ minimizing these costs, but it can be shown that $\tau=b$ need not be optimal.
Let us make a simplifying assumption that $L^0_{i,S}$ depends only on $i$ and $|S|$ and define $L^0_{i,\tau} \eqdef L^0_{i,S}$ for $|S|=\tau$. Then
\begin{eqnarray*}
    &&\hspace{-8mm}\Exp{\mathrm{cost}_{\varepsilon^2}(K)} \\
	&\propto& \frac{\max_{i\in[b]} \parens{\sum_{\substack{S\subseteq[b]\\|S|=\tau}} L^1_{i,\tau} \I{i\in S}}^2}{\frac{\tau}{b} \binom{b-1}{\tau-1}} \\
    &&\qquad\times\sum_{i=1}^b \frac{\sum_{\substack{S\subseteq[b]\\|S|=\tau}} L^0_{i,\tau} \I{i\in S}}{\parens{\sum_{\substack{S\subseteq[b]\\|S|=\tau}} L^1_{i,\tau} \I{i\in S}}^2}
	\parens{c_{\mathrm{ov}} + \sum_{j=1}^b c_j \parens{1 - \frac{\binom{b-j}{\tau}}{\binom{b}{\tau}}} + \frac{\tau}{b} \sum_{j=1}^b c_j^{\sharp}} \\
	&=& \frac{\max_{i\in[b]} \parens{L^1_{i,\tau} \binom{b-1}{\tau-1}}^2}{\frac{\tau}{b} \binom{b-1}{\tau-1}} \sum_{i=1}^b \frac{L^0_{i,\tau} \binom{b-1}{\tau-1}}{\parens{L^1_{i,\tau} \binom{b-1}{\tau-1}}^2}
	\parens{c_{\mathrm{ov}} + \sum_{j=1}^b c_j \parens{1 - \frac{\binom{b-j}{\tau}}{\binom{b}{\tau}}} + \frac{\tau}{b} \sum_{j=1}^b c_j^{\sharp}} \\
	&=& \underbrace{\max_{i\in[b]} \parens{L^1_{i,\tau}}^2 \sum_{i=1}^b \frac{L^0_{i,\tau}}{\parens{L^1_{i,\tau}}^2}}_{\eqdef A_{\varepsilon^2}(\tau)}
	\underbrace{\parens{\frac{b}{\tau} c_{\mathrm{ov}} + \sum_{j=1}^b c_j \parens{\frac{b}{\tau} - \frac{\binom{b-j}{\tau}}{\binom{b-1}{\tau-1}}} + \sum_{j=1}^b c_j^{\sharp}}}_{\eqdef B(\tau)}
\end{eqnarray*}
and
\begin{eqnarray*}
	\Exp{\mathrm{cost}_{\varepsilon}(K)}
	&\propto& \frac{\max_{i\in[b]} \sum_{\substack{S\subseteq[b]\\|S|=\tau}} L^1_{i,\tau} \I{i\in S}}{\parens{\frac{\tau}{b}}^2 \binom{b}{\tau}}
	\parens{c_{\mathrm{ov}} + \sum_{j=1}^b c_j \parens{1 - \frac{\binom{b-j}{\tau}}{\binom{b}{\tau}}} + \frac{\tau}{b} \sum_{j=1}^b c_j^{\sharp}} \\
	&=& \frac{\max_{i\in[b]} L^1_{i,\tau} \binom{b-1}{\tau-1}}{\parens{\frac{\tau}{b}}^2 \binom{b}{\tau}}
	\parens{c_{\mathrm{ov}} + \sum_{j=1}^b c_j \parens{1 - \frac{\binom{b-j}{\tau}}{\binom{b}{\tau}}} + \frac{\tau}{b} \sum_{j=1}^b c_j^{\sharp}} \\
	&=& \underbrace{\max_{i\in[b]} L^1_{i,\tau}}_{\eqdef A_{\varepsilon}(\tau)}
	\underbrace{\parens{\frac{b}{\tau} c_{\mathrm{ov}} + \sum_{j=1}^b c_j \parens{\frac{b}{\tau} - \frac{\binom{b-j}{\tau}}{\binom{b-1}{\tau-1}}} + \sum_{j=1}^b c_j^{\sharp}}}_{\eqdef B(\tau)}.
\end{eqnarray*}
We have already established in \Cref{sec:tau_nice_smooth} that $B$ is decreasing in $\tau$. In fact, the expression to be minimized there is structurally very similar to these obtained above. Specifically, in the smooth case we had $\mathrm{cost}_{\varepsilon}(\cD) \propto A(\tau) B(\tau)$, where $A(\tau) \eqdef \max_{i\in[b]} L^0_{i,\tau}$. Since $L^1_{i,\tau}$, like $L^0_{i,\tau}$, is non-decreasing in $\tau$, the same reasoning as in \Cref{sec:tau_nice_smooth} applies to $\Exp{\mathrm{cost}_{\varepsilon}(K)}$.
On the other hand, the dependence of $A_{\varepsilon^2}(\tau)$ on $\tau$ can be arbitrary depending on the scaling between $L^0_{i,\tau}$ and~$L^1_{i,\tau}$.

Consequently, both objectives $\Exp{\mathrm{cost}_{\varepsilon}(K)}$ and $\Exp{\mathrm{cost}_{\varepsilon^2}(K)}$ present a trade-off between a decreasing factor $B(\tau)$ and a (typically) increasing factor $A_1(\tau)$. If $A_{\varepsilon}(\tau)$ or $A_{\varepsilon^2}(\tau)$ grows sufficiently fast in $\tau$, it can compensate for the decrease of $B(\tau)$ and make a smaller $\tau$ (or even $\tau=1$) optimal.

\subsubsection{$\tau$-Submodel Sampling}

\begin{mdframed}[linecolor=lightgray!12, backgroundcolor=lightgray!12]
    Sample a starting index $\hat{s}\in\{1,\dots,b-\tau+1\}$ with probability $p_i = \Prob{\hat{s}=i}$ and set $\hat{S}=\{\hat{s},\dots,\hat{s}+\tau-1\}$.
\end{mdframed}

For any fixed $i\in[b]$, we have
\begin{eqnarray*}
	\Prob{i \in \hat{S}}
	&=& \sum_{j=\max\{1,i-\tau+1\}}^{\min\{i,b-\tau+1\}} p_j
\end{eqnarray*}
and
\begin{eqnarray*}
	\Exp{L^1_{i,\hat{S}} \I{i \in \hat{S}}}
	&=& \sum_{j=\max\{1,i-\tau+1\}}^{\min\{i,b-\tau+1\}} p_j L^1_{i,\{j,\dots,j+\tau-1\}}.
\end{eqnarray*}
We have also already shown that
\begin{eqnarray*}
    \Exp{\mathrm{cost}(\hat{S})}
    = c_{\mathrm{ov}} + \sum_{j=1}^b c_j \parens{\sum_{i=1}^{\min\{j,b-\tau+1\}} p_i}
    + \sum_{j=1}^b c_j^{\sharp} \parens{\sum_{i=\max\{1,j-\tau+1\}}^{\min\{j,b-\tau+1\}} p_i}.
\end{eqnarray*}

\paragraph{Partitioned $\tau$-submodel sampling.}

Mimicking \Cref{sec:part_tau_submodel}, let us consider a special case of the above sampling scheme where the submodels partition $[b]$. For simplicity, we assume that $b$ is divisible by $\tau$ and let $m=\nicefrac{b}{\tau}$. The algorithm is then equivalent to $\tau$-submodel sampling with starting-index distribution satisfying
\begin{eqnarray*}
	p_{s_k} > 0 \quad (k=1,\dots,m), \qquad p_j=0 \text{ otherwise},
\end{eqnarray*}
where $s_k=(k-1)\tau+1$, $k\in[m]$. For any fixed $i\in[b]$, the derivations above simplify to
\begin{eqnarray*}
	\Prob{i \in \hat{S}}
	= \sum_{j=\max\{1,i-\tau+1\}}^{\min\{i,b-\tau+1\}} p_j
	= p_{s_{\lceil i/\tau \rceil}}
\end{eqnarray*}
and
\begin{eqnarray*}
	\Exp{L^{\alpha}_{i,\hat{S}} \I{i \in \hat{S}}}
	= \sum_{j=\max\{1,i-\tau+1\}}^{\min\{i,b-\tau+1\}} p_j L^{\alpha}_{i,\{j,\dots,j+\tau-1\}}
	= p_{s_{\lceil i/\tau \rceil}} L^{\alpha}_{i,B_{\ceil{\nicefrac{i}{\tau}}}}
\end{eqnarray*}
for $\alpha\in\{0,1\}$.
Plugging this choice into \eqref{eq:cost_opt_prob_e2} and \eqref{eq:cost_opt_prob_e} gives
\begin{eqnarray*}
    \Exp{\mathrm{cost}_{\varepsilon^2}(K)}
	&\propto& \frac{\sum_{i=1}^b \frac{\Prob{i\in \hat{S}}^2 \Exp{L^0_{i,\hat{S}} \I{i\in \hat{S}}}}{\parens{\Exp{L^1_{i,\hat{S}} \I{i\in \hat{S}}}}^2}}{\min_{i\in[b]} \parens{\frac{\Prob{i\in \hat{S}}^2}{\Exp{L^1_{i,\hat{S}} \I{i\in \hat{S}}}}}^2} \\
	&&\qquad\times\parens{c_{\mathrm{ov}} + \sum_{j=1}^b c_j \parens{\sum_{i=1}^{\min\{j,b-\tau+1\}} p_i} + \sum_{j=1}^b c_j^{\sharp} \parens{\sum_{i=\max\{1,j-\tau+1\}}^{\min\{j,b-\tau+1\}} p_i}} \\
	&=& \frac{\sum_{i=1}^b \frac{p_{s_{\lceil i/\tau \rceil}} L^0_{i,B_{\ceil{\nicefrac{i}{\tau}}}}}{\parens{L^1_{i,B_{\ceil{\nicefrac{i}{\tau}}}}}^2}}{\min_{i\in[b]} \parens{\frac{p_{s_{\lceil i/\tau \rceil}}}{L^1_{i,B_{\ceil{\nicefrac{i}{\tau}}}}}}^2} 
    \parens{c_{\mathrm{ov}} + \sum_{j=1}^b c_j \parens{\sum_{k=1}^{\ceil{\nicefrac{j}{\tau}}} p_{s_k}} + \sum_{j=1}^b c_j^{\sharp} p_{s_{\ceil{\nicefrac{j}{\tau}}}}}
\end{eqnarray*}
and
\begin{eqnarray*}
	\Exp{\mathrm{cost}_{\varepsilon}(K)}
	&\propto& \frac{1}{\min_{i\in[b]} \frac{\Prob{i\in \hat{S}}^2}{\Exp{L^1_{i,\hat{S}} \I{i\in \hat{S}}}}} \\
	&&\qquad\times\parens{c_{\mathrm{ov}} + \sum_{j=1}^b c_j \parens{\sum_{i=1}^{\min\{j,b-\tau+1\}} p_i} + \sum_{j=1}^b c_j^{\sharp} \parens{\sum_{i=\max\{1,j-\tau+1\}}^{\min\{j,b-\tau+1\}} p_i}} \\
	&=& \frac{1}{\min_{i\in[b]} \frac{p_{s_{\lceil i/\tau \rceil}}}{L^1_{i,B_{\ceil{\nicefrac{i}{\tau}}}}}}
	\parens{c_{\mathrm{ov}} + \sum_{j=1}^b c_j \parens{\sum_{k=1}^{\ceil{\nicefrac{j}{\tau}}} p_{s_k}} + \sum_{j=1}^b c_j^{\sharp} p_{s_{\ceil{\nicefrac{j}{\tau}}}}}.
\end{eqnarray*}
The expression for $\Exp{\mathrm{cost}_{\varepsilon}(K)}$ is entirely analogous to 
\eqref{eq:part_tau_submodel_cost} derived in the smooth case, with $L^0_{i,B_{\ceil{\nicefrac{i}{\tau}}}}$ replaced by $L^1_{i,B_{\ceil{\nicefrac{i}{\tau}}}}$. Consequently, for a fixed $\tau$, the optimal block probabilities are
\begin{eqnarray*}
    p_{s_k}^\star
    = \frac{\max_{i\in B_k} L^1_{i,B_k}}{\sum_{l=1}^m \max_{i\in B_l} L^1_{i,B_l}},
\end{eqnarray*}
Again, the optimal number of blocks $m$ depends on the growth rate of the constants $L^1_{i,\tau}$ with respect to $\tau$ and on their interaction with the costs $c_j$ and $c_j^{\sharp}$. In general, the best choice of $m$ is not necessarily $b$.

The dependence of $\Exp{\mathrm{cost}_{\varepsilon^2}(K)}$ on $\tau$ is significantly more complex and governed by the relative scaling between $L^0_{i,\tau}$ and $L^1_{i,\tau}$.

\subsubsection{Arbitrary Submodel Sampling}

\begin{mdframed}[linecolor=lightgray!12, backgroundcolor=lightgray!12]
    Let $\{B_1,\dots,B_m\}$ be a partition of $[b]$. Set $\hat{S} = B_i$ with probability $p_i$ (where $\sum_{i=1}^m p_i = 1$).
    For $i\in[b]$, $k(i)$ denotes the unique block with $i\in B_{k(i)}$ and $\underline{b}_k \eqdef \min B_k$.
\end{mdframed}

For any $i\in[b]$, we have
\begin{eqnarray*}
    \Prob{i \in \hat{S}} = p_{k(i)}
\end{eqnarray*}
and
\begin{eqnarray*}
	\Exp{L^{\alpha}_{i,\hat{S}} \I{i\in \hat{S}}}
	= p_{k(i)} L^{\alpha}_{i,B_{k(i)}}
\end{eqnarray*}
for $\alpha\in\{0,1\}$. Since
\begin{align*}
	\Exp{\mathrm{cost}(\hat{S})}
    = c_{\mathrm{ov}} + \sum_{j=1}^b c_j \parens{\sum_{k:\underline{b}_k \leq j} p_k} + \sum_{j=1}^b c_j^{\sharp} p_{k(j)},
\end{align*}
the total expected costs become
\begin{eqnarray*}
    \Exp{\mathrm{cost}_{\varepsilon^2}(K)}
	&\propto& \frac{\sum_{i=1}^b \frac{\Prob{i\in \hat{S}}^2 \Exp{L^0_{i,\hat{S}} \I{i\in \hat{S}}}}{\parens{\Exp{L^1_{i,\hat{S}} \I{i\in \hat{S}}}}^2}}{\min_{i\in[b]} \parens{\frac{\Prob{i\in \hat{S}}^2}{\Exp{L^1_{i,\hat{S}} \I{i\in \hat{S}}}}}^2}
	\parens{c_{\mathrm{ov}} + \sum_{j=1}^b c_j \parens{\sum_{k:\underline{b}_k \leq j} p_k} + \sum_{j=1}^b c_j^{\sharp} p_{k(j)}} \\
	&=& \frac{\sum_{i=1}^b \frac{p_{k(i)} L^0_{i,B_{k(i)}}}{\parens{L^1_{i,B_{k(i)}}}^2}}{\min_{i\in[b]} \parens{\frac{p_{k(i)}}{L^1_{i,B_{k(i)}}}}^2}
	\parens{c_{\mathrm{ov}} + \sum_{j=1}^b c_j \parens{\sum_{k:\underline{b}_k \leq j} p_k} + \sum_{j=1}^b c_j^{\sharp} p_{k(j)}}
\end{eqnarray*}
and
\begin{eqnarray*}
	\Exp{\mathrm{cost}_{\varepsilon}(K)}
	&\propto& \frac{1}{\min_{i\in[b]} \frac{\Prob{i\in \hat{S}}^2}{\Exp{L^1_{i,\hat{S}} \I{i\in \hat{S}}}}}
	\parens{c_{\mathrm{ov}} + \sum_{j=1}^b c_j \parens{\sum_{k:\underline{b}_k \leq j} p_k} + \sum_{j=1}^b c_j^{\sharp} p_{k(j)}} \\
	&=& \frac{1}{\min_{i\in[b]} \frac{p_{k(i)}}{L^1_{i,B_{k(i)}}}}
	\parens{c_{\mathrm{ov}} + \sum_{j=1}^b c_j \parens{\sum_{k:\underline{b}_k \leq j} p_k} + \sum_{j=1}^b c_j^{\sharp} p_{k(j)}}.
\end{eqnarray*}
Following the reasoning in \Cref{sec:arbitrary_submodel}, we find that for a fixed partition, the choice of probabilities minimizing $\Exp{\mathrm{cost}_{\varepsilon}(K)}$ is
\begin{align*}
	p_k^\star = \frac{\nicefrac{1}{\delta_k^{\min}}}{\sum_{l=1}^m \nicefrac{1}{\delta_l^{\min}}} 
    = \frac{\max_{i\in B_k} L^0_{i,B_k}}{\sum_{l=1}^m \max_{i\in B_l} L^0_{i,B_l}},
\end{align*}
and that submodel training can improve upon full model training in certain regimes.

\newpage

\section{Convergence Results -- Stochastic Gradient Case}

The proof of \Cref{thm:stoch_rt_l0l1_iter} relies on two preliminary lemmas, which we state and prove first.

\begin{lemma}[Descent Lemma]\label{lemma:descent2}
    Let \Cref{as:arbitrary_layer_gen_smoothness2} hold and consider the update rule $X_i^{k+1} = \lmo{\cB(X_i^k,t_i^k)}{M_i^k}$, $i=1,\ldots,b$, where $X^{k+1} = [X_1^{k+1}, \ldots, X_b^{k+1}] \in \cX$, $X^k = [X_1^k, \ldots, X_b^k] \in \cX$, $M^k = [M_1^k, \ldots, M_b^k] \in \cX$ and $t_i^k > 0$. Then
    \begin{eqnarray*}
        f(X^{k+1}) &\leq& f(X^k) + \sum_{i\in S^k} 2 t_i^k \norm{\nabla_i f(X^k) - M_i^k}_{(i) \star} - \sum_{i\in S^k} t_i^k \norm{\nabla_i f(X^k)}_{(i) \star} \\
        &&+ \sum_{i\in S^k} \frac{L^0_{i,S^k} + L^1_{i,S^k} \norm{\nabla_i f(X^k)}_{(i) \star}}{2} (t_i^k)^2.
    \end{eqnarray*}
\end{lemma}
\begin{proof}
    By \Cref{lemma:arb_layer_gen_smooth}
    \begin{eqnarray*}
        &&\hspace{-7mm}f(X^{k+1}) \\
        &\leq& f(X^k) + \inp{\nabla f(X^k)}{X^{k+1} - X^k} + \sum_{i\in S^k} \frac{L^0_{i,S^k} + L^1_{i,S^k} \norm{\nabla_i f(X^k)}_{(i) \star}}{2} \norm{X_i^k - X_i^{k+1}}_{(i)}^2 \\
        &\overset{\eqref{eq:normlmo}}{=}& f(X^k) + \sum_{i\in S^k} \parens{\inp{\nabla_i f(X^k) - M_i^k}{X_i^{k+1} - X_i^k}_{(i)} + \inp{M_i^k}{X_i^{k+1} - X_i^k}_{(i)}} \\
        &&+ \sum_{i\in S^k} \frac{L^0_{i,S^k} + L^1_{i,S^k} \norm{\nabla_i f(X^k)}_{(i) \star}}{2} (t_i^k)^2 \\
        &\overset{\eqref{eq:inplmo}}{=}& f(X^k) + \sum_{i\in S^k} \parens{\inp{\nabla_i f(X^k) - M_i^k}{X_i^{k+1} - X_i^k}_{(i)} - t_i^k \norm{M_i^k}_{(i) \star}} \\
        &&+ \sum_{i\in S^k} \frac{L^0_{i,S^k} + L^1_{i,S^k} \norm{\nabla_i f(X^k)}_{(i) \star}}{2} (t_i^k)^2 \\
        &\leq& f(X^k) + \sum_{i\in S^k} \Bigg(t_i^k \norm{\nabla_i f(X^k) - M_i^k}_{(i) \star} - t_i^k \norm{M_i^k}_{(i) \star} \\
        &&\qquad\qquad\qquad\quad+ \frac{L^0_{i,S^k} + L^1_{i,S^k} \norm{\nabla_i f(X^k)}_{(i) \star}}{2} (t_i^k)^2\Bigg),
    \end{eqnarray*}
    where in the last line we used the Cauchy-Schwarz inequality.
    Therefore, using triangle inequality, we get
    \begin{eqnarray*}
        &&\hspace{-7mm}f(X^{k+1}) \\
        &\leq& f(X^k) + \sum_{i\in S^k} \parens{t_i^k \norm{\nabla_i f(X^k) - M_i^k}_{(i) \star} + t_i^k \norm{\nabla_i f(X^k) - M_i^k}_{(i) \star} - t_i^k \norm{\nabla_i f(X^k)}_{(i) \star}} \\
        &&+ \sum_{i\in S^k} \frac{L^0_{i,S^k} + L^1_{i,S^k} \norm{\nabla_i f(X^k)}_{(i) \star}}{2} (t_i^k)^2 \\
        &=& f(X^k) + \sum_{i\in S^k} \parens{2 t_i^k \norm{\nabla_i f(X^k) - M_i^k}_{(i) \star} - t_i^k \norm{\nabla_i f(X^k)}_{(i) \star}} \\
        &&+ \sum_{i\in S^k} \frac{L^0_{i,S^k} + L^1_{i,S^k} \norm{\nabla_i f(X^k)}_{(i) \star}}{2} (t_i^k)^2
    \end{eqnarray*}
    as required.
\end{proof}

\begin{lemma}\label{lemma:nm_rec_m_l0l12}
    Let Assumptions \ref{as:arbitrary_layer_gen_smoothness2} and \ref{as:bounded_var} hold. Then, the iterates of \Cref{alg:rt_arbitrary_stoch} run with $t_i^k \equiv t_i$ satisfy
    \begin{eqnarray*}
        &&\hspace{-8mm}\Exp{\norm{M_i^{k+1} - \nabla_i f(X^{k+1})}_2} \\
        &\leq& \parens{1 - \Exp{\I{i\in \hat{S}}} \beta_i}^{k+1} \Exp{\norm{M_i^0 - \nabla_i f(X^0)}_2}
        + \frac{t_i \Exp{\parens{1-\I{i\in \hat{S}} \beta_i} L_{i,\hat{S}}^0}}{\underline{\rho}_i \beta_i \Exp{\I{i\in \hat{S}}}} \nonumber \\
        &&+ \frac{t_i}{\underline{\rho}_i} \Exp{\parens{1-\I{i\in \hat{S}} \beta_i} L_{i,\hat{S}}^1} \sum_{l=0}^{k} \parens{1 - \Exp{\I{i\in \hat{S}}} \beta_i}^{k-l} \Exp{\norm{\nabla_i f(X^l)}_{(i) \star}} \\
        &&+ \sigma_i \sqrt{\beta_i}.
    \end{eqnarray*}
\end{lemma}

\begin{proof}
    The proof is inspired by \citet[Theorem 1]{cutkosky2020momentum}.
    First, using the momentum update rule, we have
    \begin{eqnarray*}
        M_i^{k+1} &=& \parens{1 - \I{i\in S^k} \beta_i} M_i^k + \I{i\in S^k} \beta_i \nabla_i f(X^{k+1}; \xi^{k+1}) \\
        &=& \parens{1 - \beta_i^k} \parens{M_i^k - \nabla_i f(X^k)} + \parens{1 - \beta_i^k} \parens{\nabla_i f(X^k) - \nabla_i f(X^{k+1})} \\
        &&+ \beta_i^k \parens{\nabla_i f(X^{k+1}; \xi^{k+1}) - \nabla_i f(X^{k+1})} + \nabla_i f(X^{k+1}).
    \end{eqnarray*}
    where $\beta_i^k \eqdef \I{i\in S^k} \beta_i$.
    To simplify the notation, define $U_{1,i}^k \eqdef M_i^k - \nabla_i f(X^k)$, $U_{2,i}^k \eqdef \nabla_i f(X^k) - \nabla_i f(X^{k+1})$ and $U_{3,i}^k \eqdef \nabla_i f(X^k; \xi^k) - \nabla_i f(X^k)$. Then the above can be written as
    \begin{eqnarray}\label{eq:rec_st}
        U_{1,i}^{k+1} = \parens{1 - \beta_i^k} U_{1,i}^k + \parens{1 - \beta_i^k} U_{2,i}^k + \beta_i^k U_{3,i}^{k+1}.
    \end{eqnarray}
    Unrolling the recursion in \eqref{eq:rec_st}, we get
    \begin{align*}
        U_{1,i}^{k+1} &= (1-\beta_i^k) U_{1,i}^k + (1-\beta_i^k) U_{2,i}^k + \beta_i^k U_{3,i}^{k+1} \\
        &= \parens{\prod_{m=0}^{k}(1-\beta_i^m)} U_{1,i}^0
        + \sum_{l=0}^{k} \parens{\prod_{m=l}^{k}(1-\beta_i^m)} U_{2,i}^l
        + \sum_{l=0}^{k} \parens{\beta_i^l\prod_{m=l+1}^{k}(1-\beta_i^m)} U_{3,i}^{l+1},
    \end{align*}
    where by convention $\prod_{m=a}^{b}(\cdot)=1$ if $a>b$. Hence, taking norms and using the triangle inequality,
    \begin{eqnarray}\label{eq:apieghrn}
        \Exp{\norm{U_{1,i}^{k+1}}_2}
        &\leq& \Exp{\norm{\parens{\prod_{m=0}^{k}(1-\beta_i^m)} U_{1,i}^0}_2}
        + \Exp{\norm{\sum_{l=0}^{k} \parens{\prod_{m=l}^{k}(1-\beta_i^m)} U_{2,i}^l}_2} \nonumber \\
        &&+ \Exp{\norm{\sum_{l=0}^{k} \parens{\beta_i^l\prod_{m=l+1}^{k}(1-\beta_i^m)} U_{3,i}^{l+1}}_2}.
    \end{eqnarray}
    Let us consider each of the terms separately. First, using the independence of $\{S^k\}_{k\geq0}$, we have
    \begin{eqnarray}\label{eq:niuytfvrths}
        \Exp{\norm{\parens{\prod_{m=0}^{k}(1-\beta_i^m)} U_{1,i}^0}_2}
        &\leq& \Exp{\parens{\prod_{m=0}^{k}(1-\beta_i^m)} \norm{U_{1,i}^0}_2} \nonumber \\
        &=& \Exp{\ExpCond{\parens{\prod_{m=0}^{k}(1-\beta_i^m)}}{X^k, M_i^k} \norm{U_{1,i}^0}_2} \nonumber \\
        &=& \prod_{m=0}^{k} \Exp{1-\beta_i^m} \Exp{\norm{U_{1,i}^0}_2}.
    \end{eqnarray}
    Next, by \Cref{as:arbitrary_layer_gen_smoothness2}
    \begin{eqnarray*}
        &&\hspace{-8mm}\Exp{\norm{\sum_{l=0}^{k} \parens{\prod_{m=l}^{k}(1-\beta_i^m)} U_{2,i}^l}_2} \\
        &\leq& \Exp{\sum_{l=0}^{k} \parens{\prod_{m=l}^{k}(1-\beta_i^m)} \norm{U_{2,i}^l}_2} \\
        &\leq& \frac{1}{\underline{\rho}_i} \sum_{l=0}^{k} \Exp{\parens{\prod_{m=l}^{k}(1-\beta_i^m)} \norm{\nabla_i f(X^l) - \nabla_i f(X^{l+1})}_{(i) \star}} \\
        &\overset{\eqref{as:arbitrary_layer_gen_smoothness}}{\leq}& \frac{1}{\underline{\rho}_i} \sum_{l=0}^{k} \Exp{\parens{\prod_{m=l}^{k}(1-\beta_i^m)} \parens{L_{i,S^l}^0 + L_{i,S^l}^1 \norm{\nabla_i f(X^l)}_{(i) \star}} \norm{X_i^l - X_i^{l+1}}_{(i)}}.
    \end{eqnarray*}
    Since the product $\prod_{m=l+1}^{k}(1-\beta_i^m)$ depends only on samplings at iterations $>l$, these factors are independent of the $\sigma$-algebra generated by $(X^l,S^l)$. Therefore,
    \begin{eqnarray*}
        &&\hspace{-8mm}\Exp{\norm{\sum_{l=0}^{k} \parens{\prod_{m=l}^{k}(1-\beta_i^m)} U_{2,i}^l}_2} \\
        &\leq& \frac{1}{\underline{\rho}_i} \sum_{l=0}^{k} \parens{\parens{\prod_{m=l+1}^{k} \Exp{1-\beta_i^m}} \Exp{(1-\beta_i^l) \parens{L_{i,S^l}^0 + L_{i,S^l}^1 \norm{\nabla_i f(X^l)}_{(i) \star}}} t_i},
    \end{eqnarray*}
    where
    \begin{eqnarray*}
        &&\hspace{-8mm}\Exp{(1-\beta_i^l) \parens{L_{i,S^l}^0 + L_{i,S^l}^1 \norm{\nabla_i f(X^l)}_{(i) \star}}} \\
        &=& \Exp{(1-\beta_i^l) L_{i,S^l}^0} + \Exp{\ExpCond{(1-\beta_i^l) L_{i,S^l}^1 \norm{\nabla_i f(X^l)}_{(i) \star}}{X^l}} \\
        &=& \Exp{(1-\beta_i^l) L_{i,S^l}^0} + \Exp{(1-\beta_i^l) L_{i,S^l}^1} \Exp{\norm{\nabla_i f(X^l)}_{(i) \star}}.
    \end{eqnarray*}
    Thus
    \begin{eqnarray}\label{eq:dukise}
        &&\hspace{-8mm}\Exp{\norm{\sum_{l=0}^{k} \parens{\prod_{m=l}^{k}(1-\beta_i^m)} U_{2,i}^l}_2} \\
        &\leq& \frac{t_i}{\underline{\rho}_i} \sum_{l=0}^{k} \parens{\prod_{m=l+1}^{k} \Exp{1-\beta_i^m}} \parens{\Exp{(1-\beta_i^l) L_{i,S^l}^0} + \Exp{(1-\beta_i^l) L_{i,S^l}^1} \Exp{\norm{\nabla_i f(X^l)}_{(i) \star}}}. \nonumber
    \end{eqnarray}
    Lastly, by Jensen's inequality, the last term can be bounded via
    \begin{eqnarray}\label{eq:aefawcrgt}
        &&\hspace{-8mm}\Exp{\norm{\sum_{l=0}^{k} \parens{\beta_i^l\prod_{m=l+1}^{k}(1-\beta_i^m)} U_{3,i}^{l+1}}_2} \nonumber \\
        &\leq& \sqrt{\Exp{\norm{\sum_{l=0}^{k} \underbrace{\parens{\beta_i^l\prod_{m=l+1}^{k}(1-\beta_i^m)}}_{\eqdef a_l} U_{3,i}^{l+1}}^2_2}}
        = \sqrt{\Exp{\sum_{l, r = 0}^{k} a_l a_r \inp{U_{3,i}^{l+1}}{U_{3,i}^{r+1}}}} \nonumber \\
        &\overset{\eqref{as:bounded_var}}{=}& \sqrt{\Exp{\sum_{l = 0}^{k} a_l^2 \norm{U_{3,i}^{l+1}}_2^2}}
        = \sqrt{\sum_{l = 0}^{k} \Exp{\ExpCond{a_l^2 \norm{U_{3,i}^{l+1}}_2^2}{\{S^r\}_{r=l}^{k}, X^{l+1}}}} \nonumber \\
        &=& \sqrt{\sum_{l = 0}^{k} \Exp{a_l^2 \ExpCond{\norm{U_{3,i}^{l+1}}_2^2}{\{S^r\}_{r=l}^{k}, X^{l+1}}}}
        \overset{\eqref{as:bounded_var}}{\leq} \sigma_i \sqrt{\sum_{l = 0}^{k} \Exp{a_l^2}} \nonumber \\
        &=& \sigma_i \sqrt{\sum_{l = 0}^{k} \parens{\Exp{(\beta_i^l)^2} \prod_{m=l+1}^{k} \Exp{(1-\beta_i^m)^2}}} \nonumber \\
        &=& \sigma_i \sqrt{\sum_{l = 0}^{k} \parens{\Exp{\I{i\in S^l} \beta_i^2} \prod_{m=l+1}^{k} \Exp{\parens{1-\I{i\in S^m} \beta_i}^2}}} \nonumber \\
        &=& \sigma_i \sqrt{\sum_{l = 0}^{k} \parens{\beta_i^2 \Exp{\I{i\in \hat{S}}} \Exp{\parens{1-\I{i\in \hat{S}} \beta_i}^2}^{k-l}}} \nonumber \\
        &\leq& \sigma_i \sqrt{\frac{\beta_i^2 \Exp{\I{i\in \hat{S}}}}{1 - \Exp{\parens{1-\I{i\in \hat{S}} \beta_i}^2}}} \nonumber \\
        &=& \sigma_i \beta_i \sqrt{\frac{\Exp{\I{i\in \hat{S}}}}{\Exp{\parens{2 - \beta_i} \I{i\in \hat{S}} \beta_i}}}
        = \frac{\sigma_i \beta_i}{\sqrt{\parens{2 - \beta_i} \beta_i}}
        \leq \sigma_i \sqrt{\beta_i}.
    \end{eqnarray}
    Applying \eqref{eq:niuytfvrths}, \eqref{eq:dukise} and \eqref{eq:aefawcrgt} in \eqref{eq:apieghrn} and noting that $\Exp{\beta_i^k} = \Exp{\I{i\in \hat{S}}} \beta_i$ yields
    \begin{eqnarray*}
        &&\hspace{-6mm}\Exp{\norm{U_{1,i}^{k+1}}_2} \\
        &\leq& \prod_{m=0}^{k} \Exp{1-\beta_i^m} \Exp{\norm{U_{1,i}^0}_2}
        + \frac{t_i}{\underline{\rho}_i} \sum_{l=0}^{k} \parens{\prod_{m=l+1}^{k} \Exp{1-\beta_i^m}} \Exp{(1-\beta_i^l) L_{i,S^l}^0} \nonumber \\
        &&+ \frac{t_i}{\underline{\rho}_i} \sum_{l=0}^{k} \parens{\prod_{m=l+1}^{k} \Exp{1-\beta_i^m}} \Exp{(1-\beta_i^l) L_{i,S^l}^1} \Exp{\norm{\nabla_i f(X^l)}_{(i) \star}} 
        + \sigma_i \sqrt{\beta_i} \\
        &=& \parens{1 - \Exp{\I{i\in \hat{S}}} \beta_i}^{k+1} \Exp{\norm{U_{1,i}^0}_2} \nonumber \\
        &&+ \frac{t_i}{\underline{\rho}_i} \Exp{\parens{1-\I{i\in \hat{S}} \beta_i} L_{i,\hat{S}}^0} \sum_{l=0}^{k} \parens{1 - \Exp{\I{i\in \hat{S}}} \beta_i}^{k-l} \nonumber \\
        &&+ \frac{t_i}{\underline{\rho}_i} \Exp{\parens{1-\I{i\in \hat{S}} \beta_i} L_{i,\hat{S}}^1} \sum_{l=0}^{k} \parens{1 - \Exp{\I{i\in \hat{S}}} \beta_i}^{k-l} \Exp{\norm{\nabla_i f(X^l)}_{(i) \star}} \nonumber \\
        &&+ \sigma_i \sqrt{\beta_i} \\
        &\leq& \parens{1 - \Exp{\I{i\in \hat{S}}} \beta_i}^{k+1} \Exp{\norm{U_{1,i}^0}_2}
        + \frac{t_i}{\underline{\rho}_i \beta_i \Exp{\I{i\in \hat{S}}}} \Exp{\parens{1-\I{i\in \hat{S}} \beta_i} L_{i,\hat{S}}^0} \nonumber \\
        &&+ \frac{t_i}{\underline{\rho}_i} \Exp{\parens{1-\I{i\in \hat{S}} \beta_i} L_{i,\hat{S}}^1} \sum_{l=0}^{k} \parens{1 - \Exp{\I{i\in \hat{S}}} \beta_i}^{k-l} \Exp{\norm{\nabla_i f(X^l)}_{(i) \star}} \nonumber \\
        &&+ \sigma_i \sqrt{\beta_i}.
    \end{eqnarray*}
\end{proof}

\begin{theorem}\label{thm:stoch_rt_l0l1_iter}
    Let Assumptions \ref{as:lower_bound}, \ref{as:arbitrary_layer_gen_smoothness2}, and \ref{as:bounded_var} hold. Let $\{X^k\}_{k=0}^{K-1}$, $K \geq 1$, be the iterates of \Cref{alg:rt_arbitrary_stoch} initialized with $M_i^0 = \nabla_i f(X^0; \xi^0)$ and run with $\beta_i \equiv \beta = \nicefrac{1}{(K+1)^{1/2}}$ and
    \begin{align*}
        0 < t_i^k \equiv t_i = \frac{\eta_i}{(K+1)^{3/4}}, \qquad i=1,\ldots,b,
    \end{align*}
    where $\eta_i^2 \leq \min\brac{\frac{(K+1)^{1/2}}{4 \Exp{\I{i\in \hat{S}} L_{i,\hat{S}}^1} \Exp{\max\limits_{i\in [b]} L_{i,\hat{S}}^1}}, \frac{\underline{\rho}_i \min\limits_{i\in [b]} \parens{\Exp{\I{i\in \hat{S}}}}}{16 \bar{\rho}_i \Exp{\I{i\in \hat{S}}} \Exp{\parens{1-\I{i\in \hat{S}} \beta_i} L_{i,\hat{S}}^1} \Exp{\max\limits_{i\in [b]} L_{i,\hat{S}}^1}}, 1}$.
    Then
    \begin{eqnarray*}
        &&\hspace{-8mm}\min_{k=0,\ldots,K} \sum_{i=1}^b \frac{\Exp{\I{i\in \hat{S}}} \eta_i}{\frac{1}{b} \sum_{l=1}^b \Exp{\I{l \in \hat{S}}} \eta_l} \Exp{\norm{\nabla_i f(X^k)}_{(i) \star}} \\
        &\leq& \frac{3 \delta^0}{(K+1)^{1/4} \frac{1}{b} \sum_{l=1}^b \Exp{\I{l \in \hat{S}}} \eta_l}
        + \frac{6}{(K+1)^{1/2}} \sum_{i=1}^b \frac{\eta_i \bar{\rho}_i \sigma_i}{\frac{1}{b} \sum_{l=1}^b \Exp{\I{l \in \hat{S}}} \eta_l} \\
        &&+ \sum_{i=1}^b \frac{\eta_i^2}{(K+1)^{1/4} \frac{1}{b} \sum_{l=1}^b \Exp{\I{l \in \hat{S}}} \eta_l} \frac{2 \bar{\rho}_i}{\underline{\rho}_i} \parens{\Exp{L_{i,\hat{S}}^0} + \Exp{L_{i,\hat{S}}^1} \Exp{\frac{L_{i,\hat{S}}^0}{L_{i,\hat{S}}^1}}} \\
        &&+ \sum_{i=1}^b \frac{\eta_i^2}{2 (K+1)^{3/4} \frac{1}{b} \sum_{l=1}^b \Exp{\I{l \in \hat{S}}} \eta_l} \\
        &&\qquad\quad\times \parens{\Exp{\I{i\in \hat{S}} L^0_{i,\hat{S}}} + \Exp{\I{i\in \hat{S}} L_{i,\hat{S}}^1} \Exp{\frac{L_{i,\hat{S}}^0}{L_{i,\hat{S}}^1}}} \\
        &&+ \sum_{i=1}^b \frac{2 \Exp{\I{i\in \hat{S}}} \eta_i}{(K+1)^{1/4} \frac{1}{b} \sum_{l=1}^b \Exp{\I{l \in \hat{S}}} \eta_l} \bar{\rho}_i \sigma_i.
    \end{eqnarray*}
\end{theorem}

\begin{remark}
    For \Cref{thm:stoch_rt_l0l1_iter} to be meaningful, the sampling must be such that $\Exp{\I{i\in \hat{S}}} \eta_i > 0$ for every $i\in[b]$.
    Thus, every layer has to be sampled with positive probability.
\end{remark}

\begin{remark}\label{rem:stoch_rpt_l0l1_iter}
    Note that under \Cref{as:arbitrary_layer_gen_smoothness2}, without loss of generality we can set $L^0_{i,S} = L^1_{i,S} = 0$ whenever $i\not\in S$. Hence, in the case of \algnamesmall{RPT} (see \Cref{sec:rpt}), we have
    \begin{eqnarray*}
        \Exp{\I{i\in \hat{S}}}
        &=& \sum_{s=1}^i p_s, \\
        \Exp{\max_{i\in[b]} L^1_{i,\hat{S}}}
        &=& \sum_{s=1}^b p_s \max_{i\in[b]} L^1_{i,\{s,\dots,b\}}, \\
        \Exp{\frac{L^0_{i,\hat{S}}}{L^1_{i,\hat{S}}}}
        &=& \sum_{s=1}^b p_s \frac{L^0_{i,\{s,\dots,b\}}}{L^1_{i,\{s,\dots,b\}}}
        = \sum_{s=1}^i p_s \frac{L^0_{i,\{s,\dots,b\}}}{L^1_{i,\{s,\dots,b\}}}, \\
        \Exp{\parens{1-\I{i\in\hat{S}} \beta_i} L^1_{i,\hat{S}}}
        &=& \sum_{s=1}^b p_s (1-\I{i\in\{s,\dots,b\}} \beta_i) L^1_{i,\{s,\dots,b\}} \\
        &=& (1-\beta_i)\sum_{s=1}^i p_s L^1_{i,\{s,\dots,b\}}, \\
        \Exp{L^{\alpha}_{i,\hat{S}}}
        &=& \sum_{s=1}^b p_s L^{\alpha}_{i,\{s,\dots,b\}}
        = \sum_{s=1}^i p_s L^{\alpha}_{i,\{s,\dots,b\}}, \\
    	\Exp{L^{\alpha}_{i,\hat{S}} \I{i\in \hat{S}}}
    	&=& \sum_{s=1}^i p_s L^{\alpha}_{i, \{s,\dots,b\}}
    \end{eqnarray*}
    for $\alpha\in\{0,1\}$.
    Substituting these expressions into the rate yields the result in \Cref{thm:stoch_rpt_l0l1_iter}.
\end{remark}

\begin{proof}[Proof of \Cref{thm:stoch_rt_l0l1_iter}]
    Taking expectation conditional on $[X^k, \{M_i^k\}_{i\in[b]}]$ in \Cref{lemma:descent2} gives
    \begin{align*}
        &\ExpCond{f(X^{k+1})}{X^k, \{M_i^k\}_{i\in[b]}} \\
        &\leq f(X^k) + \ExpCond{\sum_{i\in S^k} 2 t_i \norm{\nabla_i f(X^k) - M_i^k}_{(i) \star} - \sum_{i\in S^k} t_i \norm{\nabla_i f(X^k)}_{(i) \star}}{X^k, \{M_i^k\}_{i\in[b]}} \\
        &\quad+ \ExpCond{\sum_{i\in S^k} \frac{L^0_{i,S} + L_{i,S}^1 \norm{\nabla_i f(X^k)}_{(i) \star}}{2} t_i^2}{X^k, \{M_i^k\}_{i\in[b]}} \\
        &= f(X^k) + \sum_{i=1}^b \ExpCond{\I{i\in S^k} \parens{2 t_i \norm{\nabla_i f(X^k) - M_i^k}_{(i) \star} - t_i \norm{\nabla_i f(X^k)}_{(i) \star}}}{X^k, \{M_i^k\}_{i\in[b]}} \\
        &\quad+ \sum_{i=1}^b \ExpCond{\I{i\in S^k} \frac{L^0_{i,S} + L_{i,S}^1 \norm{\nabla_i f(X^k)}_{(i) \star}}{2} t_i^2}{X^k, \{M_i^k\}_{i\in[b]}} \\
        &\leq f(X^k) + \sum_{i=1}^b \Exp{\I{i\in \hat{S}}} \parens{2 t_i \bar{\rho}_i \norm{\nabla_i f(X^k) - M_i^k}_2 - t_i \norm{\nabla_i f(X^k)}_{(i) \star}} \\
        &\quad+ \frac{1}{2} \sum_{i=1}^b \Exp{\I{i\in \hat{S}} L^0_{i,\hat{S}}} t_i^2
        + \frac{1}{2} \sum_{i=1}^b \Exp{\I{i\in \hat{S}} L_{i,\hat{S}}^1} t_i^2 \norm{\nabla_i f(X^k)}_{(i) \star}.
    \end{align*}
    Hence, taking full expectation
    \begin{eqnarray*}
        &&\hspace{-6mm}\Exp{f(X^{k+1})} \\
        &\leq& \Exp{f(X^k)} + \sum_{i=1}^b 2 t_i \bar{\rho}_i \Exp{\I{i\in \hat{S}}} \Exp{\norm{\nabla_i f(X^k) - M_i^k}_2} \\
        &&- \sum_{i=1}^b \Exp{\I{i\in \hat{S}}} t_i \Exp{\norm{\nabla_i f(X^k)}_{(i) \star}} \\
        &&+ \frac{1}{2} \sum_{i=1}^b \Exp{\I{i\in \hat{S}} L^0_{i,\hat{S}}} t_i^2
        + \frac{1}{2} \sum_{i=1}^b \Exp{\I{i\in \hat{S}} L_{i,\hat{S}}^1} t_i^2 \Exp{\norm{\nabla_i f(X^k)}_{(i) \star}}.
    \end{eqnarray*}
    To simplify the notation, let $\delta^k \eqdef \Exp{f(X^k) - f^{\star}}$ and $P_i^k \eqdef \Exp{\norm{\nabla_i f(X^k) - M_i^k}_2}$.
    Then, according to the descent inequality above and \Cref{lemma:nm_rec_m_l0l12}
    \begin{eqnarray}
        \delta^{k+1}
        &\leq& \delta^k - \sum_{i=1}^b t_i \Exp{\I{i\in \hat{S}}} \Exp{\norm{\nabla_i f(X^k)}_{(i) \star}} + \sum_{i=1}^b 2 t_i \bar{\rho}_i \Exp{\I{i\in \hat{S}}} P_i^k \nonumber \\
        &&+ \frac{1}{2} \sum_{i=1}^b t_i^2 \Exp{\I{i\in \hat{S}} L^0_{i,\hat{S}}}
        + \frac{1}{2} \sum_{i=1}^b t_i^2 \Exp{\I{i\in \hat{S}} L_{i,\hat{S}}^1} \Exp{\norm{\nabla_i f(X^k)}_{(i) \star}}, \label{eq:i1} \\
        P_i^k
        &\leq& \parens{1 - \Exp{\I{i\in \hat{S}}} \beta_i}^k P_i^0
        + \frac{t_i \Exp{\parens{1-\I{i\in \hat{S}} \beta_i} L_{i,\hat{S}}^0}}{\underline{\rho}_i \beta_i \Exp{\I{i\in \hat{S}}}} \nonumber \\
        &&+ \frac{t_i}{\underline{\rho}_i} \Exp{\parens{1-\I{i\in \hat{S}} \beta_i} L_{i,\hat{S}}^1} \sum_{l=0}^{k-1} \parens{1 - \Exp{\I{i\in \hat{S}}} \beta_i}^{k-1-l} \Exp{\norm{\nabla_i f(X^l)}_{(i) \star}} \nonumber \\
        &&+ \sigma_i \sqrt{\beta_i}. \label{eq:i2}
    \end{eqnarray}
    Applying \eqref{eq:i2} in \eqref{eq:i1},
    \begin{eqnarray*}
        &&\hspace{-6mm}\delta^{k+1} \\
        &\leq& \delta^k - \sum_{i=1}^b t_i \Exp{\I{i\in \hat{S}}} \Exp{\norm{\nabla_i f(X^k)}_{(i) \star}} \\
        &&+ \sum_{i=1}^b 2 t_i \bar{\rho}_i \Exp{\I{i\in \hat{S}}} \parens{\parens{1 - \Exp{\I{i\in \hat{S}}} \beta_i}^k P_i^0
        + \frac{t_i \Exp{\parens{1-\I{i\in \hat{S}} \beta_i} L_{i,\hat{S}}^0}}{\underline{\rho}_i \beta_i \Exp{\I{i\in \hat{S}}}}} \\
        &&+ \sum_{i=1}^b \Bigg(\frac{2 t_i^2 \bar{\rho}_i}{\underline{\rho}_i} \Exp{\I{i\in \hat{S}}} \Exp{\parens{1-\I{i\in \hat{S}} \beta_i} L_{i,\hat{S}}^1} \\
        &&\qquad\qquad\times \sum_{l=0}^{k-1} \parens{1 - \Exp{\I{i\in \hat{S}}} \beta_i}^{k-1-l} \Exp{\norm{\nabla_i f(X^l)}_{(i) \star}}\Bigg) \\
        &&+ \sum_{i=1}^b 2 t_i \bar{\rho}_i \Exp{\I{i\in \hat{S}}} \sigma_i \sqrt{\beta_i} \\
        &&+ \frac{1}{2} \sum_{i=1}^b t_i^2\Exp{\I{i\in \hat{S}} L^0_{i,\hat{S}}}
        + \frac{1}{2} \sum_{i=1}^b t_i^2 \Exp{\I{i\in \hat{S}} L_{i,\hat{S}}^1} \Exp{\norm{\nabla_i f(X^k)}_{(i) \star}} \\
        &=& \delta^k - \sum_{i=1}^b t_i \Exp{\I{i\in \hat{S}}} \Exp{\norm{\nabla_i f(X^k)}_{(i) \star}} \\
        &&+ \sum_{i=1}^b 2 t_i \bar{\rho}_i \Exp{\I{i\in \hat{S}}} \parens{1 - \Exp{\I{i\in \hat{S}}} \beta_i}^k P_i^0 \\
        &&+ \sum_{i=1}^b \Bigg(\frac{2 t_i^2 \bar{\rho}_i}{\underline{\rho}_i} \Exp{\I{i\in \hat{S}}} \Exp{\parens{1-\I{i\in \hat{S}} \beta_i} L_{i,\hat{S}}^1} \\
        &&\qquad\qquad\times \sum_{l=0}^{k-1} \parens{1 - \Exp{\I{i\in \hat{S}}} \beta_i}^{k-1-l} \Exp{\norm{\nabla_i f(X^l)}_{(i) \star}}\Bigg) \\
        &&+ \frac{1}{2} \sum_{i=1}^b t_i^2 \Exp{\I{i\in \hat{S}} L_{i,\hat{S}}^1} \Exp{\norm{\nabla_i f(X^k)}_{(i) \star}} \\
        &&+ \frac{1}{2} \sum_{i=1}^b t_i^2\Exp{\I{i\in \hat{S}} L^0_{i,\hat{S}}}
        + \sum_{i=1}^b \frac{2 t_i^2 \bar{\rho}_i \Exp{\parens{1-\I{i\in \hat{S}} \beta_i} L_{i,\hat{S}}^0}}{\beta_i \underline{\rho}_i} \\
        &&+ \sum_{i=1}^b 2 t_i \bar{\rho}_i \Exp{\I{i\in \hat{S}}} \sigma_i \sqrt{\beta_i}.
    \end{eqnarray*}
    Let us now look at the terms involving the gradient norms. Using \Cref{lemma:layer_gen_smooth3}, we get
    \begin{eqnarray*}
        &&\hspace{-6mm}\sum_{i=1}^b \Bigg(t_i^2 \underbrace{\frac{2 \bar{\rho}_i}{\underline{\rho}_i} \Exp{\I{i\in \hat{S}}} \Exp{\parens{1-\I{i\in \hat{S}} \beta_i} L_{i,\hat{S}}^1}}_{\eqdef b_i}  \\
        &&\qquad\qquad\times \sum_{l=0}^{k-1} \parens{1 - \Exp{\I{i\in \hat{S}}} \beta_i}^{k-1-l} \Exp{\norm{\nabla_i f(X^l)}_{(i) \star}}\Bigg) \\
        &=& \sum_{l=0}^{k-1} \sum_{i=1}^b t_i^2 b_i \parens{1 - \Exp{\I{i\in \hat{S}}} \beta_i}^{k-1-l} \Exp{\norm{\nabla_i f(X^l)}_{(i) \star}} \\
        &\leq& \sum_{l=0}^{k-1} \Exp{4 \max_{i\in [b]} \parens{t_i^2 b_i \parens{1 - \Exp{\I{i\in \hat{S}}} \beta_i}^{k-1-l} L_{i,S^l}^1} \parens{f(X^l) - f^{\star}}} \\
        &&+ \Exp{\sum_{i=1}^b \frac{t_i^2 b_i \parens{1 - \Exp{\I{i\in \hat{S}}} \beta_i}^{k-1-l} L_{i,S^l}^0}{L_{i,S^l}^1}} \\
        &\leq& 4 \sum_{l=0}^{k-1} \Exp{\ExpCond{\max_{i\in [b]} \parens{t_i^2 b_i \parens{1 - \Exp{\I{i\in \hat{S}}} \beta_i}^{k-1-l}} \max_{i\in [b]} \parens{L_{i,S^l}^1} \parens{f(X^l) - f^{\star}}}{X^l}} \\
        &&+ \sum_{l=0}^{k-1} \sum_{i=1}^b t_i^2 b_i \parens{1 - \Exp{\I{i\in \hat{S}}} \beta_i}^{k-1-l} \Exp{\frac{L_{i,\hat{S}}^0}{L_{i,\hat{S}}^1}} \\
        &=& 4 \sum_{l=0}^{k-1} \max_{i\in [b]} \parens{t_i^2 b_i \parens{1 - \Exp{\I{i\in \hat{S}}} \beta_i}^{k-1-l}} \Exp{\max_{i\in [b]} L_{i,\hat{S}}^1} \delta^l \\
        &&+ \sum_{i=1}^b t_i^2 b_i \Exp{\frac{L_{i,\hat{S}}^0}{L_{i,\hat{S}}^1}} \sum_{l=0}^{k-1} \parens{1 - \Exp{\I{i\in \hat{S}}} \beta_i}^{k-1-l} \\
        &\leq& 4 \max_{i\in [b]} \parens{t_i^2 b_i} \Exp{\max_{i\in [b]} L_{i,\hat{S}}^1} \sum_{l=0}^{k-1} \max_{i\in [b]} \parens{\parens{1 - \Exp{\I{i\in \hat{S}}} \beta_i}^{k-1-l}} \delta^l \\
        &&+ \sum_{i=1}^b \frac{t_i^2 b_i}{\Exp{\I{i\in \hat{S}}} \beta_i} \Exp{\frac{L_{i,\hat{S}}^0}{L_{i,\hat{S}}^1}} \\
        &=& 4 \max_{i\in [b]} \parens{t_i^2 b_i} \Exp{\max_{i\in [b]} L_{i,\hat{S}}^1} \sum_{l=0}^{k-1} \parens{1 - \min_{i\in [b]} \parens{\Exp{\I{i\in \hat{S}}} \beta_i}}^{k-1-l} \delta^l \\
        &&+ \sum_{i=1}^b \frac{t_i^2 b_i}{\Exp{\I{i\in \hat{S}}} \beta_i} \Exp{\frac{L_{i,\hat{S}}^0}{L_{i,\hat{S}}^1}},
    \end{eqnarray*}
    and
    \begin{eqnarray*}
        &&\hspace{-6mm}\sum_{i=1}^b t_i^2 \Exp{\I{i\in \hat{S}} L_{i,\hat{S}}^1} \Exp{\norm{\nabla_i f(X^k)}_{(i) \star}} \\
        &\leq& \Exp{4 \max_{i\in [b]} \parens{t_i^2 \Exp{\I{i\in \hat{S}} L_{i,\hat{S}}^1} L_{i,S^k}^1} \parens{f(X^k) - f^{\star}} + \sum_{i=1}^b \frac{t_i^2 \Exp{\I{i\in \hat{S}} L_{i,\hat{S}}^1} L_{i,S^k}^0}{L_{i,S^k}^1}} \\
        &\leq& 4 \max_{i\in [b]} \parens{\Exp{\I{i\in \hat{S}} L_{i,\hat{S}}^1}} \Exp{\max_{i\in [b]} \parens{t_i^2 L_{i,S^k}^1} \parens{f(X^k) - f^{\star}}} \\
        &&+ \sum_{i=1}^b t_i^2 \Exp{\I{i\in \hat{S}} L_{i,\hat{S}}^1} \Exp{\frac{L_{i,S^k}^0}{L_{i,S^k}^1}} \\
        &=& 4 \max_{i\in [b]} \parens{\Exp{\I{i\in \hat{S}} L_{i,\hat{S}}^1}} \Exp{\ExpCond{\max_{i\in [b]} \parens{t_i^2 L_{i,S^k}^1} \parens{f(X^k) - f^{\star}}}{X^k}} \\
        &&+ \sum_{i=1}^b t_i^2 \Exp{\I{i\in \hat{S}} L_{i,\hat{S}}^1} \Exp{\frac{L_{i,\hat{S}}^0}{L_{i,\hat{S}}^1}} \\
        &=& 4 \max_{i\in [b]} \parens{\Exp{\I{i\in \hat{S}} L_{i,\hat{S}}^1}} \Exp{\max_{i\in [b]} \parens{t_i^2 L_{i,\hat{S}}^1}} \delta^k \\
        &&+ \sum_{i=1}^b t_i^2 \Exp{\I{i\in \hat{S}} L_{i,\hat{S}}^1} \Exp{\frac{L_{i,\hat{S}}^0}{L_{i,\hat{S}}^1}}.
    \end{eqnarray*}
    Thus
    \begin{eqnarray*}
        &&\hspace{-6mm}\delta^{k+1} \\
        &\leq& \delta^k - \sum_{i=1}^b t_i \Exp{\I{i\in \hat{S}}} \Exp{\norm{\nabla_i f(X^k)}_{(i) \star}} \\
        &&+ \sum_{i=1}^b 2 t_i \bar{\rho}_i \Exp{\I{i\in \hat{S}}} \parens{1 - \Exp{\I{i\in \hat{S}}} \beta_i}^k P_i^0 \\
        &&+ 4 \max_{i\in [b]} \parens{t_i^2 b_i} \Exp{\max_{i\in [b]} L_{i,\hat{S}}^1} \sum_{l=0}^{k-1} \parens{1 - \min_{i\in [b]} \parens{\Exp{\I{i\in \hat{S}}} \beta_i}}^{k-1-l} \delta^l \\
        &&+ \sum_{i=1}^b \frac{t_i^2 b_i}{\Exp{\I{i\in \hat{S}}} \beta_i} \Exp{\frac{L_{i,\hat{S}}^0}{L_{i,\hat{S}}^1}}
        + 2 \max_{i\in [b]} \parens{\Exp{\I{i\in \hat{S}} L_{i,\hat{S}}^1}} \Exp{\max_{i\in [b]} \parens{t_i^2 L_{i,\hat{S}}^1}} \delta^k \\
        &&+ \frac{1}{2} \sum_{i=1}^b t_i^2 \Exp{\I{i\in \hat{S}} L_{i,\hat{S}}^1} \Exp{\frac{L_{i,\hat{S}}^0}{L_{i,\hat{S}}^1}} \\
        &&+ \frac{1}{2} \sum_{i=1}^b t_i^2\Exp{\I{i\in \hat{S}} L^0_{i,\hat{S}}}
        + \sum_{i=1}^b \frac{2 t_i^2 \bar{\rho}_i \Exp{\parens{1-\I{i\in \hat{S}} \beta_i} L_{i,\hat{S}}^0}}{\beta_i \underline{\rho}_i} \\
        &&+ \sum_{i=1}^b 2 t_i \bar{\rho}_i \Exp{\I{i\in \hat{S}}} \sigma_i \sqrt{\beta_i} \\
        &\leq& \parens{1 + 2 \max_{i\in [b]} \parens{\Exp{\I{i\in \hat{S}} L_{i,\hat{S}}^1}} \Exp{\max_{i\in [b]} \parens{t_i^2 L_{i,\hat{S}}^1}}} \delta^k \\
        &&- \sum_{i=1}^b t_i \Exp{\I{i\in \hat{S}}} \Exp{\norm{\nabla_i f(X^k)}_{(i) \star}} \\
        &&+ 4 \max_{i\in [b]} \parens{t_i^2 b_i} \Exp{\max_{i\in [b]} L_{i,\hat{S}}^1} \sum_{l=0}^{k-1} \parens{1 - \min_{i\in [b]} \parens{\Exp{\I{i\in \hat{S}}} \beta_i}}^{k-1-l} \delta^l \\
        &&+ \sum_{i=1}^b 2 t_i \bar{\rho}_i \Exp{\I{i\in \hat{S}}} \parens{1 - \Exp{\I{i\in \hat{S}}} \beta_i}^k P_i^0 \\
        &&+ \sum_{i=1}^b t_i^2 \parens{\frac{2 \bar{\rho}_i \Exp{\parens{1-\I{i\in \hat{S}} \beta_i} L_{i,\hat{S}}^1}}{\beta_i \underline{\rho}_i} + \frac{\Exp{\I{i\in \hat{S}} L_{i,\hat{S}}^1}}{2}} \Exp{\frac{L_{i,\hat{S}}^0}{L_{i,\hat{S}}^1}} \\
        &&+ \sum_{i=1}^b t_i^2 \parens{\frac{2 \bar{\rho}_i \Exp{\parens{1-\I{i\in \hat{S}} \beta_i} L_{i,\hat{S}}^0}}{\beta_i \underline{\rho}_i} + \frac{\Exp{\I{i\in \hat{S}} L^0_{i,\hat{S}}}}{2}} \\
        &&+ 2 \sum_{i=1}^b t_i \Exp{\I{i\in \hat{S}}} \bar{\rho}_i \sigma_i \sqrt{\beta_i} \\
        &=& \parens{1 + c_1} \delta^k
        - \sum_{i=1}^b t_i \Exp{\I{i\in \hat{S}}} \Exp{\norm{\nabla_i f(X^k)}_{(i) \star}} \\
        &&+ c_2 \sum_{l=0}^{k-1} \parens{1 - \min_{i\in [b]} \parens{\Exp{\I{i\in \hat{S}}} \beta_i}}^{k-1-l} \delta^l \\
        &&+ \sum_{i=1}^b 2 t_i \bar{\rho}_i \Exp{\I{i\in \hat{S}}} \parens{1 - \Exp{\I{i\in \hat{S}}} \beta_i}^k P_i^0
        + \sum_{i=1}^b t_i^2 c_{3,i}
        + \sum_{i=1}^b t_i c_{4,i},
    \end{eqnarray*}
    where
    \begin{eqnarray*}
        c_1 &\eqdef& 2 \max_{i\in [b]} \parens{\Exp{\I{i\in \hat{S}} L_{i,\hat{S}}^1}} \Exp{\max_{i\in [b]} \parens{t_i^2 L_{i,\hat{S}}^1}}, \\
        c_2 &\eqdef& 4 \max_{i\in [b]} \parens{t_i^2 b_i} \Exp{\max_{i\in [b]} L_{i,\hat{S}}^1}, \\
        c_{3,i} &\eqdef& \parens{\frac{2 \bar{\rho}_i \Exp{\parens{1-\I{i\in \hat{S}} \beta_i} L_{i,\hat{S}}^1}}{\beta_i \underline{\rho}_i} + \frac{\Exp{\I{i\in \hat{S}} L_{i,\hat{S}}^1}}{2}} \Exp{\frac{L_{i,\hat{S}}^0}{L_{i,\hat{S}}^1}} \\
        &&+ \parens{\frac{2 \bar{\rho}_i \Exp{\parens{1-\I{i\in \hat{S}} \beta_i} L_{i,\hat{S}}^0}}{\beta_i \underline{\rho}_i} + \frac{\Exp{\I{i\in \hat{S}} L^0_{i,\hat{S}}}}{2}}, \\
        c_{4,i} &\eqdef& 2 \Exp{\I{i\in \hat{S}}} \bar{\rho}_i \sigma_i \sqrt{\beta_i}.
    \end{eqnarray*}
    Now, let us introduce a weighting sequence $w^k \eqdef w^{k-1} \parens{1 + c_1 + \frac{c_2}{\min_{i\in [b]} \parens{\Exp{\I{i\in \hat{S}}} \beta_i}}}^{-1}$, where $w^{-1} = 1$ and $W^K \eqdef \sum_{k=0}^{K} w^k$. Then, multiplying the inequality above by $w^k$ and summing over the first $K+1$ iterations gives
    \begin{eqnarray*}
        \sum_{k=0}^{K} w^k \delta^{k+1}
        &\leq& \sum_{k=0}^{K} w^k \parens{1 + c_1} \delta^k
        - \sum_{k=0}^{K} w^k \sum_{i=1}^b t_i \Exp{\I{i\in \hat{S}}} \Exp{\norm{\nabla_i f(X^k)}_{(i) \star}} \\
        &&+ \sum_{k=0}^{K} w^k c_2 \sum_{l=0}^{k-1} \parens{1 - \min_{i\in [b]} \parens{\Exp{\I{i\in \hat{S}}} \beta_i}}^{k-1-l} \delta^l \\
        &&+ \sum_{k=0}^{K} w^k \sum_{i=1}^b 2 t_i \bar{\rho}_i \Exp{\I{i\in \hat{S}}} \parens{1 - \Exp{\I{i\in \hat{S}}} \beta_i}^k P_i^0 \\
        &&+ \sum_{k=0}^{K} w^k \sum_{i=1}^b t_i^2 c_{3,i}
        + \sum_{k=0}^{K} w^k \sum_{i=1}^b t_i c_{4,i} \\
        &\leq& \parens{1 + c_1} \sum_{k=0}^{K} w^k \delta^k
        - \sum_{k=0}^{K} w^k \sum_{i=1}^b t_i \Exp{\I{i\in \hat{S}}} \Exp{\norm{\nabla_i f(X^k)}_{(i) \star}} \\
        &&+ c_2 \sum_{l=0}^{K-1} \sum_{k=l+1}^{K} w^k \parens{1 - \min_{i\in [b]} \parens{\Exp{\I{i\in \hat{S}}} \beta_i}}^{k-1-l} \delta^l \\
        &&+ \sum_{i=1}^b 2 t_i \bar{\rho}_i \Exp{\I{i\in \hat{S}}} P_i^0 \sum_{k=0}^{K} \parens{1 - \Exp{\I{i\in \hat{S}}} \beta_i}^k \\
        &&+ W^K \sum_{i=1}^b t_i^2 c_{3,i}
        + W^K \sum_{i=1}^b t_i c_{4,i},
    \end{eqnarray*}
    where in the last line we used the fact that $w^k \leq w^{k-1} \leq w^{-1} = 1$. Therefore,
    \begin{eqnarray*}
        \sum_{k=0}^{K} w^k \delta^{k+1}
        &\leq& \parens{1 + c_1} \sum_{k=0}^{K} w^k \delta^k
        - \sum_{k=0}^{K} w^k \sum_{i=1}^b t_i \Exp{\I{i\in \hat{S}}} \Exp{\norm{\nabla_i f(X^k)}_{(i) \star}} \\
        &&+ c_2 \sum_{l=0}^{K-1} \sum_{k=l+1}^{K} w^l \parens{1 - \min_{i\in [b]} \parens{\Exp{\I{i\in \hat{S}}} \beta_i}}^{k-1-l} \delta^l \\
        &&+ \sum_{i=1}^b \frac{2 t_i \bar{\rho}_i \Exp{\I{i\in \hat{S}}}}{\Exp{\I{i\in \hat{S}}} \beta_i} P_i^0
        + W^K \sum_{i=1}^b t_i^2 c_{3,i}
        + W^K \sum_{i=1}^b t_i c_{4,i} \\
        &\leq& \parens{1 + c_1} \sum_{k=0}^{K} w^k \delta^k
        - \sum_{k=0}^{K} w^k \sum_{i=1}^b t_i \Exp{\I{i\in \hat{S}}} \Exp{\norm{\nabla_i f(X^k)}_{(i) \star}} \\
        &&+ c_2 \sum_{l=0}^{K-1} w^l \delta^l \sum_{k=l+1}^{K} \parens{1 - \min_{i\in [b]} \parens{\Exp{\I{i\in \hat{S}}} \beta_i}}^{k-1-l} \\
        &&+ \sum_{i=1}^b \frac{2 t_i \bar{\rho}_i}{\beta_i} P_i^0
        + W^K \sum_{i=1}^b t_i^2 c_{3,i}
        + W^K \sum_{i=1}^b t_i c_{4,i} \\
        &\leq& \parens{1 + c_1} \sum_{k=0}^{K} w^k \delta^k
        - \sum_{k=0}^{K} w^k \sum_{i=1}^b t_i \Exp{\I{i\in \hat{S}}} \Exp{\norm{\nabla_i f(X^k)}_{(i) \star}} \\
        &&+ c_2 \sum_{l=0}^{K-1} w^l \delta^l \sum_{k=0}^{\infty} \parens{1 - \min_{i\in [b]} \parens{\Exp{\I{i\in \hat{S}}} \beta_i}}^k \\
        &&+ \sum_{i=1}^b \frac{2 t_i \bar{\rho}_i}{\beta_i} P_i^0
        + W^K \sum_{i=1}^b t_i^2 c_{3,i}
        + W^K \sum_{i=1}^b t_i c_{4,i} \\
        &\leq& \parens{1 + c_1 + \frac{c_2}{\min_{i\in [b]} \parens{\Exp{\I{i\in \hat{S}}} \beta_i}}} \sum_{k=0}^{K} w^k \delta^k \\
        &&- \sum_{k=0}^{K} w^k \sum_{i=1}^b t_i \Exp{\I{i\in \hat{S}}} \Exp{\norm{\nabla_i f(X^k)}_{(i) \star}} \\
        &&+ \sum_{i=1}^b \frac{2 t_i \bar{\rho}_i}{\beta_i} P_i^0
        + W^K \sum_{i=1}^b t_i^2 c_{3,i}
        + W^K \sum_{i=1}^b t_i c_{4,i} \\
        &=& \sum_{k=0}^{K} w^{k-1} \delta^k
        - \sum_{k=0}^{K} w^k \sum_{i=1}^b t_i \Exp{\I{i\in \hat{S}}} \Exp{\norm{\nabla_i f(X^k)}_{(i) \star}} \\
        &&+ \sum_{i=1}^b \frac{2 t_i \bar{\rho}_i}{\beta_i} P_i^0
        + W^K \sum_{i=1}^b t_i^2 c_{3,i}
        + W^K \sum_{i=1}^b t_i c_{4,i}.
    \end{eqnarray*}
    Rearranging the terms and dividing by $W^K$, we get
    \begin{eqnarray*}
        &&\hspace{-8mm}\min_{k=0,\ldots,K} \sum_{i=1}^b t_i \Exp{\I{i\in \hat{S}}} \Exp{\norm{\nabla_i f(X^k)}_{(i) \star}} \\
        &\leq& \sum_{k=0}^{K} \sum_{i=1}^b \frac{w^k}{W^K} t_i \Exp{\I{i\in \hat{S}}} \Exp{\norm{\nabla_i f(X^k)}_{(i) \star}} \\
        &\leq& \frac{1}{W^K} \sum_{k=0}^{K} \parens{w^{k-1} \delta^k - w^k \delta^{k+1}}
        + \frac{1}{W^K} \sum_{i=1}^b \frac{2 t_i \bar{\rho}_i}{\beta_i} P_i^0
        + \sum_{i=1}^b t_i^2 c_{3,i}
        + \sum_{i=1}^b t_i c_{4,i} \\
        &\leq& \frac{\delta^0}{W^K}
        + \frac{1}{W^K} \sum_{i=1}^b \frac{2 t_i \bar{\rho}_i}{\beta_i} P_i^0
        + \sum_{i=1}^b t_i^2 c_{3,i}
        + \sum_{i=1}^b t_i c_{4,i}.
    \end{eqnarray*}
    Now, note that taking $t_i = \frac{\eta_i}{(K+1)^{3/4}}$, where $$\textstyle\eta_i^2 \leq \min\brac{\frac{(K+1)^{1/2}}{4 \Exp{\I{i\in \hat{S}} L_{i,\hat{S}}^1} \Exp{\max_{i\in [b]} L_{i,\hat{S}}^1}}, \frac{(K+1)^{1/2} \underline{\rho}_i \min_{i\in [b]} \parens{\Exp{\I{i\in \hat{S}}} \beta_i}}{16 \bar{\rho}_i \Exp{\I{i\in \hat{S}}} \Exp{\parens{1-\I{i\in \hat{S}} \beta_i} L_{i,\hat{S}}^1} \Exp{\max_{i\in [b]} L_{i,\hat{S}}^1}}, 1},$$ we have
    \begin{align*}
        2 (K+1) \max_{i\in [b]} \parens{\Exp{\I{i\in \hat{S}} L_{i,\hat{S}}^1}} \Exp{\max_{i\in [b]} \parens{t_i^2 L_{i,\hat{S}}^1}} &\leq \frac{1}{2}, \\
        (K+1) \frac{4 \max_{i\in [b]} \parens{t_i^2 b_i} \Exp{\max_{i\in [b]} L_{i,\hat{S}}^1}}{\min_{i\in [b]} \parens{\Exp{\I{i\in \hat{S}}} \beta_i}} &\leq \frac{1}{2},
    \end{align*}
    and hence $(K+1) \parens{c_1 + \frac{c_2}{\min_{i\in [b]} \parens{\Exp{\I{i\in \hat{S}}} \beta_i}}} \leq 1$. Thus, the weights satisfy
    \begin{align*}
        W^K &= \sum_{k=0}^{K} w^k
        \geq (K+1) w^K
        = \frac{(K+1) w^{-1}}{\parens{1 + c_1 + \frac{c_2}{\min_{i\in [b]} \parens{\Exp{\I{i\in \hat{S}}} \beta_i}}}^{K+1}} \\
        &\geq \frac{K+1}{\exp\parens{(K+1) \parens{c_1 + \frac{c_2}{\min_{i\in [b]} \parens{\Exp{\I{i\in \hat{S}}} \beta_i}}}}}
        \geq \frac{K+1}{\exp(1)} \geq \frac{K+1}{3},
    \end{align*}
    meaning that
    \begin{align*}
        &\min_{k=0,\ldots,K} \sum_{i=1}^b t_i \Exp{\I{i\in \hat{S}}} \Exp{\norm{\nabla_i f(X^k)}_{(i) \star}} \\
        &\leq \frac{3 \delta^0}{K+1}
        + \frac{6}{K+1} \sum_{i=1}^b \frac{t_i \bar{\rho}_i}{\beta_i} P_i^0
        + \sum_{i=1}^b t_i^2 c_{3,i}
        + \sum_{i=1}^b t_i c_{4,i} \\
        &= \frac{3 \delta^0}{K+1}
        + \frac{6}{K+1} \sum_{i=1}^b \frac{\eta_i \bar{\rho}_i}{\beta_i (K+1)^{3/4}} P_i^0 \\
        &\quad+ \sum_{i=1}^b \frac{\eta_i^2}{(K+1)^{3/2}} \frac{2 \bar{\rho}_i}{\beta_i \underline{\rho}_i} \parens{\Exp{\parens{1-\I{i\in \hat{S}} \beta_i} L_{i,\hat{S}}^0} + \Exp{\parens{1-\I{i\in \hat{S}} \beta_i} L_{i,\hat{S}}^1} \Exp{\frac{L_{i,\hat{S}}^0}{L_{i,\hat{S}}^1}}} \\
        &\quad+ \sum_{i=1}^b \frac{\eta_i^2}{2 (K+1)^{3/2}} \parens{\Exp{\I{i\in \hat{S}} L^0_{i,\hat{S}}} + \Exp{\I{i\in \hat{S}} L_{i,\hat{S}}^1} \Exp{\frac{L_{i,\hat{S}}^0}{L_{i,\hat{S}}^1}}} \\
        &\quad+ \sum_{i=1}^b \frac{2 \eta_i}{(K+1)^{3/4}} \Exp{\I{i\in \hat{S}}} \bar{\rho}_i \sigma_i \sqrt{\beta_i}.
    \end{align*}
    Dividing by $\frac{1}{b} \sum_{l=1}^b \Exp{\I{l \in \hat{S}}} t_l = \frac{1}{(K+1)^{3/4}} \frac{1}{b} \sum_{l=1}^b \Exp{\I{l \in \hat{S}}} \eta_l$ gives
    \begin{align*}
        &\min_{k=0,\ldots,K} \sum_{i=1}^b \frac{\Exp{\I{i\in \hat{S}}} \eta_i}{\frac{1}{b} \sum_{l=1}^b \Exp{\I{l \in \hat{S}}} \eta_l} \Exp{\norm{\nabla_i f(X^k)}_{(i) \star}} \\
        &\leq \frac{3 \delta^0}{(K+1)^{1/4} \frac{1}{b} \sum_{l=1}^b \Exp{\I{l \in \hat{S}}} \eta_l}
        + \frac{6}{K+1} \sum_{i=1}^b \frac{\eta_i \bar{\rho}_i}{\beta_i \frac{1}{b} \sum_{l=1}^b \Exp{\I{l \in \hat{S}}} \eta_l} P_i^0 \\
        &\quad+ \sum_{i=1}^b \frac{\eta_i^2}{(K+1)^{3/4} \frac{1}{b} \sum_{l=1}^b \Exp{\I{l \in \hat{S}}} \eta_l} \frac{2 \bar{\rho}_i}{\beta_i \underline{\rho}_i} \\
        &\qquad+\times\parens{\Exp{\parens{1-\I{i\in \hat{S}} \beta_i} L_{i,\hat{S}}^0} + \Exp{\parens{1-\I{i\in \hat{S}} \beta_i} L_{i,\hat{S}}^1} \Exp{\frac{L_{i,\hat{S}}^0}{L_{i,\hat{S}}^1}}} \\
        &\quad+ \sum_{i=1}^b \frac{\eta_i^2}{2 (K+1)^{3/4} \frac{1}{b} \sum_{l=1}^b \Exp{\I{l \in \hat{S}}} \eta_l} \parens{\Exp{\I{i\in \hat{S}} L^0_{i,\hat{S}}} + \Exp{\I{i\in \hat{S}} L_{i,\hat{S}}^1} \Exp{\frac{L_{i,\hat{S}}^0}{L_{i,\hat{S}}^1}}} \\
        &\quad+ \sum_{i=1}^b \frac{2 \Exp{\I{i\in \hat{S}}} \eta_i}{\frac{1}{b} \sum_{l=1}^b \Exp{\I{l \in \hat{S}}} \eta_l} \bar{\rho}_i \sigma_i \sqrt{\beta_i} \\
        &\leq \frac{3 \delta^0}{(K+1)^{1/4} \frac{1}{b} \sum_{l=1}^b \Exp{\I{l \in \hat{S}}} \eta_l}
        + \frac{6}{(K+1)^{1/2}} \sum_{i=1}^b \frac{\eta_i \bar{\rho}_i}{\frac{1}{b} \sum_{l=1}^b \Exp{\I{l \in \hat{S}}} \eta_l} P_i^0 \\
        &\quad+ \sum_{i=1}^b \frac{\eta_i^2}{(K+1)^{1/4} \frac{1}{b} \sum_{l=1}^b \Exp{\I{l \in \hat{S}}} \eta_l} \frac{2 \bar{\rho}_i}{\underline{\rho}_i} \parens{\Exp{L_{i,\hat{S}}^0} + \Exp{L_{i,\hat{S}}^1} \Exp{\frac{L_{i,\hat{S}}^0}{L_{i,\hat{S}}^1}}} \\
        &\quad+ \sum_{i=1}^b \frac{\eta_i^2}{2 (K+1)^{3/4} \frac{1}{b} \sum_{l=1}^b \Exp{\I{l \in \hat{S}}} \eta_l} \parens{\Exp{\I{i\in \hat{S}} L^0_{i,\hat{S}}} + \Exp{\I{i\in \hat{S}} L_{i,\hat{S}}^1} \Exp{\frac{L_{i,\hat{S}}^0}{L_{i,\hat{S}}^1}}} \\
        &\quad+ \sum_{i=1}^b \frac{2 \Exp{\I{i\in \hat{S}}} \eta_i}{(K+1)^{1/4} \frac{1}{b} \sum_{l=1}^b \Exp{\I{l \in \hat{S}}} \eta_l} \bar{\rho}_i \sigma_i,
    \end{align*}
    where in the last equality we set $\beta_i = \frac{1}{(K+1)^{1/2}}$. Finally, the initialization $M_i^0 = \nabla_i f(X^0; \xi^0)$ guarantees that
    \begin{eqnarray*}
        P_i^0 \eqdef \Exp{\norm{\nabla_i f(X^0) - M_i^0}_2}
        \leq \sqrt{\Exp{\norm{\nabla_i f(X^0) - \nabla_i f(X^0; \xi^0)}^2_2}}
        \overset{\eqref{as:bounded_var}}{\leq} \sigma_i.
    \end{eqnarray*}
\end{proof}

\newpage

\section{Experiments}\label{sec:experiments_app}

We evaluate {\newalgsmall} on three standard benchmarks--\texttt{MNIST}, \texttt{Fashion-MNIST}, and \texttt{CIFAR-10}--using 3-layer convolutional neural networks (CNNs) of varying capacity.
Our goal is to study how partial layer updates of {\newalgsmall} accelerate training relative to standard \algnamesmall{Muon} (a special case of {\newalgsmall} with full-network updates).

Experiments were run on various NVIDIA GPUs based on availability.
\texttt{MNIST} and \texttt{Fashion-MNIST} experiments were run on Tesla V100-SXM2-32GB. \texttt{CIFAR-10} experiments were performed on a mix of A100-SXM4-80GB, Tesla P100-PCIE-16GB, and GeForce GTX 1080 Ti GPUs.

\subsection{Key Training Parameters}\label{subsec:grid_search}

We identify several key training parameters and conduct a grid search across these parameters to systematically explore their effects.

\paragraph{Layer sampling distribution.}  
The choice of layer distribution determines which layers are updated at each iteration and is central to the behavior of {\newalgsmall}. We considered the following strategies:

\begin{itemize}
    \item \textbf{Uniform Distribution:} Each layer is sampled with equal probability ($p_i = \nicefrac{1}{b}$ for all $i\in[b]$ in the notation from \Cref{sec:rpt}).
    \item \textbf{Linear Distribution:} Sampling probability increases linearly with the layer index, favoring deeper layers. Empirically, this distribution performed poorly.
    \item \textbf{Quadratic Distribution:} Probability increases quadratically with depth, strongly biasing toward deeper layers. Similarly, this performed poorly.
    \item \textbf{Exponential Distribution:} Layer probability grows exponentially with depth, emphasizing deeper layers. This also proved suboptimal.
    \item \textbf{Epoch-Shift Distribution:} The distribution biases the sampling towards shallow layers at the beginning of training, gradually shifting towards deeper layers as the training progresses (see more details in \Cref{sec:epoch_shift}). This dynamic adjustment might help in balancing early feature extraction and later complex feature learning during the optimization process--see example in \Cref{fig:distribution_epoch_shift}.
\end{itemize}

\paragraph{Batch size and learning rate.}  
We test batch sizes $\{64, 512, 8192, 16384, 32768\}$ and learning rates $\{0.1, 0.01, 0.001\}$.

\paragraph{Model depth and capacity.}  
To assess the impact of network complexity, we train multiple 3-layer CNNs with varying channel configurations: $[8,16,32]$, $[16,32,64]$, $[64,128,256]$, $[128,256,512]$, $[64,16,8]$, and $[256,128,64]$.

\paragraph{Fixed parameters.} We fix the number of epochs to $20$ or $50$, momentum to $0.5$, and the number of Newton-Schulz iterations to $5$ (see \Cref{sec:muon}).

This grid search allows us to systematically evaluate {\newalgsmall}. While many combinations performed poorly (and are therefore omitted in the results that follow), the experiments highlighted configurations--particularly with uniform or epoch-shift distributions--that consistently reduce training time while maintaining final accuracy.

\subsection{Epoch-Shift Distribution}\label{sec:epoch_shift}

\begin{figure}[t]
    \centering
    \includegraphics[width=0.7\linewidth]{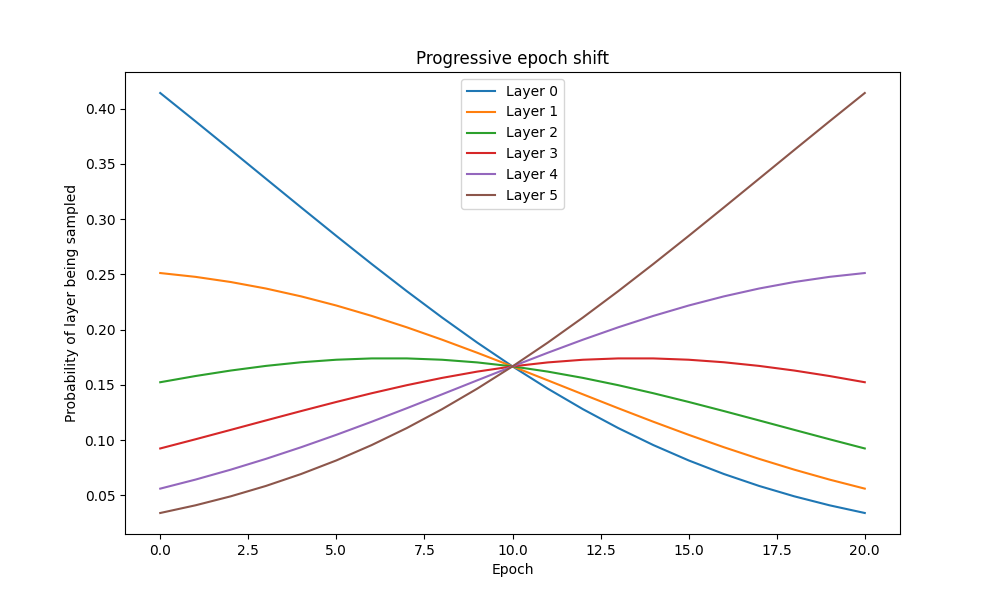}
    \caption{Evolution of the layer sampling distribution as a function of the epochs. Shallow layers are more trained in the first epochs but their probabilities of being sampled decrease with the epochs. This effect can be amplified or reduced by varying the value of $\alpha$; here we chose $\alpha=0.5$.}
    \label{fig:distribution_epoch_shift}
\end{figure}

{\newalgsmall} can dynamically adjust which layers are updated at each iteration. The strategy we consider is the \emph{epoch-shift} distribution, which biases training towards shallow layers in the early epochs, gradually shifting focus to deeper layers as training progresses (see \Cref{fig:distribution_epoch_shift}). Such a schedule balances early feature extraction with later complex feature learning, improving convergence in practice.

Formally, let $b$ denote the total number of layers, $\alpha$ be a constant controlling the sharpness of the bias, and define the training progress as
\[
\text{progress} = \frac{\text{epoch}}{\text{max\_epochs}} \in [0,1].
\]
The weight for each layer $i$ is given by
\[
w_i = \exp \left( \alpha \left[ (1 - \text{progress})(b - 1 - i) + \text{progress} \cdot i \right] \right).
\]
We then normalize the weights to obtain a valid probability distribution:
\[
W = \sum_{i=0}^{b-1} w_i, \qquad 
p_i = \frac{w_i}{W}.
\]

By adjusting $\alpha$, one can control the rate at which training emphasis shifts from shallow to deeper layers. In \Cref{fig:distribution_epoch_shift}, we set $\alpha = 0.5$.

\subsection{Evaluation Metrics}

We repeat every experiment over multiple random seeds. When reporting aggregated results, we use two complementary procedures:

\begin{itemize}
    \item \textbf{Normalized curve averaging:} We normalize wall-clock time for each run so that \algnamesmall{Muon} always ends at $t=1$, interpolate both \algnamesmall{Muon} and {\newalgsmall} curves onto this grid, and average over seeds.
    \item \textbf{Time-to-target evaluation:} We measure the wall-clock time to reach fixed accuracy thresholds (e.g. $60\%$,$70\%$,$80\%$,...), reporting the ratio between \algnamesmall{Muon} and {\newalgsmall}. This metric is interpolation-free and directly quantifies practical speedup.
\end{itemize}

\subsection{Results on MNIST}

We first evaluate {\newalgsmall} on the \texttt{MNIST} dataset. Figure \ref{fig:MNIST_unif_tr} shows a representative run comparing standard full-network \algnamesmall{Muon} training with {\newalgsmall} using a uniform layer sampling distribution.
Although the per-epoch train accuracy of {\newalgsmall} is initially lower than that of full-layer \algnamesmall{Muon}, the wall-clock training time tells a different story: for training times up to approximately 150 seconds, {\newalgsmall} consistently achieves higher accuracy. In practice, this means that reaching a train accuracy of $95\%$ or less is faster with {\newalgsmall}, highlighting its efficiency advantage.

To account for variability across independent runs, we evaluate two aggregation strategies. \Cref{fig:MNIST_normalized_uniform} shows normalized curve averaging, where each run is rescaled to a common time grid. Figure \ref{fig:MNIST_FashionMNIST_epochshift_bar} (left) presents averaged time-to-target evaluation across multiple seeds, reporting the ratio of \algnamesmall{Muon} time to {\newalgsmall} time for fixed accuracy thresholds.

The time-to-target results clearly demonstrate a speedup of up to $1.4\times$ across thresholds of $60\%, 70\%, 80\%$, and $99\%$, with slightly lower speedup for the $90\%$ threshold. We note variability across seeds, emphasizing the stochastic nature of layer sampling. Despite this, {\newalgsmall} consistently provides practical acceleration across training.

\begin{figure}[t]
    \centering
    \includegraphics[width=\linewidth]{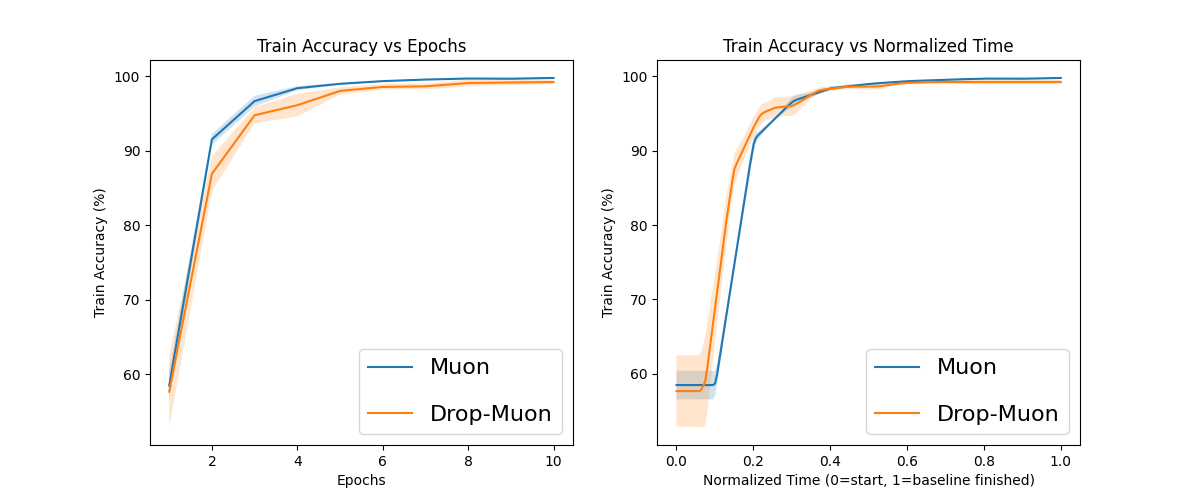}
    \caption{Normalized curve averaging of several runs of  \algnamesmall{Muon} and {\newalgsmall} with uniform index sampling on \texttt{MNIST}.
    Batch size $=8192$, learning rate $=0.1$, channels $=[64,128,256]$.}
    \label{fig:MNIST_normalized_uniform}
\end{figure}

\subsection{Results on Fashion-MNIST}

Figure \ref{fig:images//FashionMNIST/train_epochs_and_train_time_bs32768_lr0.1_ch64-128-256_cutoffepoch-shift.png} shows a typical run comparing standard \algnamesmall{Muon} training with {\newalgsmall} using the epoch-shift layer sampling distribution. Early in training, {\newalgsmall} achieves faster progress in wall-clock time, even if per-epoch accuracy is slightly lower.

\begin{figure}[t]
    \centering
    \includegraphics[width=0.9\linewidth]{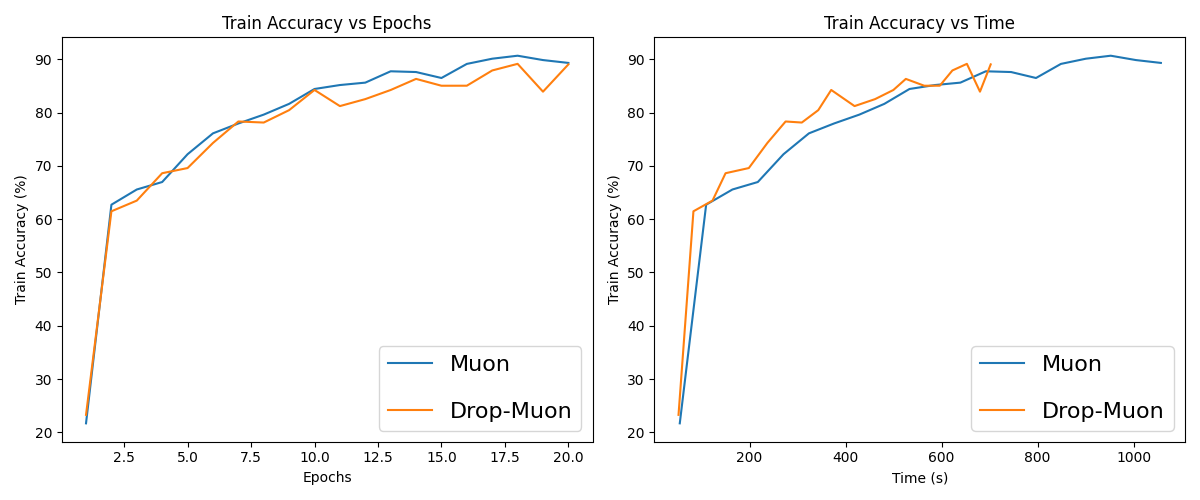}
    \caption{Evolution of the training accuracy for \algnamesmall{Muon} and {\newalgsmall} with epoch-shift index sampling on \texttt{Fashion-MNIST}. Batch size $=32768$, learning rate $=0.1$, channels $=[64,128,256]$.}
    \label{fig:images//FashionMNIST/train_epochs_and_train_time_bs32768_lr0.1_ch64-128-256_cutoffepoch-shift.png}
\end{figure}

Normalized curve averaging across multiple seeds is presented in \Cref{fig:images//FashionMNIST/normalized_curve_averaging_bs32768_lr0.1_ch64-128-256_cutoffepoch-shift.png}. Although the curves appear smoother at higher accuracy, this should not be interpreted as a reduction in variance; early-stage training is inherently noisier due to layer sampling, whereas later training reflects fewer updates per wall-clock time.

\begin{figure}[t]
    \centering
    \includegraphics[width=0.9\linewidth]{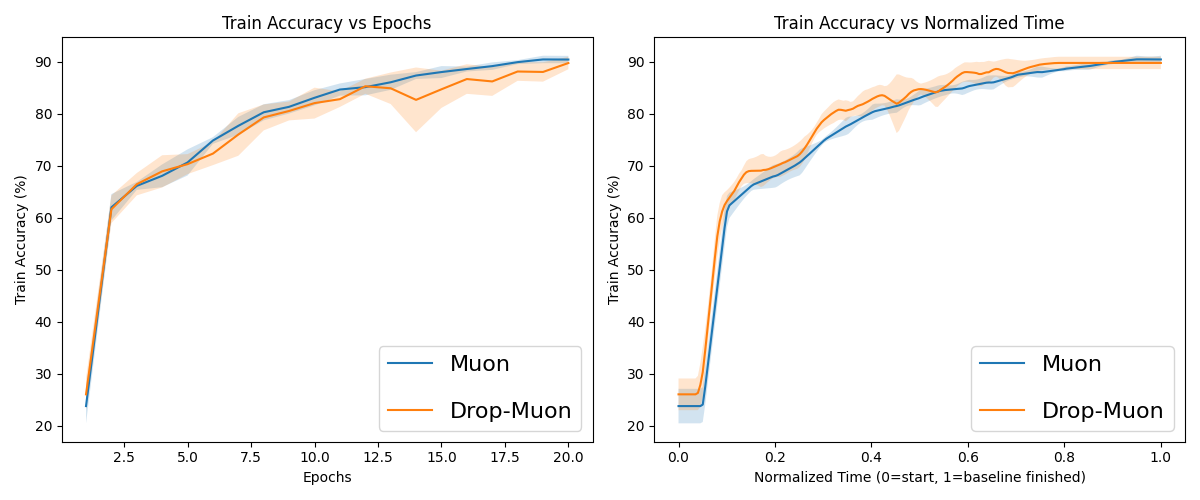}
    \caption{Normalized curve averaging of several runs of  \algnamesmall{Muon} and {\newalgsmall} with epoch-shift index sampling on \texttt{Fashion-MNIST}.
    Batch size $=32768$, learning rate $=0.1$, channels $=[64,128,256]$.}
    \label{fig:images//FashionMNIST/normalized_curve_averaging_bs32768_lr0.1_ch64-128-256_cutoffepoch-shift.png}
\end{figure}

Figure \ref{fig:MNIST_FashionMNIST_epochshift_bar} (right) reports the averaged time-to-target evaluation. {\newalgsmall} consistently achieves speedups of approximately $1.2$ across all considered accuracy thresholds. The $99\%$ threshold is omitted as neither \algnamesmall{Muon} nor {\newalgsmall} reach this level within the plotted runs.

Overall, \texttt{Fashion-MNIST} experiments reinforce that {\newalgsmall} provides consistent practical speedups, similar to \texttt{MNIST}, while maintaining comparable final accuracy.

\subsection{Results on CIFAR-10}

\begin{figure}[t]
    \centering
    \includegraphics[width=0.9\linewidth]{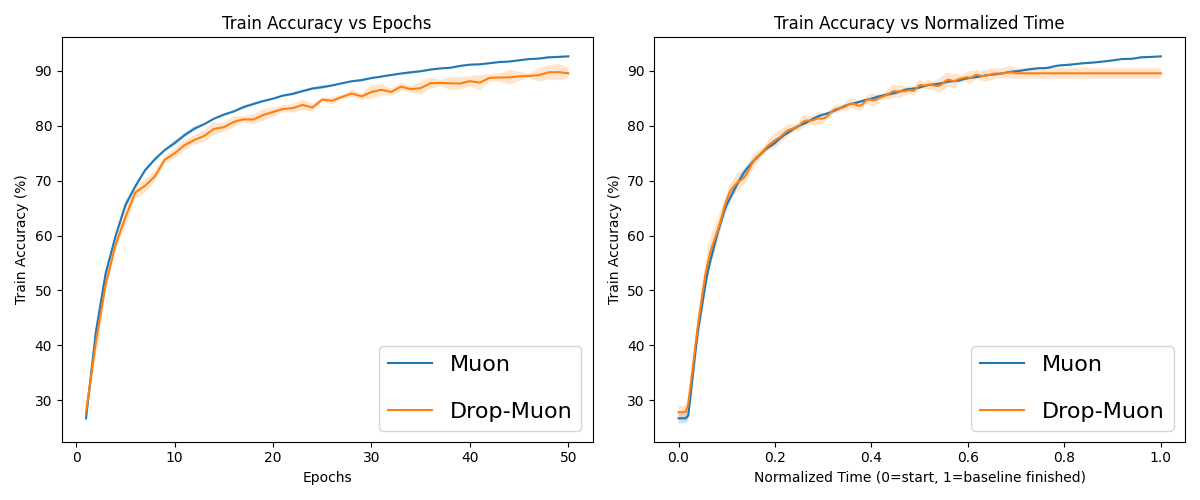}
    \caption{Normalized curve averaging of several runs of  \algnamesmall{Muon} and {\newalgsmall} with epoch-shift index sampling on \texttt{CIFAR-10}.
    Batch size $=8192$, learning rate $=0.1$, channels $=[128,256,512]$.}
    \label{fig:normalized_curve_averaging_bs8192_lr0.1_ch128-256-512_cutoffepoch-shift_batch_normFalse}
\end{figure}

\texttt{CIFAR-10} requires more sophisticated network architectures. We use a CNN with channels $[128,256,512]$ and insert batch normalization layers between each convolution and ReLU activation to improve training stability.

\Cref{fig:cifar10_epochshift_tr} shows an example run comparing standard full-layer \algnamesmall{Muon} training with {\newalgsmall} using the epoch-shift layer sampling distribution. While {\newalgsmall} has slightly lower per-epoch training accuracy, it achieves faster progress in terms of wall-clock time.
\Cref{fig:normalized_curve_averaging_bs8192_lr0.1_ch128-256-512_cutoffepoch-shift_batch_normFalse} shows normalized curve averaging across multiple seeds (the flat end of the curve reflects the earlier completion of {\newalgsmall}).

Time-to-target results in \Cref{fig:cifar10_epochshift_bar} indicate a notable speedup at the $90\%$ train accuracy threshold. However, for lower thresholds ($60\%, 70\%, 80\%$), speedups are minimal or absent for this configuration. We attribute this limitation primarily to sub-optimal hyperparameter choices (see \Cref{sec:experiments}).

\begin{figure}
    \centering
    \includegraphics[width=0.8\linewidth]{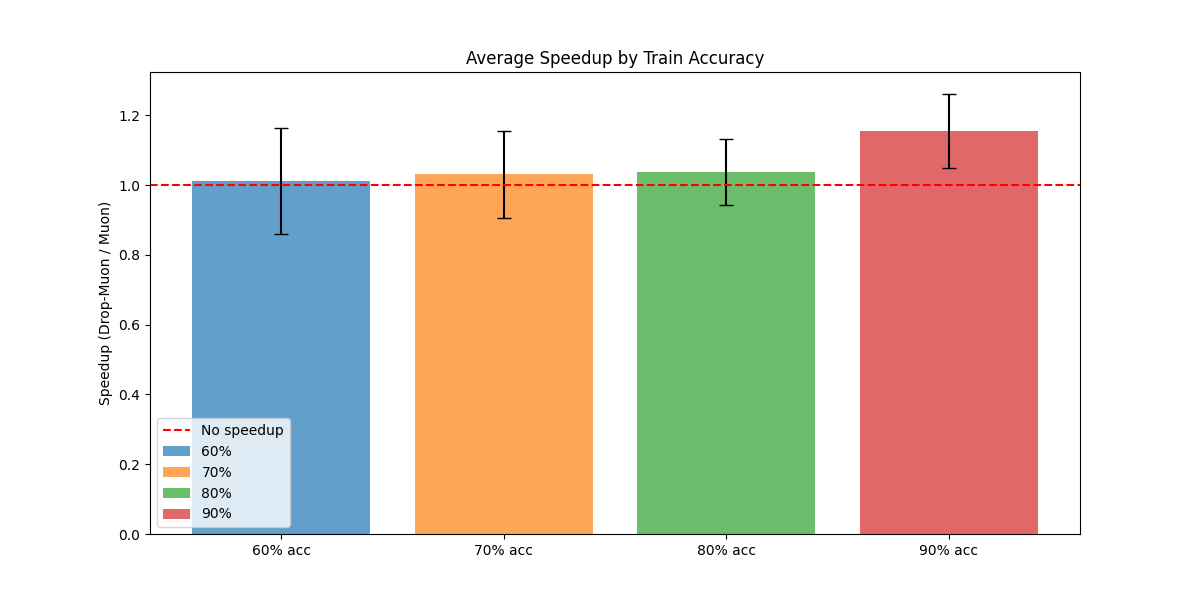}
    \caption{Averaged time-to-target speed-up over multiple runs comparing \algnamesmall{Muon} and {\newalgsmall} with epoch-shift index sampling on \texttt{CIFAR-10} with batch size $8192$, learning rate $0.1$, and channels $[128,256,512]$.}
    \label{fig:cifar10_epochshift_bar}
\end{figure}

\subsection{Discussion and Practical Remarks}

We summarize several key observations and practical considerations arising from our experiments:
\begin{itemize}
\item \textbf{Simplicity of implementation:} {\newalgsmall} can be implemented in just a few lines of code, making it easy to integrate into existing training pipelines.
\item \textbf{Further benefits from tuning:} Performance can be further improved through dedicated tuning of hyperparameters (see \Cref{sec:experiments}).
\item \textbf{Potential for implementation improvements:}
\begin{itemize}
\item Turning off gradient computations at every iteration can be costly for large models. Sampling the cutoff layer only every few iterations can reduce overhead and improve efficiency.
\item When the batch size or model size is small, the gradient computation may be dominated by the Newton-Schulz routine. In such cases, computing all gradients but orthogonalizing only a subset can be advantageous.
\end{itemize}
\item \textbf{Variance across seeds:} The method exhibits non-negligible variance, particularly on smaller datasets. Stabilization mechanisms could mitigate this.
\item \textbf{Adaptive learning of sampling distributions:} A promising future direction is to learn the layer sampling distribution online, so that it automatically adapts to the training dynamics of a given dataset and architecture.
\end{itemize}

\newpage

\section{Useful Facts}

For all $X,Y \in \cX$ and $t>0$, we have:
\begin{align}
    \label{eq:normlmo}
    \norm{\lmo{\cB(0,t)}{G}} &= t \\
    \label{eq:inplmo}
    \inp{G}{\lmo{\cB(X,t)}{G}} &= -t \norm{G}_\star \\
    \label{eq:inpsharp}
    \inp{X}{X^\sharp} &= \norm{X^\sharp}^2, \\
    \label{eq:normsharp}
    \norm{X}_{\star} &= \norm{X^\sharp}, \\
    \label{eq:subdiff}
    \inp{H}{X} &= \norm{X}_{\star}, \quad \norm{H} = 1, \qquad\forall X\neq0.
\end{align}
where $H \in \partial\norm{\cdot}_{\star}(X)$ belongs to the subdifferential of the dual norm.

\begin{lemma}[Variance decomposition]\label{lemma:vardecomp}
    For any random vector $X\in\cX$ and any non-random $c\in\cX$, we have
    \begin{align*}
       \Exp{\norm{X-c}_2^2} = \Exp{\norm{X - \Exp{X}}_2^2} + \norm{\Exp{X}-c}_2^2.
    \end{align*}
\end{lemma}

\begin{lemma}[\citet{riabinin2025gluon}, Lemma 3]\label{lemma:ineq}
    Suppose that $x_1, \ldots, x_p, y_1, \ldots, y_p \in \R$, $\max_{i\in[b]} |x_i| > 0$ and $z_1, \ldots, z_p > 0$. Then
    \begin{align*}
        \sum_{i=1}^b \frac{y_i^2}{z_i} \geq \frac{\parens{\sum_{i=1}^b x_i y_i}^2}{\sum_{i=1}^b z_i x_i^2}.
    \end{align*}
\end{lemma}

\begin{lemma}\label{lemma:arb_layer_gen_smooth}
    Let \Cref{as:arbitrary_layer_gen_smoothness2} hold and let $S\subseteq[b]$. Then, for any vectors $X = [X_1, \ldots, X_b]\in \cX$ and $\Gamma = [\Gamma_1, \ldots, \Gamma_b] \in \cX$ such that $\Gamma_i = 0$ for all $i\not\in S$,
    \begin{eqnarray*}
		\left|f(X + \Gamma) - f(X) - \inp{\nabla f(X)}{\Gamma}\right|
		\leq \sum_{i\in S} \frac{L_{i,S}^0 + L_{i,S}^1 \norm{\nabla_i f(X)}_{(i) \star}}{2}\norm{\Gamma_i}_{(i)}^2.
	\end{eqnarray*}
\end{lemma}

\begin{proof}
	For all $X \in \cX$ we have
	\begin{align*}
	   f(X + \Gamma) &= f(X) + \int_0^1 \inp{\nabla f(X + \tau \Gamma)}{\Gamma} d \tau \\
    	&= f(X) + \int_0^1 \inp{\nabla f(X + \tau \Gamma) - \nabla f(X)}{\Gamma} d \tau + \inp{\nabla f(X)}{\Gamma}.
	\end{align*}
	Therefore, using the Cauchy-Schwarz inequality
	\begin{eqnarray*}
		\left|f(X + \Gamma) - f(X) - \inp{\nabla f(X)}{\Gamma}\right|
		&=&\left|\int_0^1 \sum_{i=1}^b \inp{\nabla_i f(X + \tau \Gamma)-\nabla_i f(X)}{\Gamma_i}_{(i)} d \tau \right| \\
		&=&\left|\int_0^1 \sum_{i\in S} \inp{\nabla_i f(X + \tau \Gamma)-\nabla_i f(X)}{\Gamma_i}_{(i)} d \tau \right| \\
		&\leq&\int_0^1 \sum_{i\in S} \left| \inp{\nabla_i f(X + \tau \Gamma)-\nabla_i f(X)}{\Gamma_i}_{(i)} \right| d \tau \\ 
		&\leq& \int_0^1 \sum_{i\in S} \norm{\nabla_i f(X + \tau \Gamma)-\nabla_i f(X)}_{(i) \star} \norm{\Gamma_i}_{(i)} d \tau \\
		&\overset{\eqref{as:arbitrary_layer_gen_smoothness2}}{\leq}& \int_0^1 \sum_{i\in S} \tau \parens{L_{i,S}^0 + L_{i,S}^1 \norm{\nabla_i f(X)}_{(i) \star}} \norm{\Gamma_i}_{(i)}^2 d \tau\\
		&=& \sum_{i\in S} \frac{L_{i,S}^0 + L_{i,S}^1 \norm{\nabla_i f(X)}_{(i) \star}}{2}\norm{\Gamma_i}_{(i)}^2.
	\end{eqnarray*}
\end{proof}

\begin{lemma}\label{lemma:layer_gen_smooth}
    Let Assumptions \ref{as:lower_bound} and \Cref{as:arbitrary_layer_gen_smoothness} hold and let $S\subseteq[b]$. Then
    \begin{eqnarray*}
        \sum_{i\in S} \frac{\norm{\nabla_i f(X)}^2_{(i) \star}}{2 \parens{L_{i,S}^0 + L_{i,S}^1 \norm{\nabla_i f(X)}_{(i) \star}}} \leq f(X) - f^{\star}
    \end{eqnarray*}
    for all $X \in \cX$.
\end{lemma}
\begin{proof}
    Let $Y = [Y_1, \ldots, Y_b] \in \cX$, where $Y_i = X_i - \frac{\norm{\nabla_i f(X)}_{(i) \star}}{L_{i,S}^0 + L_{i,S}^1 \norm{\nabla_i f(X)}_{(i) \star}} H_i$ for some $H_i \in \partial \norm{\cdot}_{(i) \star} (\nabla_i f(X))$ for $i\in S$ and $Y_i=X_i$ otherwise. By \Cref{lemma:arb_layer_gen_smooth}
    \begin{eqnarray*}
        f(Y) &\leq& f(X) + \inp{\nabla f(X)}{Y-X} + \sum_{i\in S} \frac{L^0_{i,S} + L_{i,S}^1 \norm{\nabla_i f(X)}_{(i) \star}}{2} \norm{X_i - Y_i}_{(i)}^2 \\
        &=& f(X) + \sum_{i\in S} \inp{\nabla_i f(X)}{Y_i - X_i}_{(i)} + \sum_{i\in S} \frac{L^0_{i,S} + L_{i,S}^1 \norm{\nabla_i f(X)}_{(i) \star}}{2} \norm{X_i - Y_i}_{(i)}^2 \\
        &=& f(X) - \sum_{i\in S} \frac{\norm{\nabla_i f(X)}_{(i) \star}}{L_{i,S}^0 + L_{i,S}^1 \norm{\nabla_i f(X)}_{(i) \star}} \inp{\nabla_i f(X)}{H_i}_{(i)} \\
        &&+ \sum_{i\in S} \parens{\frac{L^0_{i,S} + L_{i,S}^1 \norm{\nabla_i f(X)}_{(i) \star}}{2} \frac{\norm{\nabla_i f(X)}^2_{(i) \star}}{\parens{L_{i,S}^0 + L_{i,S}^1 \norm{\nabla_i f(X)}_{(i) \star}}^2} \norm{H_i}_{(i)}^2} \\
        &\overset{\eqref{eq:subdiff}}{=}& f(X) + \sum_{i\in S} \parens{- \frac{\norm{\nabla_i f(X)}^2_{(i) \star}}{L_{i,S}^0 + L_{i,S}^1 \norm{\nabla_i f(X)}_{(i) \star}} + \frac{\norm{\nabla_i f(X)}^2_{(i) \star}}{2 \parens{L_{i,S}^0 + L_{i,S}^1 \norm{\nabla_i f(X)}_{(i) \star}}}} \\
        &=& f(X) - \sum_{i\in S} \frac{\norm{\nabla_i f(X)}^2_{(i) \star}}{2 \parens{L_{i,S}^0 + L_{i,S}^1 \norm{\nabla_i f(X)}_{(i) \star}}},
    \end{eqnarray*}
    and hence
    \begin{eqnarray*}
        \sum_{i\in S} \frac{\norm{\nabla_i f(X)}^2_{(i) \star}}{2 \parens{L_{i,S}^0 + L_{i,S}^1 \norm{\nabla_i f(X)}_{(i) \star}}} \leq f(X) - f(Y) \leq f(X) - f^{\star}.
    \end{eqnarray*}
\end{proof}

\begin{lemma}\label{lemma:layer_gen_smooth3}
    Let Assumptions \ref{as:lower_bound} and \Cref{as:arbitrary_layer_gen_smoothness} hold and let $S\subseteq[b]$. Then, for any $x_i >0$, $i\in[b]$, we have
    \begin{eqnarray*}
        \sum_{i\in S} x_i \norm{\nabla_i f(X)}_{(i) \star}
        \leq 4 \max_{i\in S} (x_i L_{i,S}^1) \parens{f(X) - f^{\star}} + \sum_{i\in S} \frac{x_i L_{i,S}^0}{L_{i,S}^1}
    \end{eqnarray*}
    for all $X \in \cX$.
\end{lemma}
\begin{proof}
    Applying \Cref{lemma:layer_gen_smooth} and \Cref{lemma:ineq} with $y_i = \norm{\nabla_i f(X)}_{(i) \star}$, $z_i = L_{i,S}^0 + L_{i,S}^1 \norm{\nabla_i f(X)}_{(i) \star}$ and any $x_i>0$, we have
    \begin{eqnarray*}
        2 \parens{f(X) - f^{\star}}
        &\geq& \sum_{i\in S} \frac{\norm{\nabla_i f(X)}^2_{(i) \star}}{L_{i,S}^0 + L_{i,S}^1 \norm{\nabla_i f(X)}_{(i) \star}} \\
        &\geq& \frac{\parens{\sum_{i\in S} x_i \norm{\nabla_i f(X)}_{(i) \star}}^2}{\sum_{i\in S} x_i^2 L_{i,S}^0 + \sum_{i\in S} x_i^2 L_{i,S}^1 \norm{\nabla_i f(X)}_{(i) \star}} \\
        &\geq& \frac{\parens{\sum_{i\in S} x_i \norm{\nabla_i f(X)}_{(i) \star}}^2}{\sum_{i\in S} x_i^2 L_{i,S}^0 + \max_{i\in S} (x_i L_{i,S}^1) \sum_{i\in S} x_i \norm{\nabla_i f(X)}_{(i) \star}} \\
        &\geq& \begin{cases}
            \frac{\parens{\sum_{i\in S} x_i \norm{\nabla_i f(X)}_{(i) \star}}^2}{2 \sum_{i\in S} x_i^2 L_{i,S}^0} & \textnormal{if } \frac{\sum_{i\in S} x_i^2 L_{i,S}^0}{\max_{i\in S} (x_i L_{i,S}^1)} \geq \sum_{i\in S} x_i \norm{\nabla_i f(X)}_{(i) \star}, \\
            \frac{\sum_{i\in S} x_i \norm{\nabla_i f(X)}_{(i) \star}}{2 \max_{i\in S} (x_i L_{i,S}^1)} & \textnormal{otherwise}.
        \end{cases}
    \end{eqnarray*}
    Therefore,
    \begin{eqnarray*}
        \sum_{i\in S} x_i \norm{\nabla_i f(X)}_{(i) \star}
        &\leq& \max\brac{4 \max_{i\in S} (x_i L_{i,S}^1) \parens{f(X) - f^{\star}}, \frac{\sum_{i\in S} x_i^2 L_{i,S}^0}{\max_{i\in S} (x_i L_{i,S}^1)}} \\
        &\leq& 4 \max_{i\in S} (x_i L_{i,S}^1) \parens{f(X) - f^{\star}} + \frac{\sum_{i\in S} x_i^2 L_{i,S}^0}{\max_{i\in S} (x_i L_{i,S}^1)} \\
        &\leq& 4 \max_{i\in S} (x_i L_{i,S}^1) \parens{f(X) - f^{\star}} + \sum_{i\in S} \frac{x_i L_{i,S}^0}{L_{i,S}^1}.
    \end{eqnarray*}
\end{proof}

\end{document}